\documentclass[12pt]{article}

\usepackage{amsmath}
\usepackage{amsfonts}
\usepackage{amssymb}
\usepackage{amsthm}
\usepackage{dsfont}
\usepackage{subcaption}
\usepackage[colorlinks,citecolor=blue,urlcolor=blue]{hyperref}

\usepackage{enumitem}

\usepackage{fullpage}

\usepackage{algorithm}
\usepackage{algpseudocode}

\usepackage{natbib}

\newtheorem{prop}{Proposition}
\newtheorem{thm}{Theorem}
\newtheorem{remark}{Remark}
\newtheorem{lemma}{Lemma}

\newtheorem{defin}{Definition}

\newtheorem{assumption}{Assumption}

\newcommand{\Ncal}{\mathcal{N}}

\def\R{\mathbb{R}}
\def\E{\mathbb{E}} 

\def\diag{\textup{diag}}
\def\R{\mathbb{R}}
\def\E{\mathbb{E}}

\DeclareMathOperator*{\cov}{Cov}
\DeclareMathOperator*{\Cov}{Cov}
\DeclareMathOperator*{\Cor}{Cor}

\DeclareMathOperator*{\rank}{rank}
\DeclareMathOperator{\Tr}{Tr}
\DeclareMathOperator*{\card}{card}

\DeclareMathOperator*{\minimize}{minimize}
\DeclareMathOperator*{\argmax}{argmax}

\DeclareMathOperator*{\Pen}{Pen}
\DeclareMathOperator*{\Ind}{\mathds{1}}

\newcommand{\bpm}{\begin{pmatrix}}
\newcommand{\epm}{\end{pmatrix}}
\newcommand{\be}{\begin{equation}}
\newcommand{\ee}{\end{equation}}
\newcommand{\bes}{\begin{equation*}}
\newcommand{\ees}{\end{equation*}}

\def\boxit#1{\vbox{\hrule\hbox{\vrule\kern6pt  \vbox{\kern6pt#1\kern6pt}\kern6pt\vrule}\hrule}}

\usepackage{amsmath}
\usepackage{amsfonts}
\usepackage{amssymb}
\usepackage{amsthm}
\usepackage{dsfont}
\usepackage{subcaption}
\usepackage[colorlinks,citecolor=blue,urlcolor=blue]{hyperref}
\usepackage{blkarray}
\usepackage{mathtools}
\usepackage{graphicx} 
\usepackage{booktabs} 
\usepackage{mathrsfs}
\usepackage{amsthm,amssymb,epsfig}

\usepackage[algo2e]{algorithm2e}
\usepackage{algorithm}
\usepackage{algpseudocode}
\usepackage{siunitx}
\usepackage{array}
\newcolumntype{L}{>{\centering\arraybackslash}m{1.6cm}}

\usepackage{fullpage}

\usepackage{natbib}

\newcommand\undermat[2]{%
  \makebox[0pt][l]{$\smash{\underbrace{\phantom{%
    \begin{matrix}#2\end{matrix}}}_{\text{$#1$}}}$}#2}

\def\boxit#1{\vbox{\hrule\hbox{\vrule\kern6pt  \vbox{\kern6pt#1\kern6pt}\kern6pt\vrule}\hrule}}

\usepackage[textwidth=2cm, textsize=scriptsize]{todonotes}

\def\be{\boldsymbol{e}}

\def\bs{\boldsymbol{s}}

\def\bu{\boldsymbol{u}}
\def\bv{\boldsymbol{v}}
\def\bw{\boldsymbol{w}}
\def\bx{\boldsymbol{x}}
\def\by{\boldsymbol{y}}

\def\bA{\boldsymbol{A}}
\def\bB{\boldsymbol{B}}

\def\bD{\boldsymbol{D}}

\def\bG{\boldsymbol{G}}
\def\bH{\boldsymbol{H}}
\def\bI{\boldsymbol{I}}

\def\bM{\boldsymbol{M}}
\def\bN{\boldsymbol{N}}

\def\bP{\boldsymbol{P}}
\def\bQ{\boldsymbol{Q}}
\def\bR{\boldsymbol{R}}

\def\bU{\boldsymbol{U}}
\def\bV{\boldsymbol{V}}
\def\bW{\boldsymbol{W}}
\def\bX{\boldsymbol{X}}
\def\bY{\boldsymbol{Y}}
\def\bZ{\boldsymbol{Z}}

\def\bmu{\boldsymbol{\mu}}
\def\btheta{\boldsymbol{\theta}}

\def\blambda{\boldsymbol{\lambda}}
\def\bpi{\boldsymbol{\pi}}
\def\bphi{\boldsymbol{\phi}}
\def\bDelta{\boldsymbol{\Delta}}
\def\bLambda{\boldsymbol{\Lambda}}

\def\bPsi{\boldsymbol{\Psi}}
\def\bSigma{\boldsymbol{\Sigma}}
\def\bTheta{\boldsymbol{\Theta}}

\DeclareMathOperator*{\pr}{pr}
\usepackage{amsmath}
\usepackage{graphicx,psfrag,epsf}
\usepackage{natbib}

\usepackage{pdfsync}
\begin{document}

\def\spacingset#1{\renewcommand{\baselinestretch}%
{#1}\small\normalsize} \spacingset{1}

  \title{\bf Joint association and classification analysis of multi-view data}
  \author{Yunfeng Zhang\thanks{
    Yunfeng Zhang is PhD Student, Department of Statistics, Texas A\&M University, email: yfzhang@stat.tamu.edu}\hspace{.2cm}
    and 
   Irina Gaynanova\thanks{Irina Gaynanova is Assistant Professor, Department of Statistics, Texas A\&M University, email: irinag@stat.tamu.edu}}
    \date{} 
  \maketitle

\bigskip
\begin{abstract}
Multi-view data, that is matched sets of measurements on the same subjects, have become increasingly common with advances in multi-omics technology. Often, it is of interest to find associations between the views that are related to the intrinsic class memberships. Existing association methods cannot directly incorporate class information, while existing classification methods do not take into account between-views associations. In this work, we propose a framework for Joint Association and Classification Analysis of multi-view data (JACA).  Our goal is not to merely improve the misclassification rates, but to provide a latent representation of high-dimensional data that is both relevant for the subtype discrimination and coherent across the views. We motivate the methodology by establishing a connection between canonical correlation analysis and discriminant analysis. We also establish the estimation consistency of JACA in high-dimensional settings. A distinct advantage of JACA is that it can be applied to the multi-view data with block-missing structure, that is to cases where a subset of views or class labels is missing for some subjects. The application of JACA to quantify the associations between RNAseq and miRNA views with respect to consensus molecular subtypes in colorectal cancer data from The Cancer Genome Atlas project leads to improved misclassification rates and stronger found associations compared to existing methods. 

\end{abstract}

%
\vfill

\section{Introduction}
\label{sec:intro}

Multi-view data, that is matched sets of measurements on the same subjects, have become increasingly common with advances in multi-omics technology. For example, The Cancer Genome Atlas Project \citep{weinstein2013cancer} contains multiple views for the same set of subjects: gene expression, methylation, etc. At the same time, the subjects are often separated into subtypes (classes).
Our motivating example is the colorectal cancer (COAD) data with two views: RNASeq data of normalized counts and miRNA expression. The Colorectal Cancer Consortium determined four consensus molecular subtypes (CMS) of colorectal cancer based on gene expression information \citep{Guinney:2015dm}, and these subtypes were shown to have distinct survival prognosis. Since each view presents complementary information regarding the subject's biological system, our goal is to identify co-varying patterns between RNA-Seq and miRNA views that are relevant for discrimination of these subtypes.

A line of research has focused on finding co-varying patterns between the views based on canonical correlation analysis (CCA) \citep{Chen:2013uk, Gao:2014uz, Witten:2009vx}. These methods, however, do not use subtype information. Thus, while they find associations between the views, these associations are not necessarily related to subtypes of interest. \citet{Witten:2009wa} propose supervised CCA, where the relevant variables from each view are pre-selected before CCA based on the strength of their marginal association with the response. 
\citet{Li:2017uj} consider factor model, where each view is decomposed into shared and individual structures informed by the covariats. However, both~\citet{Witten:2009wa} and~\citet{Li:2017uj} use subtype information indirectly, and the methods are not tailored for subtype discrimination.




On the other hand, the subtype discrimination can be achieved using one of the many classification methods such as multinomial regression, multi-class support vector machines, discriminant analysis, etc. However, one has to either apply the chosen method separately to each view, or apply the method to the concatenated matrix of views. The separate approach may lead to inconsistent results across the views. The concatenation approach, however, ignores differences in signal strength across the views. When one view has a much stronger subtype-specific signal, the concatenation masks the less-dominant signals in other views. Our numerical results in Section~\ref{sec:COAD} demonstrate this phenomenon for COAD data: the subtype signal is much stronger in RNAseq view, and the discriminant analysis on concatenated dataset leads to almost no selected variables in miRNA view. Hence, the classification approaches do not allow to answer our primary question: what are the co-varying patterns between RNA-Seq and miRNA views that are relevant for subtype discrimination. These approaches also require known class assignments, and thus can not borrow strength from samples for which multiple views are available, but class information is missing.  

In this work, we develop a framework for Joint Association and Classification Analysis (JACA) of multi-view data by connecting discriminant analysis with canonical correlation analysis. Our goal is not to merely improve the misclassification rates, but to provide a latent representation of high-dimensional data that is both relevant for the subtype discrimination and coherent across the views. A distinct advantage of the proposed method is that it can be applied to the multi-view data with block-missing structure, that is to cases where a subset of views or class labels is missing for some subjects. For COAD data example, out of 282 subjects with RNAseq data, only 167 subjects have corresponding miRNA and cancer subtype information. While most methods can only use data from these 167 subjects with complete information, JACA can also use data from 78 extra subjects for which at least two types of information are available (two views with no class labels, or class labels with only one view). This extra information leads to improved classification accuracy, as shown in Section~\ref{sec:COAD} and in 
Section~\ref{sec:simu} of the Appendix.

In summary, our work makes the following contributions. First, we establish a connection between CCA and discriminant analysis using the factor model. Secondly, we use this connection to develop the JACA method for Joint Classification and Association Analysis, the method's formulation via convex optimization problem leads to efficient computations. Third, we establish estimation consistency of JACA in high-dimensional settings. Finally, we extend JACA to the settings with missing subsets of views or classes.

The rest of the paper is organized as follows. Section~\ref{sec:methods} establishes the connection between CCA and linear discriminant analysis, and describes the proposed JACA method. Section~\ref{sec:theory} provides the estimation error bound in high-dimensional settings. Section~\ref{sec:implementation} describes the method's implementation. Section~\ref{sec:missing} describes extension of JACA to block-missing data. 
Section~\ref{sec:COAD} provides the analysis of COAD data. Section~\ref{sec:discussion} concludes with discussion. The technical proofs of the main results, analysis of the breast cancer data from The Cancer Genome Atlas project, and additional numerical studies are in the Appendix.

\subsection{Relation to prior work}

Several method combine the task of finding associations between the views with the task of learning the regression coefficients. \citet{Gross:2014ux} propose to combine canonical correlation analysis with linear regression. The method, however, is restricted to univariate continuous response and can only be applied to two views. \citet{Luo:2016tb} propose to combine canonical correlation analysis objective with a general class of loss functions. Unlike \citet{Gross:2014ux}, the method could be applied to more than two views, and binary response. Nevertheless, the method is not suited for multi-group classification .
Finally, neither~\citet{Gross:2014ux} nor~\citet{Luo:2016tb} discuss the underlying population model. In contrast, the established connection between CCA and discriminant analysis in Section~\ref{sec:methods} allows us to both establish underlying JACA population model, as well as establish finite-sample estimation error bounds in Section~\ref{sec:theory}. To our knowledge, this is the first result that shows consistency of a joint learning method from theoretical perspective, the methods of \citet{Luo:2016tb, Gross:2014ux} come with no theoretical guarantees.

Since the method of \citet{Luo:2016tb} allows to perform joint association and classification in the two-class case, we further contrast it with JACA. First, we use discriminant analysis rather than the regression framework, which allows us to fix the rank for model fitting to be $K-1$, where $K$ is the number of classes. In~\citet{Luo:2016tb}, the rank of the model has to be chosen by the user. Secondly, we are able to formulate JACA as a convex optimization problem by using the optimal scoring formulation of multi-class discriminant analysis \citep{Hastie:1994cx} and fixing the scores to be orthogonally invariant \citep{Gaynanova:2019pe}. We add group-lasso type penalty to the optimization objective to allow for variable selection, and use block-coordinate descent algorithm to solve the corresponding convex problem. In contrast, the method of \citet{Luo:2016tb} is nonconvex, and requires the use of variable splitting and augmented Lagrangian. 

While estimation consistency has been established separately for discriminant analysis \citep{Li:2017kb, Gaynanova:2019pe} and canonical correlation analysis \citep{Gao:2014uz}, providing similar guarantees for JACA is not straightforward. We use the augmented data approach to rewrite our method as a penalized linear regression problem, and 
 use sub-exponential concentration bounds to control the inner-product between the augmented random design matrix and the random matrix of residuals. Despite the dependency between corresponding design matrix and the matrix of residuals, we obtain the estimation error bound that is of the same order as the known bounds for group-lasso linear regression \citep{Lounici:2011fl,Nardi:2008cf}. 
 
 \subsection{Notation}
 For two scalars $a, b \in \R$, we let $a \vee b = \max(a,b)$. 
For a vector $\bv\in \R^p$, we let $\|\bv\|_2 = (\sum_{j=1}^pv_j^2)^{1/2}$, $\|\bv\|_1 =\sum_{j=1}^p|v_j|$ and $\|\bv\|_{\infty}=\max_j|v_j|$. 
For matrices $\bM,\bN \in \R^{n \times p}$, we let $\|\bM\|_F = (\sum_{i=1}^n\sum_{j=1}^pm_{ij}^2)^{1/2}$, $\|\bM\|_{\infty,2}=\max_{1\leq i\leq n}(\sum_{j=1}^pm_{ij}^2)^{1/2}$, $\|\bM\|_{1,2}=\sum_{i=1}^ n(\sum_{j=1}^pm_{ij}^2)^{1/2}$ and $\langle \bM,\bN \rangle = \Tr(\bM^\top \bN) $.
We use $\bI=\bI_p$ to denote $p\times p$ identity matrix, and ${\bf{0}}$ to denote zero matrix. For two  sequences of scalars $a_1,\dots, a_n,\dots$ and $b_1,\dots, b_n,\dots$, we use $b_n = o(a_n)$ if $\lim_{n\to \infty}(b_n/a_n)=0$ and $b_n = O(a_n)$ if $\lim_{n \to \infty}(b_n/a_n) < C$ for some finite constant $C$.
For two sequences of random variables $x_1, \dots, x_n, \dots$ and $y_1, \dots, y_n, \dots$, we use $y_n = o_p(x_n)$ if for any $\varepsilon>0$ $P(|y_n/x_n|<\varepsilon)\to 0$ as $n \to \infty$, and $y_n = O_p(x_n)$ if for any $\varepsilon > 0$ there exists $M_{\varepsilon}$ such that $P(|y_n/x_n| > M_{\varepsilon})<\varepsilon$ for all $n$.

\section{Proposed methodology}\label{sec:methods}

\subsection{Connection between canonical correlation and linear discriminant analysis}\label{sec:sCCALDA}

In this section, we review the canonical correlation analysis (CCA) and the linear discriminant analysis (LDA). We demonstrate that discriminant vectors in LDA coincide with the subset of canonical vectors in CCA, and use this connection to motivate the proposed method.

We consider $n$ independent realizations of a random vector $(\bx_{1},\dots,\bx_{D}, y)\in \R^{p_1}\times \dots \times \R^{p_D}\times \{1,\dots, K\}$, where $\bx_{d}$ is the vector of measurements from view $d\in\{1,\dots, D\}$, and $y$ is the class assignment, with $P(y=k) = \pi_{k}$, $k=1,\dots, K.$ 



\begin{assumption}\label{a:xmeancov}
The marginal means $\E(\bx_d)=\bf{0}$ with marginal covariances $\bSigma_d = \E(\bx_d\bx_d^{\top})$ and marginal cross-covariances $\bSigma_{ld} = \E(\bx_l\bx_d^{\top})$, $l\neq d$, $l, d \in \{1, \dots, D\}$. Further,
\begin{equation}\label{eq:condX12}
\E \left[\bpm
\bx_1 \\
\vdots\\
\bx_D
\epm \Big | y = k \right]
= \bpm
\bmu_{1k}\\
\vdots\\
\bmu_{Dk}
\epm,\quad \Cov\left[ \bpm \bx_1 \\\vdots\\ \bx_D \epm \Big | y = k\right] = \bSigma_y =  \bpm
\bSigma_{1y} &...&\bSigma_{1Dy}\\
&\vdots&\\
\bSigma_{1Dy}^{\top}&...&\bSigma_{Dy}
\epm.
\end{equation}
\end{assumption}
 Thus, we assume that the marginal means are zero (this simplifies the notation, in practice we column-center the data), and the conditional class-covariance matrices are equal across $k$ (the key assumption in LDA). 
 
The population CCA for the given two views $d$ and $l$ seeks linear combinations $(\btheta_d, \btheta_l)$ that maximize $\Cor(\btheta_d^{\top}\bx_d, \btheta_l^{\top}\bx_l)$, that is it seeks at most $r$ pairs $(\btheta_{d}^{(k)}, \btheta_{l}^{(k)})$ that satisfy
\begin{align*}
(\btheta_{d}^{(l)}, \btheta_{l}^{(k)}) =\argmax_{\bw_d^{(k)}, \bw_l^{(k)}}&\left\{ \bw_d^{(k)\top}\bSigma_{dl}\bw_l^{(k)}\right\}\\
\mbox{subject to}&\quad \bw_{d}^{(k)\top}\bSigma_d\bw_{d}^{(k)} = 1,\bw_{l}^{(k)\top}\bSigma_l\bw_{l}^{(k)} = 1,\\
&\quad \bw_{d}^{(k)\top}\bSigma_d\bw_{d}^{(j)} = 0,\  \bw_{l}^{(k)\top}\bSigma_l\bw_{l}^{(j)} = 0\quad \mbox{for}\quad j<k.
\end{align*}
The pairs $(\btheta_d^{(k)}, \btheta_l^{(k)})$ are called canonical vectors, and the values $\rho_k = \btheta_d^{(k)\top}\bSigma_{dl}\btheta_l^{(k)}$ are canonical correlations. 
Furthermore, given the above constraints, the r pairs $(\btheta_{l}^{(k)}, \btheta_{d}^{(k)})$ solve the above problem if and only if \citep{Chen:2013uk}
\begin{equation}\label{eq:sigma12CCA}
\bSigma_{dl} = \bSigma_d \left(\sum_{k=1}^r \rho_k \btheta_{d}^{(k)}\btheta_{l}^{(k)\top}\right) \bSigma_l.
\end{equation}
That is, the population CCA problem is equivalent to the matrix decomposition problem of $\bSigma_{ld}$
We will use this alternative formulation to draw the connections with LDA. 

The population Fisher's LDA for the given view $d$ seeks matrix of discriminant vectors $\bTheta_d\in \R^{p_d \times (K-1)}$ that maximizes the between-class variability with respect to within-class variability \citep{Mardia:1979vm}. 
\citet{Gaynanova:2016wk} show that $\bTheta_d =  \bSigma_{d}^{-1}\bDelta_d$, where $\bDelta_d$ is the matrix of orthogonal mean contrasts between $K$ classes~\citep{Searle:2006ww}. Furthemore, the matrix of discriminant vectors $\bTheta_d$ is only unique up to orthogonal transformation and scaling, since for any $(K-1)\times (K-1)$ orthogonal matrix $\bR$ and a full rank $(K-1)\times (K-1)$ diagonal matrix $\bD$, the matrix $\bTheta_d \bR \bD$ leads to equivalent discrimination as $\bTheta_d$.

To connect the canonical vectors in CCA with discriminant vectors in LDA, we consider two cases. In the first case, we assume that the views are uncorrelated conditional on the class membership, that is $\Cov(\bx_d, \bx_l | y) = {\bf{0}}$, or equivalently $\bSigma_{dly} = {\bf{0}}$ for $d\neq l$. 
In the second case, we assume that there exist other  factors independent from class membership that drive associations between the views.

Consider the first case - the views are only related due to shared class membership.
\begin{thm}\label{p:factor} Let random $\by\in\{1,\dots, K\}$, $\bx_d\in \R^{p_d}$ be as in Assumption~\ref{a:xmeancov}, and let $\Cov(\bx_d, \bx_l | y) = \bSigma_{dly} = {\bf{0}}$ for all $l\neq d\in \{1,\dots, D\}$. Then\\
1. The following factor model holds:
\begin{equation}\label{eq:factor1}
\bx_d = \bDelta_d \bu_y  + \bSigma_{dy}^{1/2}\be_d,
\end{equation}
where $\bu_y = f(y) \in \R^{K-1}$ satisfies $\E(\bu_y)={\bf0}$, $\Cov(\bu_y)=\bI$; $\bDelta_d\in \R^{p_d\times (K-1)}$ is the matrix of orthogonal contrasts between class means such that $\bLambda_d = \bDelta_d^{\top} \bSigma_{dy}^{-1}\bDelta_d$ is diagonal; and $\be_d\in \R^{p_d}$ is independent from $y$ with $\E(\be_d)={\bf0}$, $\cov(\be_d)=\bI$. 

2. The marginal cross-covariance matrices satisfy
$$
\bSigma_{dl} = \bSigma_d \left(\sum_{k=1}^{K-1}\rho_k\btheta_{d}^{(k)}\btheta_{l}^{(k)\top}\right)\bSigma_l,
$$
where $\bTheta_d = [\btheta_{d}^{(1)}\dots \btheta_{d}^{(K-1)}]$ is equal to $\bSigma_d^{-1}\bDelta_d$ up to column-scaling and is orthonormal with respect to $\bSigma_d$, and $\rho_k$ are diagonal elements of matrix $(\bI+\bLambda_l)^{-1/2}\bLambda_l^{1/2}\bLambda_d^{1/2}(\bI+\bLambda_d)^{-1/2}$.
\end{thm}

Theorem~\ref{p:factor} states that when $\bSigma_{dly} = {\bf{0}}$, the canonical vectors in CCA coincide with discriminant vectors in LDA up to column-scaling. Since the scaling affects neither the canonical correlations nor the classification rule, the population CCA and LDA are equivalent in this case. Their sample counterparts, however, will not be equivalent due to (i) the contamination by noise, and (ii) distinct sparsity regularization in high-dimensional settings.

\begin{remark}
$\bu_y$ represents a transformed class indicator vector. 
When $K=2$, $\bu_y = f(y) = \sqrt{\pi_2/\pi_1}\Ind\{y=1\} - \sqrt{\pi_1/\pi_2}\Ind\{y=2\}$, case $K>2$ is in the Appendix.
\end{remark}
\begin{remark}
If $\rank(\bDelta_d^{\top}\bSigma_{dy}^{-1}\bDelta_d) < K-1$, then~\eqref{eq:factor1} is not identifiable. For clarity, we assume full rank $K-1$, but the results can be generalized at the expense of a more technical proof. When $K=2$, this is equivalent to requiring the class-conditional means to be distinct, which is a minor condition.
\end{remark}

Consider now the second case, that is there exists other factors independent from class membership that drive associations between the views. This scenario is illustrated in Figure~\ref{eq:factor2}, and the corresponding extension of the factor model~\eqref{eq:factor1} is
\begin{equation}\label{eq:factor2}
\bx_d  =  \bDelta_d \bu_y + \bA_d \bu + \widetilde \bSigma_d^{1/2}\be_d,
\end{equation}
where $\bDelta_d$, $\bu_y$ are as in Proposition~\ref{p:factor}, $\bu\in \R^q$ represents $q$ extra common factors between the $D$ views, and $\be_d \in \R^{p_d}$ is an independent noise vector. Following standard identifiability conditions for factor models \citep[Chapter~9.2]{Mardia:1979vm}, we assume that $\bu\in \R^q$ is independent from $y$ with $\E(\bu)={\bf0}$, $\Cov(\bu)=\bI$; and the loadings matrix $\bV_d = [\widetilde \bSigma^{-1/2}_d\bDelta_d\ \widetilde \bSigma_d^{-1/2}\bA_d]\in \R^{p_d \times (K-1+q)}$ is orthogonal. Here $\widetilde \bSigma_d$ is no longer class-conditional covariance matrix, but rather covariance matrix after accounting for both class membership ($\bu_y$) and other factors ($\bu$). It follows that $\bSigma_{dy} = \widetilde \bSigma_{d} + \bA_d\bA_d^{\top}$. When $\bA_d = \bf{0}$, the model reduces to~\eqref{eq:factor1}. We assume $\bA_d$ is full rank given $q$ (with $\bA_d = \bf{0}$ for $q=0$).

\begin{figure}[!t]
    \centering
    \includegraphics[width = \textwidth]{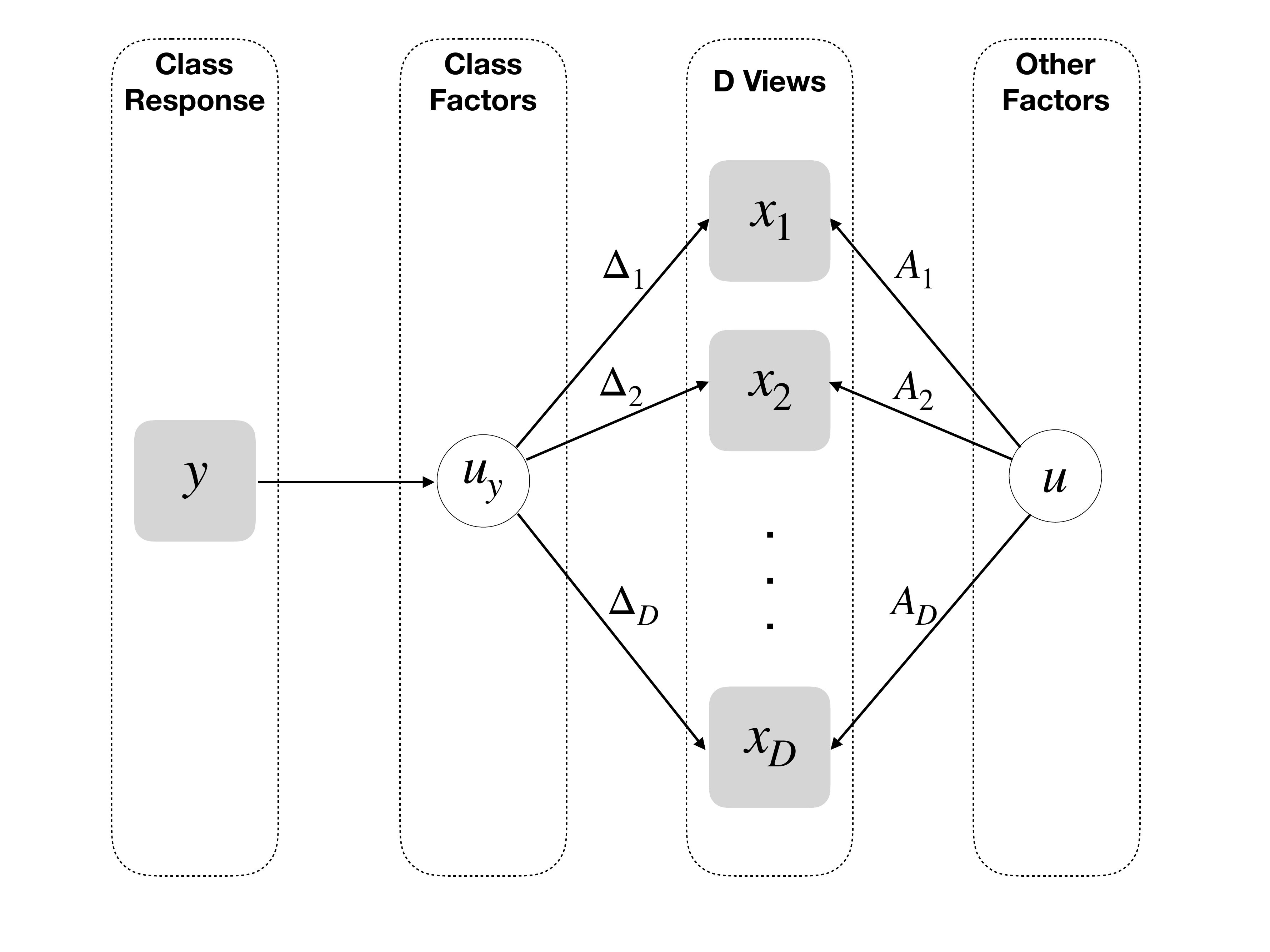}
    \caption{Illustration of factor model~\eqref{eq:factor2}, the $D$ views are connected due to both shared class factors $u_y$ and other independent factors $u$.}
    \label{fig:factor_model}
\end{figure}



$$
$$


\begin{thm}\label{t:sigma12}
Consider model~\eqref{eq:factor2} with corresponding identifiability conditions . Let $\bSigma_{ld} = \E (\bx_l\bx_d^{\top})$ be the marginal cross-covariance matrix between $\bx_l$ and $\bx_d$. Then
\begin{equation*}
    \bSigma_{ld} = \bSigma_l\left(\sum_{k=1}^{q+K-1} \rho_k \btheta_{l}^{(k)}\btheta_{d}^{(k)\top}\right)\bSigma_d,
\end{equation*}
where $\Big\{\btheta_{d}^{(k)}\Big\}_{k=1}^{q+K-1}$ are orthonormal with respect to $\bSigma_d$, $ \bSigma_l\left(\sum_{k=1}^{q} \rho_k \btheta_{l}^{(k)}\btheta_{d}^{(k)\top}\right)\bSigma_d = \bA_l\bA_d^{\top}$ and $\bSigma_l\left(\sum_{k=q+1}^{q+K-1} \rho_k \btheta_{l}^{(k)}\btheta_{d}^{(k)\top}\right)\bSigma_d = \bDelta_l\bDelta_d^{\top}$. That is, $\Big\{\btheta_{d}^{(k)}\Big\}_{k=q+1}^{q+K-1}$ are equal to $\bSigma_d^{-1}\bDelta_d$ up to column-scaling.
\end{thm}

When $q=0$, by Theorem~\ref{p:factor} the canonical vectors in CCA coincide with discriminant vectors in LDA. When $q>0$, there exist $q$ extra pairs of canonical vectors in CCA. If the LDA directions correspond to the maximal $\rho_k$, then the first $K-1$ canonical pairs coincide with discriminant vectors. If the LDA directions do not correspond to the maximal $\rho_k$, then the first $K-1$ canonical pairs include other shared factors that are independent of class membership.

\subsection{Joint association and classification analysis}\label{sec:proposal}

In light of correspondence between CCA and LDA explored in Theorems~\ref{p:factor}--~\ref{t:sigma12}, our proposal is based on combining the strengths of both approaches. Specifically, our goal is to estimate the view-specific matrices of canonical vectors that are relevant for class discrimination, that is to estimate the discriminant vectors $ \bSigma_d^{-1}\bDelta_d$ (up to orthogonal transformation and scaling). In comparison to separate LDA, we want to improve the estimation accuracy by jointly analyzing multiple views aided by CCA formulation interpretation. In comparison to CCA, we want the proposed model not to be fooled by leading canonical correlations that are independent of shared class memberships (other factors in Figure~\ref{fig:factor_model}). Our approach thus combines the canonical correlation objective with the classification objective of LDA.


For the correlation between the views, we consider the sample CCA criterion for column-centered views $\bX_d$ and $\bX_l$ as 
\begin{equation}\label{eq:CCA}
\minimize_{\bW_d, \bW_l}\|\bX_d\bW_d - \bX_l\bW_l\|_F^2\quad\mbox{subject to}\quad \frac1n\bW_d^{\top}\bX_d^{\top}\bX_d\bW_d = \bI, \quad \frac1n\bW_l^{\top}\bX_l^{\top}\bX_l\bW_l = \bI.
\end{equation}

For the classification, we consider the optimal scoring formulation \citep{Hastie:1994cx} of multi-class discriminant analysis \citep{Gaynanova:2016wk, Gaynanova:2019pe} for view $d$: 
\begin{equation}\label{eq:Vcan}
\minimize_{\bW_d \in \R^{p \times (K-1)}} \Big\{\frac1{2n}\|\widetilde \bY - \bX_d\bW_d\|^2_F + \lambda \Pen(\bW_d)\Big\},
\end{equation}
where $\Pen(\bW_d)$ is an optional penalty used to put structural assumptions such as sparsity, and $\widetilde \bY \in \R^{n \times (K-1)}$ is the transformed class response.  Let $\bZ\in \R^{n \times K}$ be the class-indicator matrix, $n_k$ be the number of samples in class $k$ and $s_k = \sum_{i=1}^k n_i$. Then $\widetilde \bY = \bZ\bH$, where $\bH \in \R^{K \times (K-1)}$ has columns $\bH_l\in \R^{K}$ defined as
$$
\bH_l=\Big(\Big\{(nn_{l+1})^{1/2}(s_ls_{l+1})^{-1/2}\Big\}_{l},\quad-(ns_l)^{1/2}(n_{l+1}s_{l+1})^{-1/2},\quad{\bf0}_{K-1-l}\Big)^{\top}.
$$


 To estimate the discriminant vectors $ \bSigma_d^{-1}\bDelta_d$ (up to orthogonal transformation and scaling), 
we propose to combine canonical correlation objective~\eqref{eq:CCA} with classification objective~\eqref{eq:Vcan}:
\begin{equation}\label{eq:colP_old}
\begin{split}
\minimize_{\bW_1,...,\bW_D}\Big\{&\frac{\alpha}{2nD} \sum_{d=1}^D \|\widetilde \bY - \bX_d \bW_d\|^2_F +\frac{1-\alpha}{2nD(D-1)} \sum_{d=1}^{D-1}\sum_{l = d+1}^D \|\bX_d\bW_d - \bX_l \bW_l\|^2_F\\& +\sum_{d=1}^D\lambda_d\Pen(\bW_d)\Big\}\quad\mbox{subject to}\quad \frac1n\bW_d^{\top}\bX_d^{\top}\bX_d\bW_d = \bI,\quad\mbox{}d=1,\dots, D.
\end{split}
\end{equation}
Here $\alpha \in [0,1]$ controls the relative weights between LDA and CCA criteria. When $\alpha = 0$, \eqref{eq:colP_old} reduces to sparse CCA. When $\alpha=1$, \eqref{eq:colP_old} reduces to sparse LDA with additional orthogonality constraints. While the orthogonality constraints are required for CCA criterion~\eqref{eq:CCA} to avoid trivial zero solution, they are  not necessary in~\eqref{eq:colP_old} as long as $\alpha>0$ due to the addition of classification objective. Moreover, it is sufficient to estimate the discriminant vectors up to orthogonal transformation and scaling as the classification rule is invariant to these transformations~\citep{Gaynanova:2016wk}. Therefore, we only consider $\alpha>0$, and drop the orthogonality constraints in~\eqref{eq:colP_old} leading to 
\begin{equation}\label{eq:colP}
\begin{split}
\minimize_{\bW_1,...,\bW_D}\Big\{&\frac{\alpha}{2nD} \sum_{d=1}^D \|\widetilde \bY - \bX_d \bW_d\|^2_F \\&+\frac{1-\alpha}{2nD(D-1)} \sum_{d=1}^{D-1}\sum_{l = d+1}^D \|\bX_d\bW_d - \bX_l \bW_l\|^2_F +\sum_{d=1}^D\lambda_d\Pen(\bW_d)\Big\}.
\end{split}
\end{equation}
We call~\eqref{eq:colP} JACA for Joint Association and Classification Analysis, and choose convex $\Pen(\bW_d)=\sum_{i=1}^{p_d}\|\bw_{di}\|_2$ to encourage row-wise sparsity in $\bW_d$. With this choice of penalty, problem~\eqref{eq:colP} is jointly convex in $\bW_1,\dots, \bW_D$. We do not consider $\ell_1$ penalty since it induces element-wise rather than row-wise sparsity in $\bW_d$, hence the variables are not completely eliminated from the model and the sparsity pattern is distorted by orthogonal transformation. Other row-wise sparse penalties that are nonconvex are discussed in \citet{Huang:2012wg}. 

JACA problem~\eqref{eq:colP} can be rewritten as a multi-response linear regression problem using the augmented data approach. For simplicity, we illustrate the case $D=2$, the more general case is described in the Appendix. 
Let $ \bW = ( \bW_1^\top, \bW_2^\top)^\top$,
$$
\bY' = \frac{\sqrt{\alpha}}{\sqrt{nD}}\bpm
\widetilde \bY\\
\widetilde \bY\\
\bf{0}\\
\epm,\quad
\bX' = \frac1{\sqrt{nD}}\bpm
\sqrt{\alpha}\bX_1 & \bf{0}\\
\bf{0} & \sqrt{\alpha}\bX_2\\
\sqrt{(1-\alpha)/(D-1)}\bX_1 & -\sqrt{(1-\alpha)/(D-1)}\bX_2 
\epm.
$$

Then~\eqref{eq:colP} is equivalent to
\begin{equation}\label{eq:augmented}
\minimize_{\bW}\Big\{2^{-1}\|\bY'-\bX'\bW\|^2_F  +\sum_{d=1}^D\lambda_d\Pen(\bW_d)\Big\}.
\end{equation}

\section{Estimation consistency}\label{sec:theory}
In this section, we derive the finite sample bound on the estimation error of the minimizer of~\eqref{eq:colP} with $\Pen(\bW_d)=\sum_{i=1}^{p_d}\|\bw_{di}\|_2$. Recall that our goal is to estimate discriminant vectors $\bTheta_d = \bSigma_d^{-1}\bDelta_d$ (up to orthogonal rotation and column scaling). 

First, consider the population objective function of~\eqref{eq:colP} with $\lambda_d=0$. To simplify the notation, we will work with equivalent augmented formulation~\eqref{eq:augmented}. Using the definition of augmented $\bX'$, $\bY'$ and Lemma~8 in \citet{Gaynanova:2015km}, 
\begin{align*}
\E(\bX'^\top \bX') =:\bG,\quad \E(\bX'^\top  \bY')= \widetilde\bDelta + o(1).
\end{align*}
Here $o(1)$ term captures the differences between empirical class proportions $n_k/n$ and prior class probabilities $\pi_k$, and $\widetilde\bDelta\in \R^{\left(\sum_{i=1}^Dp_i\right)\times(K-1)}$ has $r$th column defined as
\begin{equation*}
\widetilde\bDelta_r = \frac\alpha D\frac{\sqrt{\pi_{r+1}}\sum_{k=1}^r\pi_k(\bmu_k - \bmu_{r+1})}{\sqrt{\sum_{k=1}^r\pi_k\sum_{k=1}^{r+1}\pi_{k}}}.
\end{equation*}
Therefore
\begin{equation}\label{eq:augmentTrace}
\begin{split}
\E(2^{-1}\|\bY'-\bX'\bW\|_F^2) = 2^{-1}\Tr\{\bW^\top\bG\bW\} -\Tr\{\bW^{\top}\widetilde\bDelta\} + o(1) + C,
\end{split}
\end{equation}
where $C$ does not depend on $\bW$. Let $\bW^* = \bG^{-1}\widetilde \bDelta$. Then $\bW^*=( \bW_1^{*\top}, \bW_2^{*\top}, \dots, \bW_D^{*\top})^{\top}$ is the minimizer of population loss~\eqref{eq:augmentTrace} up to $o(1$) term. We next show that $\bW^*_d$ corresponds to discriminant vectors $\bTheta_d$ up to orthogonal transformation and column-scaling.

\begin{lemma} \label{l:theoreticalResults}
Consider model~\eqref{eq:factor2} with corresponding identifiability conditions. For any $\alpha \in (0,1]$, there exists orthogonal matrices $\bR_d$ such that $\bW_d^*\bR_d^\top$ is equal to $\bTheta_d = \bSigma_d^{-1}\bDelta_d$ up to column scaling.
\end{lemma}

The population loss~\eqref{eq:augmentTrace} can be viewed as the quadratic loss with respect to discriminant vectors $\bTheta_d$ with a particular choice of orthogonal transformation and scaling, which affect neither the classification rule nor the row-sparsity pattern.  The proof of Lemma~\ref{l:theoreticalResults} indicates that as long as $\alpha > 0$, the choice of $\alpha$ only affects the magnitude of the columns of $\bW^*$.

Next, we show that minimizer $\widehat \bW$ of penalized sample loss~\eqref{eq:colP} is consistent at estimating population loss minimizer $\bW^*$ under the following assumptions.
 
\begin{assumption}\label{a:sparsity} $\bTheta_d = \bSigma_d^{-1}\bDelta_d$ is row-sparse with the support $S_d= \{j:\|e_j^{\top}\bTheta_d\|_2\neq 0\}$ and $s_d=\card(S_d)$
. Hence $\bW^*_d$ is also row-sparse with the same support, and $\bW^*$ is row-sparse with the support $S=(S_1,\dots, S_D)$ and $s=\card(S)=\sum_{d=1}^Ds_d$.
\end{assumption}

\begin{assumption}\label{a:p}
The prior probabilities satisfy $0<\pi_{\min}\leq \pi_{k}\leq \pi_{\max}<1$, $k = 1, \dots, K$.
\end{assumption}


\begin{assumption}\label{a:norm}
\label{a:data}
 $\bx_d|y = k \sim \Ncal(\bmu_{dk}, \bSigma_{dy})$ for all $k=1,\ldots, K$.
\end{assumption}

\begin{assumption}\label{a:sample}
Let $p_{max}=\max_dp_d$ and $p_{min}=\min_dp_d$. Then for some constant $C>0$
$$
\frac{\log(p_{max})}{\log(p_{min})}\leq C  \mbox{ and } 
\log p_d = o(n),  \mbox{ for all } d=1,\dots,D.
$$
\end{assumption}
These assumptions are typical for multivariate analysis methods and high-dimensional settings.  Assumption~\ref{a:sparsity} states that population matrices of discriminant vectors are row-sparse. Assumption~\ref{a:p} states that the class proportions are not degenerate. Assumption~\ref{a:norm} states that the measurements are normally distributed conditionally on the class membership, it can be relaxed to sub-gaussianity without affecting the rates. Assumption~\ref{a:sample} allows to have a larger number of measurements than the number of samples, and states that the views have comparable numbers of measurements on the log scale. Because of the log scale, this assumption is mild, e.g. $p_{max}=1,000,000$ and $p_{min}=100$ leads to $C={\log(p_{max})}/{\log(p_{min})}=3$.

Similar to the assumptions required for estimation consistency in linear regression with group-lasso penalty~\citep{Nardi:2008cf,Lounici:2011fl}, we
also require restricted eigenvalue condition satisfied on the weighted cone. 
\begin{defin}[Weighted cone]
Let $\blambda = (\lambda_1,\dots, \lambda_d)$ and $S=(S_1,\dots, S_D)$. Then
$$
C(S, \blambda) = \Big\{\bM\in \R^{\sum_{d=1}^Dp_d \times (K-1)}: \sum_{d=1}^D\lambda_d \|\bM_{d,S_d^c}\|_{1,2}\leq 3\sum_{d=1}^D \blambda_d \|\bM_{d,S_d}\|_{1,2}\Big\}.
$$
\end{defin}
\begin{defin}\label{d:REgroup} A matrix $\bQ \in \R^{q\times p}$ satisfies restricted eigenvalue condition $\textrm{RE}(S,  \blambda)$ with parameter $\gamma_{Q} = \gamma(S,  \blambda, \bQ)$ if for some set $S$, and for all $\bA \in \mathcal{C}(S,\blambda)$ it holds that
$$
\|\bQ\bA\|_F^2 \geq {\gamma_{\bQ}}{\|\bA\|_F^2}.
$$
\end{defin}

We are now ready to state the main result. Let $\delta = \|\widetilde\bDelta\|_{\infty,2}$, let $g=\max_j\{\bG^{-1}\}_{jj}$ be the largest diagonal entry of $\bG^{-1}$, and let $\tau = \max_{j}\sqrt{\sigma_j^2+\max_k\mu_{k,j}^2}$, where $\sigma_j$ are  diagonal elements of $\bSigma_{y}$ and $\mu_{k,j}$ are elements of $\bmu_k$. 

\begin{thm}\label{t:fast_prob_p}  Under Assumptions~\ref{a:sparsity}--\ref{a:sample}, if $\lambda_d = C\left(\tau\vee \tau^2\delta g\right)D^{-1}\sqrt{{(K-1)\log[(K-1)p_d ]}/{n}}$ for some constant $C>0$, $s_d^2{\log[(K-1)p_{d}]}=o(n)$ and $\bG^{-1/2}$ satisfies  condition~$\textrm{RE}(S, \blambda)$ with parameter  $\gamma = \gamma(S, \blambda, \bG^{-1/2})$, then
\begin{align*}
\|\widehat \bW - \bW ^*\|_F &= O_p\left(\left(\tau\vee \tau^2\delta g\right)\frac1{D\gamma}\sqrt{\frac{K-1}{n}\sum_{d=1}^Ds_d\log[(K-1)p_{d}]}\right).
\end{align*}
\end{thm}
\begin{remark} 
If $p_d\geq K$ for all $d$, then $\log[(K-1)p_d]=\log(K-1)+\log p_d<2\log p_d$, and the rate could be simplified to 
\begin{align*}
\|\widehat \bW - \bW ^*\|_F &= O_p\left(\left(\tau\vee \tau^2\delta g\right)\frac1{D\gamma}\sqrt{\frac{K-1}{n}\sum_{d=1}^Ds_d\log(p_{d})}\right).
\end{align*}
\end{remark}
Our results allow both the number of variable $p_d$ and the number of classes $K$ to grow with $n$. The scaling requirement $s_d^2{\log[(K-1)p_{d}]}=o(n)$ is needed to ensure that restricted eigenvalue condition on $\bG$ implies restricted eigenvalue condition on random $\bX'^{\top}\bX'$ via the infinity norm bound. When $K=2$, $\widehat \bW$ and $\bW^*$ are vectors, and this condition can be dropped using the results of \citet{Rudelson:2013jw}.
 Nevertheless, the estimation error itself has the same rate as estimation error in linear regression with group-lasso \citep{Lounici:2011fl,Nardi:2008cf}. While our method can be viewed as multi-response linear regression due to formulation~\eqref{eq:augmented}, the group lasso results cannot be directly applied for several reasons. First, both $\bX'$ and $\bY'$ have dependencies across rows and contain fixed blocks of $0$ values. Second, the linear model assumption between $\bY'$ and $\bX'$ does not hold. Third, the residuals $\bPsi=\bY'-\bX'\bW^*$ do not have normal distribution and are dependent with $\bX'$. These challenges required the use of different proof techniques, and the full proof of Theorem~\ref{t:fast_prob_p} can be found in the Appendix.

\section{Missing data case - semi-supervised learning}\label{sec:missing}
In the joint analysis of multi-view data, it is typical to perform complete case analysis, that is only consider the subjects for which all the views and class labels are available. This is often not the case in practice. For the COAD data described in Section~\ref{sec:COAD}, out of 282 subjects with RNAseq data, only 218 have also available miRNA measurements. Moreover, 51 subjects out of these 218 have no class labels, and therefore can not be used to train supervised classification algorithms. Most of the available methods require either imputation of missing views/group labels, or perform complete case analysis (only use samples with complete information from all sources). A particular advantage of our framework is that we can also use the samples for which we have either a class-label or at least two views available without the need to impute the missing values. In other words, our proposal allows to perform semi-supervised learning, that is to use information from both labeled and unlabeled subjects to construct classification rules. In what follows, we assume that for each view and each subject, the measurements are rather completely missing, or not missing at all, that is we do not consider the case where a subset of measurements from one view is missing. 


Let $A_{dy}$ be the subset of samples (out of $n$) for which both class label and view $d$ are available, and let $B_{dl}$ be the subset of samples for which both views $d$ and $l$ are available. In case there are no missing labels/views, $A_{dy} = B_{dl} = \{1,\dots, n\}$ for all $d, l \in\{1, \dots, D\}$. We propose to adjust~\eqref{eq:colP} as
\begin{equation}\label{eq:colPmis2}
\begin{split}
    \minimize_{\bW_1,\dots,\bW_D}\Big\{&\frac{\alpha}{2nD} \sum_{d=1}^D\sum_{i\in A_{dy}} \|\widetilde \by_i - \bx_{id}^{\top}\bW_d\|_2^2 \\
    &+\frac{\alpha}{2nD(D-1)} \sum_{d=1}^{D-1}\sum_{l = d+1}^D \sum_{i \in B_{dl}}\|\bx_{id}^{\top}\bW_d - \bx_{il}^{\top} \bW_l\|^2_2 +\sum_{d=1}^D \lambda_d \Pen(\bW_d)\Big\},
\end{split}
\end{equation}
that is we use all samples with class labels and at least one view for the classification part, and all samples with at least two views for the canonical correlation part. 
Like~\eqref{eq:colP}, problem~\eqref{eq:colPmis2} is convex, and can be written as linear regression problem~\eqref{eq:augmented} with corresponding adjustments to $\bX'$ and $\bY'$. We refer to~\eqref{eq:colPmis2} as semi-supervised JACA (ssJACA).

\section{Implementation}\label{sec:implementation}
\subsection{Additional regularization via elastic net}\label{sec:shrinkPen}

It is well known that the lasso-type penalties can lead to erratic solution paths in the presence of highly-correlated variables \citep[Chapter~4.2]{Hastie:2015wu}. To overcome this drawback, \citet{Zou:2005ex} propose an elastic net penalty which combines ridge and lasso penalties, thus making highly correlated variables either being jointly selected or not selected in the model.  
\citet{Zou:2005ex} also advocate an extra scaling step which in regression context is equivalent to replacing the sample covariance matrix $\bX^{\top}\bX/n$ with the regularized version $(1-\rho) \bX^\top \bX/n+\rho \bI$ for $\rho\in[0,1]$. We adapt this idea to JACA, and replace $\bX'^{\top}\bX'$ in~\eqref{eq:augmented} with $(1-\rho){\bX'^{\top}\bX'} + \rho \bI$ for $\rho\in[0,1]$ leading to
\begin{equation}\label{eq:elasticobj}
\minimize_{\bW}\left\{\frac12\|\bY'-\bX'\bW\|^2_F -\frac\rho2\|\bX'\bW\|_F^2+\frac\rho2\|\bW\|_F^2 +\sum_{d=1}^D\lambda_d\Pen(\bW_d)\right\}.
\end{equation}
Problem~\eqref{eq:elasticobj} is convex, and the results of Section~\ref{sec:theory} can be extended to~\eqref{eq:elasticobj} with a more technical proof \citep{Hebiri:2011hg}. When $\rho = 0$, problems~\eqref{eq:elasticobj} and~\eqref{eq:augmented} coincide.

\subsection{Optimization algorithm}

We assume that each $\bX_d$ is standardized so that the diagonal entries of $n^{-1}\bX_d^\top \bX_d$ are equal to one. This standardization is common in the literature \citep{Zou:2005ex,  Witten:2009wa}, and effectively results in penalizing each variable proportionally to its standard deviation. Moreover, using $\rho>0$ with this standardization in~\eqref{eq:elasticobj} ensures the uniqueness of solution for any $\lambda_d$ due to strict convexity of the objective function.

We use a block-coordinate descent algorithm to solve~\eqref{eq:elasticobj} for fixed values of $\rho \in [0,1]$ and $\lambda_d \geq 0$. 
Let $\bw_{dj}$ be the $j$th row of $\bW_d$, and let $\Pen(\bW_d)=\sum_{j=1}^{p_d}\|\bw_{dj}\|_2$. Since~\eqref{eq:elasticobj} is convex, and the penalty is separable with respect to each $\bw_{dj}$, the algorithm is guaranteed to converge to the global optimum from any starting point \citep{Tseng:2001wm}. Consider solving~\eqref{eq:elasticobj} with respect to $\bw_{dj}$, and let $\bX'_{dj}$ be the corresponding column of $\bX'$. The KKT conditions \citep{Boyd:2004uz} correspond to a set of $\sum_{d=1}^Dp_d$ equations of the form
\begin{equation}\label{eq:kkt}
(1-{\rho})\bX'^{\top}_{dj}\bX'\bW +\rho \bw_{dj} - \bX_{dj}'^{\top}\bY'+ \lambda_d \bu_{dj}=0, 
\end{equation}
where $\bu_{dj}$ is the subgradient of $\|\bw_{dj}\|_2$, that is $\bu_{dj}= \bw_{dj}/\|\bw_{dj}\|_2$ when $\|\bw_{dj}\|_2\neq0$ and $\bu_{dj}\in\{\bu:\|\bu\|_2\leq1\}$ otherwise. Solving \eqref{eq:kkt} with respect to $\bw_{dj}$ leads to 
$$\bw_{dj}=\left\{\bX_{dj}'^{\top}(\bY' - (1-\rho)\bX'\bW + (1-\rho)\bX_{dj}'\bw_{dj})-\lambda_d \bu_{dj}\right\}\Big/\left\{\left(1-{\rho}\right)\|\bX_{dj}'\|_2^2+\rho\right\}.$$ 
For a vector $\bv\in \R^m$ and $\lambda >0$, let $S_{\lambda}(\bv) = \max(0, 1- \lambda/\|\bv\|_2)\bv$ be the vector soft-thresholding operator. Then iterating block updates leads to Algorithm~\ref{a:block}.

\begin{algorithm}[H]
 Given:  $\bW^{(0)}$, $\bX'$, $\bY'$, $k=0$, $\lambda_d\geq0$, $\rho \in [0, 1] $ and $\varepsilon >0$;\\
 $ \bR\gets \bY' -(1-\rho) \bX'\bW^{(0)}$;\\
\While{$k\neq k_{max}$ and $\bW^{(k)}$ satisfies
 $\left|\text{objective}\left(\bW^{(k)}\right)-\text{objective}\left(\bW^{(k-1)}\right)\right|\geq\varepsilon$}{
 
 $k\gets k+1$;\\
 \For{$d=1$ {\bf to} $D$ and $j=1$ {\bf to} $p_d$ }{
        $\bw_{dj}^{(k)}\gets S_{\lambda_d}(\bX_{dj}'^{\top}\bR + (1-\rho) \|\bX'_{dj}\|_2^2 \bw_{dj}^{(k-1)})/\{(1-\rho)\|\bX'_{dj}\|_2^2+\rho\};$\\
        $\bR \gets \bR + (1-\rho)\bX'_{dj} (\bw_{dj}^{(k-1)}- \bw_{dj}^{(k)})$
        
    }    
 }
 \caption{Block-coordinate descent algorithm for~\eqref{eq:elasticobj}}\label{a:block}
\end{algorithm}

\subsection{Selection of tuning parameters}\label{sec:tuning}

JACA requires the specification of several parameters: $\alpha\in (0,1]$ that controls the relative weights of LDA and CCA criteria, $\rho \in [0,1]$ that controls the shrinkage induced by elastic net, and $\lambda_d \geq 0$ that control the sparsity level of each $\bW_d$ respectively. While it is possible to perform cross-validation over all of the parameters, due to computational considerations we restrict the space as follows. 
 First, based on the empirical results in Section~\ref{sec:COAD}, we found that $\alpha \in [0.5,0.7]$  strikes a balance between classification and association analysis, with larger values corresponding to  better misclassification rate and slightly lower found associations. 
Secondly, we set $\lambda_d = \epsilon\lambda_{\max,d}$ with $\epsilon \in(0,1)$, where $\lambda_{\max,d}$ is defined as follows. 
\begin{prop}\label{p:lambda_max_d} Let $\lambda_{\max,d}= \frac{\alpha}{nD}\|\bX_d^\top\widetilde \bY\|_{\infty, 2}$.
Then $\widehat \bW_d = 0$ for all $\lambda \ge \lambda_{\max, d}$.
\end{prop}
This allows to control the sparsity of each $\widehat \bW_d$ at similar levels, similar strategy is used in \citet{Luo:2016tb}.
%

To select $\rho\in[0,1]$ and $\epsilon\in[10^{-4},1]$, we use $F$-fold cross-validation with a course grid for $\rho$ and a fine grid for $\epsilon$. It is typical to minimize the prediction error in cross-validation, for example the least squares error in the linear regression. In our context, however, both classification rules and correlation measures are invariant to the scale of $\bW_d$, hence we need a scale-invariant metric. We propose to consider
\begin{equation}\label{eq:CV}
\begin{split}
    CV(\rho,\varepsilon) = \frac1{F}\sum_{f = 1}^{F}\Bigg\{\alpha\sum_{d=1}^D&|\text{Cor}(\widetilde \bY^{(f)},\bX_d^{(f)}\widehat \bW_d^{(-f)})|\\
    &+\frac{(1-\alpha)}{D-1}\sum_{d=1}^{D-1}\sum_{l=d+1}^D|\text{Cor}(\bX_d^{(f)}\bW_d^{(-f)},\bX_l^{(f)}\bW_l^{(-f)})|\Bigg\},
\end{split}
\end{equation}
where $\widetilde \bY^{(f)}$, $\bX_d^{(f)}$ correspond to the samples in the $f$th fold; and $\widehat \bW_d^{(-f)}$ are solutions to~\eqref{eq:elasticobj} with given $\rho$ and $\varepsilon$ based on samples in all folds except the $f$th. We define the correlation between two centered matrices $\bX$ and $\bY$ as the square root of the RV-coefficient \citep{robert1976unifying}, where 
$$
\text{RV}(\bX,\bY) := \frac{\Tr(\bX \bX^\top \bY \bY^\top)}{\sqrt{\Tr(\bX \bX^\top)^2}\sqrt{\Tr(\bY \bY^\top)^2}}.
$$
By definition, $\sqrt{\text{RV}(\bX,\bY)}\in [0,1]$, and is invariant to scale and orthogonal transformation. If $\bX$ and $\bY$ are vectors, then $\sqrt{\text{RV}(\bX,\bY)} = |\text{Cor}(\bX,\bY)|$.


To adopt the proposed cross-validation scheme 
for ssJACA, we stratify the samples based on the patterns of ``missingness", and split each stratum into $F$ folds. We illustrate the case $D=3$. Let $H$ be the subset of samples (out of $n$) with no missing labels/views. Let $M_{y}$  and $M_{d}$ be the subset of samples for which only class labels or only view $d$ is missing. respectively. Let $M_{dy}$ be the subset of samples for which only class label and view $d$ are missing, and $M_{dl}$ be the subset of samples for which only views $d$ and $l$ are missing. We first randomly divide each of these strata into $F$ folds: $H^{(f)}, M_{y}^{(f)}, M_{d}^{(f)}, M_{dy}^{(f)}$ and  $M_{dl}^{(f)}$ where $f=1,2,\dots,F$. For each $f$, we then hold out the union of $H^{(f)}, M_{y}^{(f)}, M_{d}^{(f)}, M_{dy}^{(f)}$ and  $M_{dl}^{(f)}$ for testing, and use the remaining samples for training so that the criterion~\eqref{eq:CV} can be applied.


\section{Application to TCGA COAD data analysis}\label{sec:COAD}
We consider the colorectal cancer (COAD) data from The Cancer Genome Atlas project with two views: RNAseq data of normalized counts and miRNA expression. We extracted samples corresponding to primary tumor tissue using \textsf{TCGA2STAT} R package \citep{Wan:2015dm}. To account for data skewness and zero counts, we further log10-transformed both datasets with offset 1, and filtered the data to select 1572 variables for RNA-Seq and 375 for miRNA with  highest standard deviation across samples.  The Colorectal Cancer Consortium determined 4 consensus molecular subtypes (CMS) of colorectal cancer based on gene expression \citep{Guinney:2015dm}, and we have extracted the assigned subtypes for COAD data from the Synapse platform (Synapse ID syn2623706). The resulting data has 282 subjects in total, with Table~\ref{tab:COADdata} displaying the pattern of available information for each subject. Our primary goal is to identify covarying patterns between RNA-Seq and miRNA data that are relevant for subtype discrimination.   Although traditional CCA methods are also able to find concordance between RNA-Seq and miRNA data, such associations are not guaranteed to be closely related to subtypes.  We are also interested in selecting  miRNAs that are differentially expressed across different subtypes. The challenge is that miRNA is less informative to CMS. While \citet{Guinney:2015dm}  used two sample t-tests to achieve this goal, some samples were excluded since they lack CMS assignments. On the contrary, these information can be integrated by ssJACA.

\begin{table}
\center
    \caption{ \label{tab:COADdata}Number of available samples in COAD data with different missing patterns of CMS class/RNAseq/miRNA. Complete cases analysis will only be able to use 167 samples, whereas our semi-supervised approach allows to use 245 (all except the last row).}

    \begin{tabular}{cccc}
    \hline\hline
        CMS class & RNAseq & miRNA & Sample size  \\
        \hline
        yes & yes& yes & 167\\
        yes& yes & no & 27\\
        no & yes & yes & 51 \\
        no & yes & no & 37\\
        \hline
        & & & Total: 282\\
        \hline\hline
    \end{tabular}
\end{table}

We compare the performance of the following methods using the subset of subjects with complete views and subtype information ($n=167$): (i) \textsf{JACA}: Joint Association and Classification Analysis with $\alpha=0.7$, the proposed approach; (ii) \textsf{ssJACA}: semi-supervised JACA with $\alpha=0.7$;
(iii) Sparse Linear Discriminant Analysis of \citet{Gaynanova:2016wk} as implemented in the R package MGSDA \citep{MGSDA}, either applied separately to each dataset (\textsf{SLDA\_sep}), or jointly on concatenated dataset (\textsf{SLDA\_joint}); (iv) \textsf{Sparse~CCA:} Sparse Canonical Correlation Analysis of \citet{Witten:2009wa} as implemented in the R package PMA \citep{PMA}. We use cross-validation to choose the tuning parameters instead of the permutation method introduced in \citet{Witten:2009wa}, since the former one gives better results.  (v) \textsf{Sparse sCCA:} Sparse supervised CCA proposed in \citet{Witten:2009wa}. We first choose a set of variables with largest values of F-statistic from a one-way ANOVA, and then apply Sparse CCA with selected variables. We do not consider Canonical Variate Regression by \citet{Luo:2016tb} since it is only implemented for the binary classification problem. 

We randomly select 132 subjects for training and 35 for testing for the total of 100 random splits. For ssJACA, we add 78 subjects (at least  two  views  available) into the training set. We set $\alpha$ in JACA and ssJACA to be $0.7$. The average misclassification rates and the number of selected variables for each method are presented in~Table~\ref{tab:error_1}. We consider two prediction approaches for each method: prediction based on one view alone (either RNA-seq or miRNA) using the corresponding subset of canonical vectors, and prediction using the concatenated dataset. In general, the performance using miRNA data is worse, which is not surprising since the subtypes were determined using gene expression data alone \citep{Guinney:2015dm}.  Although  ssJACA selects more variables than SLDA\_sep, it performs the best in terms of misclassification rates, with JACA  being the second  best.  SLDA\_joint achieves a competitive misclassification rate using RNAseq data but not miRNA. We conjecture this is because RNAseq view has a much stronger class-specific signal that masks miRNA's signal when datasets are concatenated. This explanation is supported by the mean number of variables selected by SLDA\_joint from each view. Both supervised and unsupervised CCA methods perform poorly in classification. Based on results from Section~\ref{sec:simu} of the Appendix, this suggests that the subtype-specific association between the views is weak compared to association due to other common factors.

\begin{table}
 \caption{\label{tab:error_1} Mean misclassification rates in percentages and mean number of selected features over 100 random splits of 167 samples from COAD data with complete information, standard errors are given in brackets and the lowest values are highlighted in bold.}
\resizebox{\textwidth}{!}{
\begin{tabular}{lcccccc}
  \hline\hline
  & \multicolumn{3}{c}{Misclassification Rate  (\%)} & \multicolumn{3}{c}{Cardinality} \\
  \cmidrule(lr){2-4} \cmidrule(lr){5-7}
 Method & RNAseq & miRNA & Both  & RNAseq & miRNA & Both  \\
\hline
JACA & {\bf1.37} (0.20) & 6.26 (0.36) & 3.17 (0.28) & 375.4 (10.8) & 199.4 (4.1) & 574.8 (14.9) \\ 
  ssJACA & {\bf1.29} (0.16) & {\bf3.03} (0.26) & {\bf1.83} (0.21) & 338.2 (10.2) & 201.5 (4.1) & 539.7 (14.20) \\ 
  SLDA\_sep & 5.34 (0.47) & 9.31 (0.39) & 5.34 (0.30) & {\bf63.7} (2.8) & 58.6 (1.7) & 122.3 (3.2) \\ 
  SLDA\_joint & 2.71 (0.29) & 51.43 (1.64) & 2.91 (0.30) & {\bf63.4} (2.8) & {\bf4.5} (0.4) & {\bf67.9} (3.2) \\ 
  Sparse sCCA & 37.4 (0.08) & 42.26 (0.21) & 37.14 (0.00) & 930.9 (3.5) & 252.5 (0.9) & 1183.4 (3.9) \\ 
  Sparse CCA & 48.83 (0.3) & 50.6 (0.19) & 49.63 (0.15) & 1282.1 (6.5) & 368.9 (0.6) & 1651 (6.7) \\ 
   \hline\hline
\end{tabular}}
\end{table}

We also compare the out-of-sample correlation values, that is $\Cor(\bX_1\widehat \bW_1, \bX_2\widehat \bW_2)$, where $(\bX_1,\bX_2)$ are RNAseq and miRNA data from test samples, and $\widehat \bW_1$, $\widehat \bW_2$ are 
estimated on the training data. We do not consider CCA methods due to their poor classification performance. The results are presented in Table~\ref{tab:cortable}, with JACA and ssJACA achieving the strongest correlation value.

\begin{table}
\center
 \caption{Analysis based on 167 samples from COAD data with complete view and subtype information based on 100 random splits. Mean correlation between $\bX_1 \widehat \bW_1$ and $\bX_2 \widehat \bW_2$ where $\bX_1, \bX_2$ are samples from test data, and $\widehat \bW_1$, $\widehat \bW_2$ are estimated from the training data, standard errors are given in brackets and the highest value is highlighted in bold.}
  \label{tab:cortable}
\begin{tabular}{lcccc}
  \hline
 & JACA & ssJACA & SLDA\_sep & SLDA\_joint \\ 
  \hline
Correlation & {\bf 0.95} (0.001) & {\bf0.95} (0.001) & 0.90 (0.002) & 0.40 (0.021) \\ 
   \hline \hline
\end{tabular}
\end{table}

We further evaluate the effect of varying $\alpha$ parameter. We compare JACA fitted on $132$ subjects (all views and subtypes available) with semi-supervised JACA 
fitted on $210$ subjects (at least two views available), and validate the results on $35$ subjects. The average misclassification rates and the number of selected variables for each method are presented in Table~\ref{tab:error_2}, and the out-of-sample  correlation  values are shown in Figure~\ref{fig:COADcor2}. When increasing $\alpha$, the misclassification rates of both JACA and ssJACA are  decreasing while the variable selection results remain stable.  Although the  out-of-sample  correlation  values decreases as $\alpha$ increases, JACA and ssJACA have similar performances and the absolute changes are insignificant. This is perhaps not surprising since we put more weight on the classification part and less weight on the association part as $\alpha$ increases.

\begin{table}
 \caption{\label{tab:error_2} Mean misclassification rates in percentages and mean number of selected features of JACA and ssJACA with different $\alpha$ values over 100 replications from COAD data, standard errors are given in brackets and the lowest values are highlighted in bold.}
\resizebox{\textwidth}{!}{
\begin{tabular}{lcccccc}
  \hline\hline
  & \multicolumn{3}{c}{Misclassification Rate  (\%)} & \multicolumn{3}{c}{Cardinality} \\
  \cmidrule(lr){2-4} \cmidrule(lr){5-7}
 Method & RNAseq & miRNA & Both  & RNAseq & miRNA & Both  \\
\hline
JACA 0.1 & 4.43 (0.24) & 7.09 (0.33) & 5.94 (0.27) & 421.9 (16.4) & 214.0 (5.2) & 635.8 (21.5) \\ 
  ssJACA 0.1 & 2.14 (0.19) & 4.91 (0.27) & 2.69 (0.2) & {\bf 279.1} (7.1) & {\bf 179.2} (3.3) & {\bf 458.3} (10.3) \\ \hline
  JACA 0.3 & 1.63 (0.24) & 6.57 (0.33) & 5.23 (0.28) & 391.4 (9.9) & 205.2 (3.9) & 596.7 (13.7) \\ 
  ssJACA 0.3 & {\bf 1.31} (0.16) & 3.43 (0.29) & {\bf 1.69} (0.2) & 305.5 (6.8) & 189.8 (3) & 495.4 (9.8) \\ \hline
  JACA 0.5 & 1.49 (0.22) & 6.43 (0.34) & 4.03 (0.3) & 391.0 (9.7) & 205.4 (3.7) & 596.4 (13.3) \\ 
  ssJACA 0.5 & {\bf 1.31} (0.17) & {\bf 3.09} (0.27) & {\bf 1.74} (0.21) & 317.2 (7.8) & 194.6 (3.4) & 511.8 (11.1) \\ \hline
  JACA 0.7 & {\bf 1.37} (0.2) & 6.26 (0.36) & 3.17 (0.28) & 375.4 (10.8) & 199.4 (4.1) & 574.8 (14.9) \\ 
  ssJACA 0.7 & {\bf 1.29} (0.16) & {\bf 3.03} (0.26) & {\bf 1.83} (0.21) & 338.2 (10.2) & 201.5 (4.1) & 539.7 (14.2) \\ \hline
  JACA 0.9 & 1.46 (0.21) & 5.57 (0.36) & 2.77 (0.27) & 375.0 (14.3) & 196.1 (5.2) & 571.1 (19.4) \\ 
  ssJACA 0.9 & {\bf 1.37} (0.17) & {\bf 3.00} (0.26) & {\bf 1.80} (0.21) & 349.7 (11.8) & 202.5 (4.5) & 552.1 (16.2) \\ 
   \hline\hline
\end{tabular}}
\end{table}


\begin{figure}
\begin{center}
\centerline{\includegraphics[scale=.6]{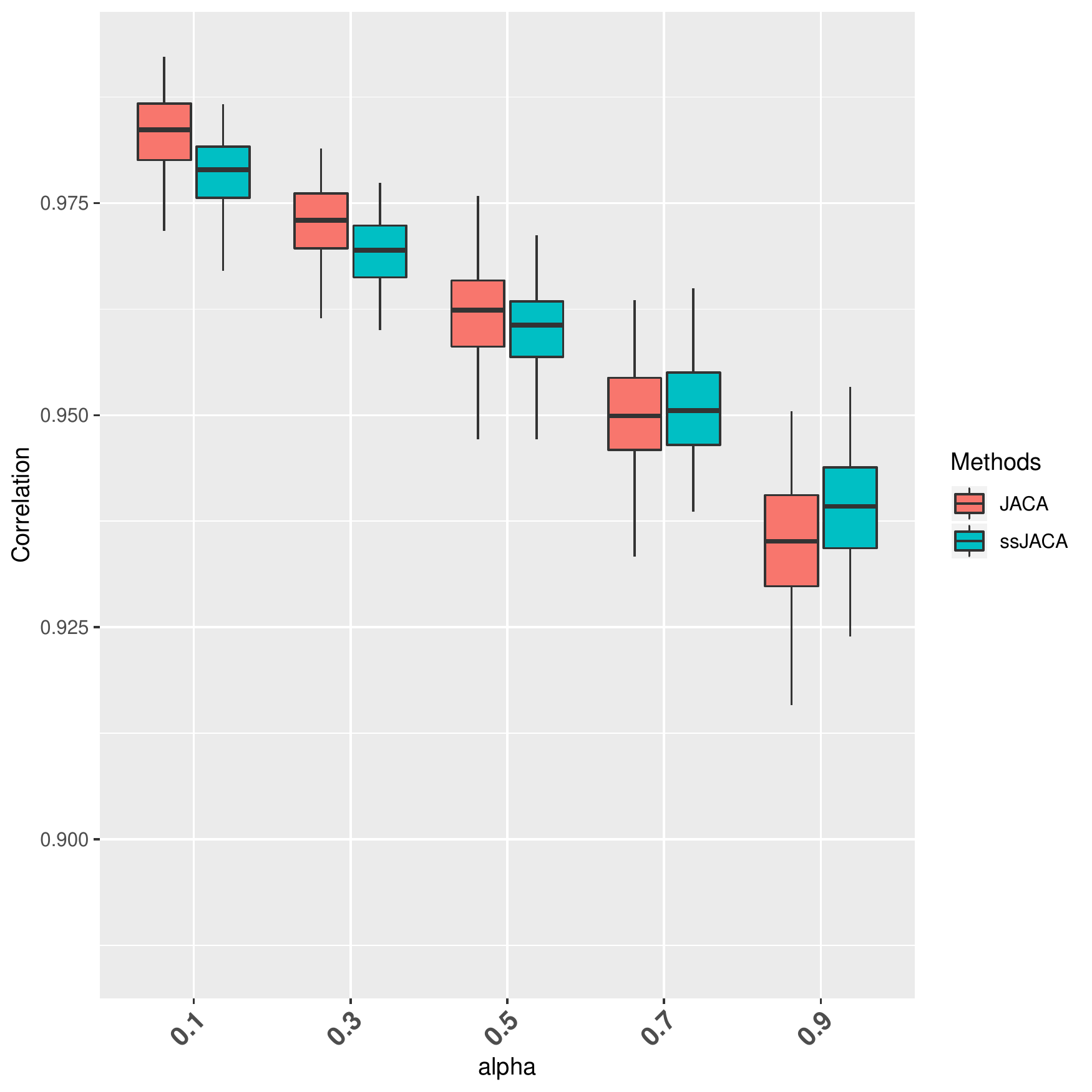} }
\end{center}
\caption{ Correlation analysis of  COAD data using JACA and ssJACA with different $\alpha$ based on 100 random splits. Correlation values between $\bX_1 \widehat \bW_1$ and $\bX_2 \widehat \bW_2$ where $\bX_1, \bX_2$ are samples from test data, and $\widehat \bW_1$, $\widehat \bW_2$ are estimated from the training data. \label{fig:COADcor2}}
\end{figure}

 The heatmaps of RNAseq and miRNA data with features selected by ssJACA are shown in Figure~\ref{fig:heatmap}, an enlarged version of this Figure as well as a projection of data onto the space spanned by the canonical vectors can be found in the Appendix.  Both views demonstrate different patterns across CMS classes, with the separation on RNASeq being visually much clearer. This is not surprising as CMS classes have been determined based on gene expression data only. Our analysis, however, also allows to determine co-varying patterns in miRNA, with subtype CMS4 being visually the most distinct in that view. 
\begin{figure}
\centering
\makebox{\includegraphics[scale=.42]{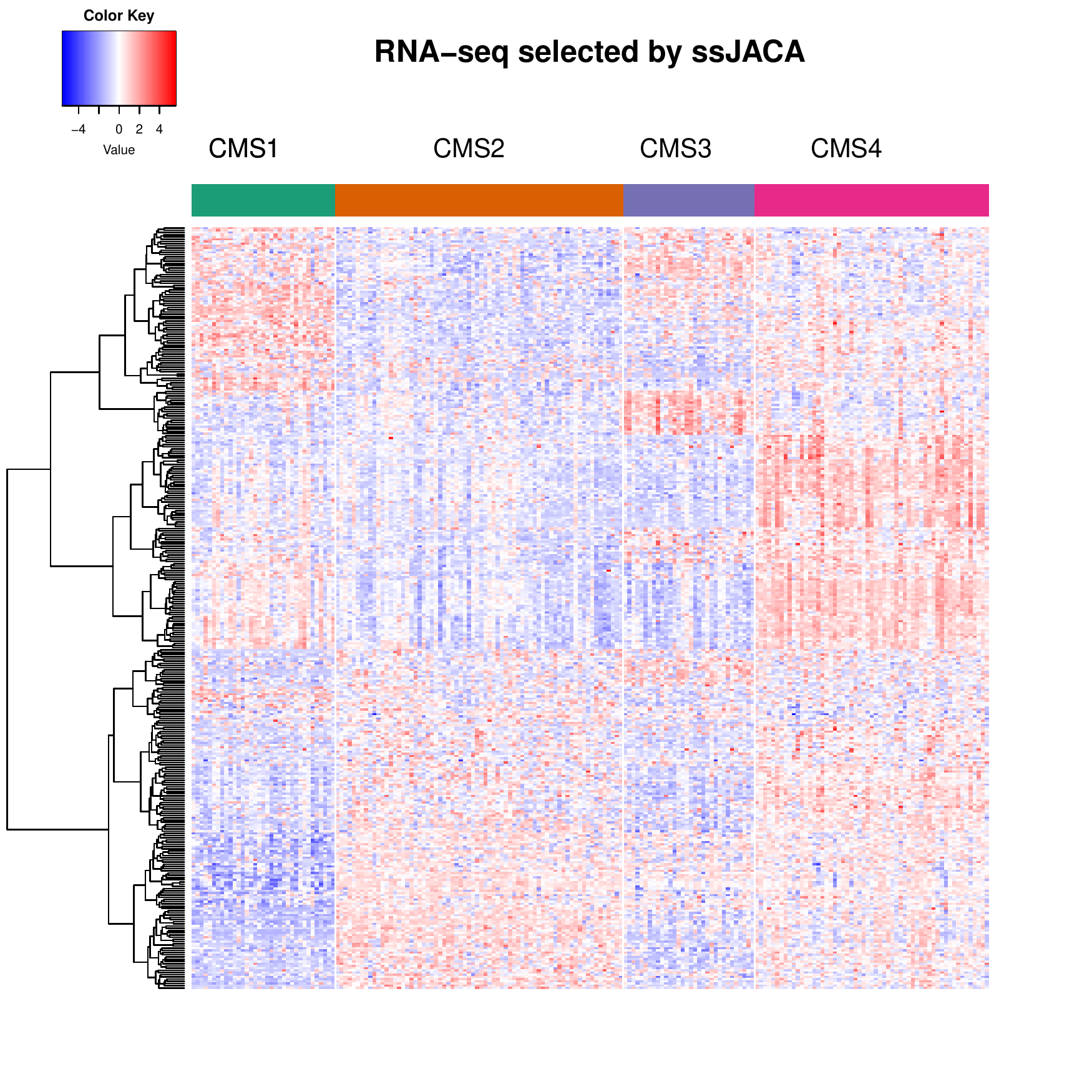} 
\includegraphics[scale=.42]{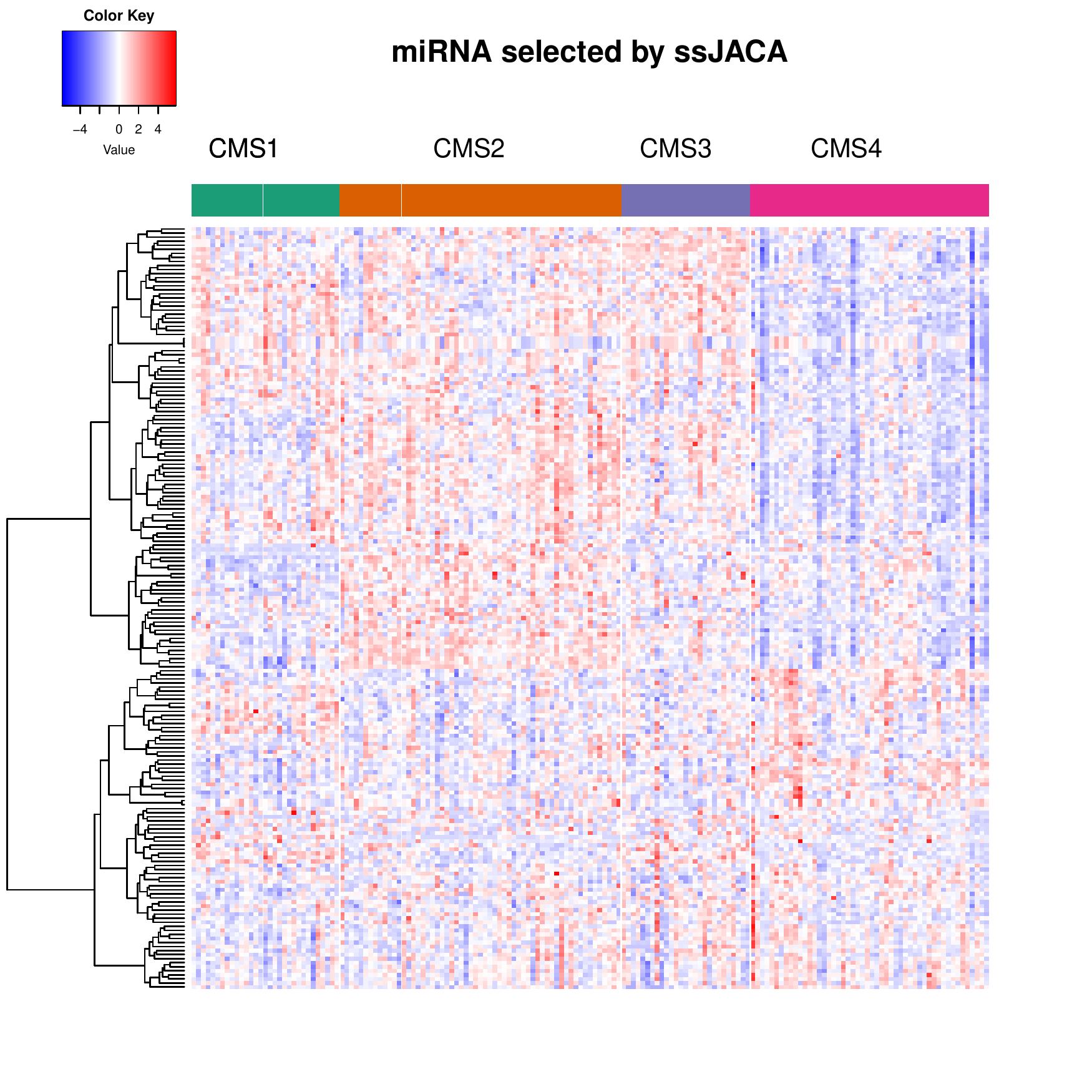} }
\caption{ \label{fig:heatmap}Heatmaps of RNAseq and miRNA views from COAD data based on features selected by ssJACA. We use Ward's linkage with euclidean distances for feature ordering.}
\end{figure}

\section{Discussion}
\label{sec:discussion}

In this work, we develop a joint framework for classification and association analysis of multi-view data by exploring the connections between linear discriminant analysis and canonical correlation analysis.
Our main objective is not to merely improve the prediction accuracy, but to find a low-dimensional representation of data that is coherent across the views and also relevant to the subtypes. A particular advantage of our approach is that it allows to use both samples with missing class labels and samples with missing subset of views. Nevertheless, there are several parts of the method that require further investigation. First, the trade off between classification and association criteria in~\eqref{eq:colP} is controlled by the parameter $\alpha$. We conducted  empirical studies that suggest the results are robust across a wide range of $\alpha$. While $\alpha\in [0.5, 0.7]$ leads to most favorable performance according to our experiments, it would be of interest to investigate whether there is an optimal value from the theoretical perspective. Secondly, we treat all views equally within our framework, however in practice some views may have stronger associations with class membership as well as with each other. This scenario can be addressed by adding view-specific weights within~\eqref{eq:colP}, however it is unclear how to choose the weights in practice. Finally, we focused on row-sparse structure via group-lasso penalty, however the method could be used with other structured penalties depending on the problem of interest.



\section*{Acknowledgements}
This work was supported by NSF DMS-1712943. The authors are grateful to Jeffrey Morris for the helpful discussion of consensus molecular subtypes of colorectal cancer. 

\begin{appendix}
\section{Proofs of the main results in the paper }\label{sec:proof}

\begin{proof}[Proof of Theorem~\ref{p:factor}] 1. Under the stated conditions, $\bx_d$ in~\eqref{eq:factor1} satisfies~\eqref{eq:condX12} by construction, therefore it remains to show the reverse. Consider~\eqref{eq:condX12} with $\bSigma_{ldy}={\bf{0}}$. Then
$$
\bx_d = \bmu_d + \sum_{k=1}^K(\bmu_{dk}-\bmu_d)\Ind\{y=k\} + \bSigma_{dy}^{1/2}\be_d,
$$
where $\be_d$ are independent from $y$. We next show that there exists function $f:\{1,2\dots,K\}\to \R^{K-1}$ such that $\bmu_d + \sum_{k=1}^K(\bmu_{dk}-\bmu_d)\Ind\{y=k\} = \bmu_d + \bDelta_d f(y) = \bmu_d + \bDelta_d \bu_y$ with $\bu_y$ and $\bDelta_d$ satisfying the stated conditions.

Consider $K=2$. Let $u_y = f(y) = \sqrt{\pi_2/\pi_1}\Ind\{y=1\} - \sqrt{\pi_1/\pi_2}\Ind\{y=2\}$, then $\E(u_y)=0$, $\cov(u_y)=1$.  Setting $\bDelta_d = \sqrt{\pi_1\pi_2}(\bmu_1-\bmu_2)$ gives the desired factor model since $$\E(\bmu_d + \bDelta_Du_y|y=1) = \pi_1\bmu_1 + \pi_2\bmu_2 + \sqrt{\pi_1\pi_2}(\bmu_1 - \bmu_2) \sqrt{\pi_2/\pi_1} = \bmu_1,$$ and similarly $\E(\bmu_d + \bDelta_Du_y|y=2) = \bmu_2$.

Consider $K\geq 2$. Let $\bTheta \in \R^{K \times (K-1)}$ have columns $\bTheta_l$ with
$$
\bTheta_l = \Bigg(\Bigg\{\sqrt{\frac{\pi_{l+1}}{\sum_{i=1}^l\pi_i\sum_{i=1}^{l+1}\pi_i}}\Bigg\}_l,-\sqrt{\frac{\sum_{i=1}^l\pi_i}{\pi_{l+1}\sum_{i=1}^{l+1}\pi_i}}, 0_{K-1-l}\Bigg)^{\top},
$$
and let $\bZ = g(y)\in \R^K$ be a unit norm class-indicator random vector with $z_{k}=1$ if observation belongs to class $k$. Consider $\widetilde \bu_y = \widetilde f(y) = \bTheta^{\top}g(y) = \bTheta^{\top}\bZ$ and let $\bpi = (\pi_1\dots \pi_K)^{
\top}$. Then
\begin{align*}
\E(\widetilde \bu_y) &= \bTheta^{\top}\E(\bZ) = \bTheta^{\top}\bpi = (\bTheta_1^{\top}\bpi\dots \bTheta_{K-1}^{\top}\bpi)={\bf0}_{K-1},\\
\cov(\widetilde \bu_y) &= \bTheta^{\top}\cov(\bZ)\bTheta = \bTheta^{\top}(\diag(\bpi) - \bpi\bpi^{\top})\bTheta =\bTheta^{\top}\diag(\bpi) \bTheta = \bI_{K}.
\end{align*}
Next define $\widetilde \bDelta_d\in \R^{p \times (K-1)}$ to have columns $\widetilde \bDelta_{dr}$ with
$$
\widetilde \bDelta_{dr} = \frac{\sqrt{\pi_{r+1}}\left\{\sum_{i=1}^r\pi_i(\bmu_{di} - \bmu_{d(r+1)})\right\}}{\sqrt{\sum_{i=1}^r\pi_i\sum_{i=1}^{r+1}\pi_i}}.
$$
Then
\begin{align*}
\E(\bmu_d + \widetilde \bDelta_d\widetilde f(y)|y = k)
=& \sum_{m=1}^K\pi_m\bmu_{dm} + \widetilde \bDelta_d\bTheta^{\top}g(k)
\\
=&\sum_{m=1}^K\pi_m\bmu_{dm} - \sqrt{\frac{\sum_{i=1}^{k-1}\pi_i}{\pi_{k}\sum_{i=1}^{k}\pi_i}}\widetilde \bDelta_{d(k-1)} + \sum_{l=k}^{K-1}\sqrt{\frac{\pi_{l+1}}{\sum_{i=1}^l\pi_i\sum_{i=1}^{l+1}\pi_i}}\widetilde \bDelta_{dl}\\
=&\bmu_{dk},
\end{align*}
where in the last step we used the properties of orthogonal group-mean contrasts for unbalanced data, see \citet{Searle:2006ww} and also Proposition~2 in \citet{Gaynanova:2016wk}.
Consider the eigendecomposition
$
\widetilde \bDelta_d^{\top}\bSigma_{dy}^{-1}\widetilde \bDelta_d = \bR_d\bLambda_d \bR_d^{\top}.
$
Setting $\bDelta_d = \widetilde \bDelta_d \bR_d$ and $\bu_y = \bR_d^{\top}\widetilde \bu_y$ leads to desired factor model.\\

2. Consider
$
\bSigma_{ld} = \bDelta_l\bDelta_d^{\top} = \bSigma_l \Big[ \bSigma_l^{-1}\bDelta_l\bDelta_d^{\top}\bSigma_d^{-1} \Big]\bSigma_d.
$ 
Let $\bLambda_d = \bDelta_d^{\top}\widetilde \bSigma_{d}^{-1}\bDelta_d$, where $\bLambda_d$ is diagonal by definition of factor model~\eqref{eq:factor2}. Using Woodbury matrix identity,
$$
\bDelta_d^{\top}\bSigma_d^{-1}\bDelta_d = \bDelta_d^{\top}(\widetilde \bSigma_d + \bDelta_d\bDelta_d^{\top})^{-1}\bDelta_d = \bLambda_d^{1/2} (\bI + \bLambda_d)^{-1}\bLambda_d^{1/2}.
$$
Let $\bTheta_d = \bSigma_d^{-1}\bDelta_d \bLambda_d^{-1/2}(\bI+\bLambda_d)^{1/2}$, then $\bTheta_d^{\top}\bSigma_d\bTheta_d = \bI$,
and 
$$
\bSigma_{ld} = \bSigma_l \left[ \bTheta_l(I+\bLambda_l)^{-1/2}\bLambda_l^{1/2}\bLambda_d^{1/2}(I+\bLambda_d)^{-1/2}\bTheta_d^{\top} \right]\bSigma_d = \bSigma_l\left(\sum_{k=1}^{K-1} \rho_k \btheta_{l}^{(k)}\btheta_{d}^{(k)\top}\right)\bSigma_d,
$$
where $\rho_k$ are the diagonal elements of matrix $(\bI+\bLambda_l)^{-1/2}\bLambda_l^{1/2}\bLambda_d^{1/2}(\bI+\bLambda_d)^{-1/2}$, and $\btheta_l^{(k)}$, $\btheta_d^{(k)}$ are corresponding columns of $\bTheta_l$, $\bTheta_d$.

\end{proof}

\begin{proof}[Proof of Theorem~\ref{t:sigma12}] Consider
\begin{align*}
\bSigma_{ld} &= \bA_l\bA_d^{\top}+ \bDelta_l\bDelta_d^{\top} \\
&=  \bSigma_l\left\{\bSigma_l^{-1/2}\Big(\bSigma_l^{-1/2}\bA_l\bA_d^{\top}\bSigma_d^{-1/2} + \bSigma_l^{-1/2}\bDelta_l\bDelta_d^{\top}\bSigma_d^{-1/2}\Big)\bSigma_d^{-1/2}\right\}\bSigma_d\\
& = \bSigma_l\left\{\bSigma_l^{-1/2}\Big(\bR_q\bD_q\bP_q^{\top} + \bR_{K-1}\bD_{K-1}\bP_{K-1}^{\top}\Big)\bSigma_d^{-1/2}\right\}\bSigma_d,
\end{align*}
where we used singular value decomposition $\bSigma_l^{-1/2}\bA_l\bA_d^{\top}\bSigma_d^{-1/2} = \bR_q\bD_q\bP_q^{\top}$ and $\bSigma_l^{-1/2}\bDelta_l\bDelta_d^{\top}\bSigma_d^{-1/2} =\bR_{K-1}\bD_{K-1}\bP_{K-1}^{\top}$. Since $\bA_d^{\top}\widetilde \bSigma_d^{-1}\bDelta_d = \bf0$, by Woodbury matrix identity  $\bA_d^{\top}\bSigma_d^{-1}\bDelta_d = \bf0$, and therefore $\bR_q^{\top}\bR_{K-1} = \bf0$ and $\bP_q^{\top}\bP_{k-1} = \bf0$. From the above display,
$$
\bSigma_{ld} = \bSigma_l\left\{\bSigma_l^{-1/2}\bR_{q+K-1}\bD_{q+K-1}\bP_{q+K-1}^{\top}\bSigma_d^{-1/2}\right\}\bSigma_d.
$$
The result follows by setting $\bTheta_d = \bSigma_d^{-1/2}\bP_{q+K-1}$, and using the results from the proof of Theorem~\ref{p:factor}.
\end{proof}

\begin{proof}[Proof of Lemma~\ref{l:theoreticalResults}]
Using the definition of augmented $\bX'$ and $\bY'$, 
$$ \bG = \E(\bX'^\top \bX')=
\bpm\bSigma_1 & -\frac{1-\alpha}{D-1}\bSigma_{12} & \cdots & -\frac{1-\alpha}{D-1}\bSigma_{1D}\\
-\frac{1-\alpha}{D-1}\bSigma_{21} & \bSigma_2 &\cdots&-\frac{1-\alpha}{D-1}\bSigma_{2D}\\
&\vdots&\\
-\frac{1-\alpha}{D-1}\bSigma_{D1}&-\frac{1-\alpha}{D-1}\bSigma_{D2}&\cdots&\bSigma_D\epm/D.$$

By multiplying the covariance matrix $\bG$ on both sides of $\bW_d^*\bR_d^\top\propto\bSigma_d^{-1}\bDelta_d$ (here $\propto$ means equal up to column-scaling), it remains to show that  for some orthogonal matrices $\bR_d$, $\widetilde\bDelta\diag(\bR_1\cdots,\bR_D)^\top\propto\bG\diag(\bSigma)^{-1}\bDelta$, where $\diag(\bSigma)^{-1}=\diag(\bSigma_1^{-1},\cdots,\bSigma_D^{-1})$. Expending the right hand side leads to
\begin{equation*}
\begin{split}
    \bG\diag(\bSigma)^{-1}\bDelta&=
    \bpm \bI&-\frac{1-\alpha}{D}\bSigma_{12}\bSigma_2^{-1}&\cdots&
    -\frac{1-\alpha}{D}\bSigma_{1D}\bSigma_D^{-1}\\
    &\vdots\\
    -\frac{1-\alpha}{D}\bSigma_{D1}\bSigma_1^{-1}&
    -\frac{1-\alpha}{D}\bSigma_{D2}\bSigma_2^{-1}&\cdots&
    \bI\epm\bpm\bDelta_1\\\vdots\\\bDelta_D\epm\\
    &=\bpm \bDelta_1-\frac{1-\alpha}{D}\sum_{d\neq 1}\bSigma_{1d}\bSigma_d^{-1}\bDelta_d\\ \vdots\\
    \bDelta_D-\frac{1-\alpha}{D}\sum_{d\neq D}\bSigma_{Dd}\bSigma_d^{-1}\bDelta_d\epm.
    \end{split}
\end{equation*}
From the factor model decomposition~\eqref{eq:factor2}, $\bA_d^\top\bSigma_d^{-1}\bDelta_d=0$ holds, and hence
\begin{equation*}
\bSigma_{ld}\bSigma_d^{-1}\bDelta_d=\bDelta_l\bDelta_d^\top\bSigma_d^{-1}\bDelta_d=\bDelta_l\bLambda_d(\bLambda_d+\bI)^{-1},
\end{equation*}
where $\bDelta_d^\top\bSigma_{dy}^{-1}\bDelta_d=\bLambda_d$. 
It follows that
$$\bDelta_l-\frac{1-\alpha}{D}\sum_{d\neq l}\bSigma_{ld}\bSigma_d^{-1}\bDelta_d = \bDelta_l-\frac{1-\alpha}{D}\sum_{d\neq l}\bDelta_l\bLambda_d(\bLambda_d+\bI)^{-1}\propto \bDelta_l.$$
Choosing $\bR_d$ as an orthogonal matrix such that $\widetilde\bDelta_d\bR_d^\top=\bDelta_d$ completes the proof.
\end{proof}

\begin{proof}[Proof of Theorem~\ref{t:fast_prob_p}]
Consider the concatenated
$\widetilde \bX = \bpm \bX_1&\bX_2&\cdots&\bX_D\epm.$
From  Lemmas~3 and 7 in \citet{Gaynanova:2019pe}, with probability at least $1-\eta$ and some constant $C$
$$
\|\frac1n\widetilde \bX^{\top}\widetilde \bX-\bSigma_T\|_{\infty} \leq C\tau^2\sqrt{\frac{\log(\sum_{d=1}^Dp_{d}\eta^{-1})}{n}},
$$
where $\tau=\max_j\sqrt{\sigma_j^2+\max_k\mu_{kj}^2}$, $\sigma_j$ are  diagonal elements of $\bSigma_{y}$ and $\mu_{k,j}$ are elements of $\bmu_k$. Therefore, with probability at least $1-\eta$
$$\|\bG-\bX'^\top \bX'\|_\infty\leq\frac1D\|\frac1n\widetilde \bX^{\top}\widetilde \bX-\bSigma_T\|_{\infty}\leq \frac{C\tau^2}{D}\sqrt{\frac{\log(\sum_{d=1}^Dp_{d}\eta^{-1})}{n}}. $$
From Lemma~\ref{l:REgroup}, if $s_d\leq\gamma\lambda^2_{min}(32D\lambda^2_d\|\bG - \bX'^{\top}\bX'\|_{\infty})^{-1}$, 
 then  $\bX'$ satisfies $RE(S,3, \blambda)$ and $\gamma \leq 2 \gamma$. Hence, using $\lambda_d = C\left(\tau\vee \tau^2\delta g\right)D^{-1}\sqrt{{(K-1)\log[(K-1)p_d ]}/{n}}$, Assumption~\ref{a:sample} and the condition $s_d^2{\log[(K-1)p_{d}]}=o(n)$ leads to $s_d\leq\gamma\lambda^2_{min}(32D\lambda^2_d\|\bG - \bX'^{\top}\bX'\|_{\infty})^{-1}$. Therefore, by Theorems~\ref{t:determine} and~\ref{t:Xtepsilon}
\begin{align*}
\|\widehat \bW - \bW ^*\|_F &= O_p\left(\left(\tau\vee \tau^2\delta g\right)\frac1{D\gamma}\sqrt{\frac{K-1}{n}\sum_{d=1}^Ds_d\log[(K-1)p_{d}]}\right).
\end{align*}
\end{proof}

\begin{proof}[Proof of Proposition~\ref{p:lambda_max_d}]
By the KKT conditions \eqref{eq:kkt}, $\bW_d=0$ leads to $\bX'^{\top}_{dj}\bY'= \lambda \bu_{dj}$, hence by the definition of subgradient
$\left\|\bX'^{\top}_{dj}\bY'\right\|_2=\left\|\left(\frac{\alpha \bX_d^\top\widetilde \bY}{nD}\right)_j\right\|_2= \lambda \|\bu_{dj}\|_2\leq  \lambda$.
This implies that $\bW_d = 0$ satisfies KKT conditions whenever $\lambda \geq \alpha(nD)^{-1}\|\bX_d^\top\widetilde \bY\|_{\infty, 2}$.
\end{proof}

\section{Supporting  Theorems and Lemmas}\label{sec:supThm}
\begin{lemma}\label{l:cone} Let $\bphi_d=\frac{\alpha}{nD}\bX_d^\top\left(\widetilde \bY - \bX_d\bW^*_d\right)+ \frac{1-\alpha}{nD(D-1)}\sum_{j\neq d}\left(\bX_j\bW^*_j-\bX_d\bW^*_d\right)$, and let $\widehat \bW$ be the solution to~\eqref{eq:augmented} with $\lambda_d \geq 2\|\bX_d^{\top}\bphi_d\|_{\infty,2}$. Let $\bH = \widehat \bW-\bW^*$, and  $S$ as defined in Assumption~\ref{a:sparsity}, then $\bH \in C(S, \lambda)$.
\end{lemma}
\begin{proof}
Consider the KKT conditions for~\eqref{eq:augmented}
$$
0 = -\bX'^{\top}(\bY'- \bX'\widehat \bW) + \widehat \bs,
$$
where $\widehat \bs_{dj} \in \partial (\lambda_d \|\bw_{dj}\|_2)$ evaluated at $\widehat \bW$. Multiplying $(\bW^*-\widehat \bW)^{\top}$ on both sides gives  
\begin{equation*}
    (\bW^*-\widehat \bW)^{\top}\left(\bX'^{\top}(\bY'- \bX'\widehat \bW) - \widehat \bs\right) = 0.
\end{equation*}
Let $\bPsi=\bY'-\bX'\bW^*$. Replacing $\bY'$ with $\bY' + \bX'\bW^*-\bX'\bW^*$ and using properties of subgradient of convex functions leads to
\begin{equation*}
    \|\bX'(\bW^*-\widehat \bW)\|_F^2\leq \langle \bX'^{\top}\bPsi, \widehat \bW-\bW^*\rangle  + \sum_{d=1}^D\lambda_d \|\bW_d^*\|_{1,2} - \sum_{d=1}^D\lambda_d \|\widehat \bW_d\|_{1,2}.
\end{equation*}
Since $\langle \bX'^{\top}\bPsi, \widehat \bW-\bW^*\rangle=\sum_{d=1}^D \langle \bX_d^{\top}\bphi_d, \widehat \bW_d - \widehat \bW^*_d\rangle$, applying  H\"older inequality twice and using conditions on $\lambda_d$ leads to
\begin{align*}
    \|\bX'(\bW^*-\widehat \bW)\|_F^2&\leq\sum_{d=1}^D \langle \bX_d^{\top}\bphi_d, \widehat \bW_d -  \bW^*_d\rangle+ \sum_{d=1}^D\lambda_d \|\bW_d^*\|_{1,2} - \sum_{d=1}^D\lambda_d \|\widehat \bW_d\|_{1,2}\\
    &\leq\sum_{d=1}^D \|\bX_d^{\top}\bphi_d\|_{\infty,2}\|\widehat \bW_d -  \bW^*_d\|_{1,2}+ \sum_{d=1}^D\lambda_d \|\bW_d^*\|_{1,2} - \sum_{d=1}^D\lambda_d \|\widehat \bW_d\|_{1,2}\\
    &\leq \sum_{d=1}^D\frac{{\lambda_d}}2 \left(\|\bH_{d,S_d}\|_{1,2}+\|\bH_{d,S^c_d}\|_{1,2}\right)+\sum_{d=1}^D\lambda_d \|\bW^*_d\|_{1,2} - \sum_{d=1}^D\lambda_d \|\widehat \bW_d\|_{1,2}.
\end{align*}
Since 
\begin{align*}
\|\widehat \bW_d\|_{1,2} &= \|\bW^*_d + \widehat \bW_d-\bW^*_d\|_{1,2} = \|\bW^*_{d,S_d} + \bH_{d,S_d}\|_{1,2} + \|\bH_{d, S_d^c}\|_{1,2}\\
&\geq \|\bW^*_{d,S_d}\|_{1,2} - \|\bH_{d,S_d}\|_{1,2} + \|\bH_{d, S_d^c}\|_{1,2},
\end{align*}
combining the above two displays gives

\begin{align}\label{eq:XWUpperBound}
    \|\bX'(\bW^*-\widehat \bW)\|_F^2 \leq \sum_{d=1}^D\frac{3}{2}\lambda_d\|\bH_{d,S_d}\|_{1,2} - \sum_{d=1}^D\frac1{2}\lambda_d\|\bH_{d,S_d^c}\|_{1,2}.
\end{align}
Since $\|\bX'(\bW^*-\widehat \bW)\|_F^2 \geq 0$, the statement follows.
\end{proof}

\begin{thm}\label{t:determine} Let $\widehat \bW$ be the solution to~\eqref{eq:augmented} with $\lambda_d \geq 2\|\bX_d^{\top}\bphi_d\|_{\infty,2}$, where $\bphi_d$ are defined in Lemma~\ref{l:cone}. Under Assumption~\ref{a:sparsity}, if $\bX'$ satisfies  $\textrm{RE}(S,  \blambda)$ with $\gamma = \gamma(S,  \blambda, \bX')$, then
$$\|\widehat \bW - \bW^*\|_F\leq\frac3{2\gamma}\sqrt{\sum_{d=1}^D\lambda_d^2s_d}.
$$

\end{thm}
\begin{proof}
From equation~\eqref{eq:XWUpperBound}, using $\bH = \widehat \bW-\bW^*$, 
\begin{align*}
      \|\bX'(\bW^*-\widehat \bW)\|_F^2\leq \sum_{d=1}^D\frac{{3\lambda_d}}2 \|\bW^*_{d,S_d}-\widehat \bW_{d,S_d}\|_{1,2}\leq \frac32\sum_{d=1}^D\lambda_d\sqrt{s_d}\|\bW^*_{d,S_d}-\widehat \bW_{d,S_d}\|_F.
\end{align*}
Applying Cauchy-Schwartz inequality gives
\begin{align*}
    \|\bX'(\bW^*-\widehat \bW)\|_F^2 \leq \frac32\sqrt{\sum_{d=1}^D\lambda_d^2s_d}\|\bW_S^*-\widehat \bW_S\|_F.
\end{align*}
Since $\bX'$ satisfies  $\textrm{RE}(S,  \blambda)$ and $\bH\in C(S,  \blambda)$, by Lemma~\ref{l:cone}
\begin{align*}
   \|\bW^*-\widehat \bW\|^2_F &\leq \frac1{\gamma}\|\bX'(\bW^*-\widehat \bW)\|_2^2 \leq\frac1{\gamma}\frac32\sqrt{\sum_{d=1}^D\lambda_d^2s_d}\|\bW_S^*-\widehat \bW_S\|_F\\
   &\leq \frac1{\gamma}\frac32\sqrt{\sum_{d=1}^D\lambda_d^2s_d}\|\bW^*-\widehat \bW\|_F.
\end{align*}
If $\|\bW^*-\widehat \bW\|_F^2= 0$, the bound holds trivially. Otherwise, dividing by $\|\bW^*-\widehat \bW\|_F$ on both sides leads to the desired bound.
\end{proof}

\begin{thm}\label{t:Xtepsilon} Under Assumptions~\ref{a:p}--\ref{a:sample}, there exists $C>0$ such that
$$
\|\bX_d^{\top}\bphi_d\|_{\infty,2} \leq C\left(\tau\vee \tau^2\delta g\right)\frac1D\sqrt{\frac{(K-1)\log((K-1)p_d \eta^{-1})}{n}},\quad d=1,\dots,D,$$
with probability at least $1-\eta$, where $\bphi_d$ are from Lemma~\ref{l:cone}.
\end{thm}
\begin{proof}
Without loss of generality, consider $d=1$ and let $\widetilde\bDelta=\bpm \widetilde\bDelta_1^\top & \widetilde\bDelta_2^\top & \cdots& \widetilde\bDelta_D^\top\epm^\top$, where $\widetilde\bDelta_d\in\mathbb{R}^{p_d\times K-1}$. Applying the triangle inequality gives
\begin{align*}
\|&\bX_1^{\top}\bphi_1\|_{\infty,2} \\
&=   \|\frac{\alpha}{nD}\bX_{1}^{\top}(\widetilde \bY- \bX_1\bW_1^*) +  \frac{1-\alpha}{nD(D-1)}\sum_{l\neq 1} \bX_{1}^{\top}(\bX_{l}\bW_l^* - \bX_{1}\bW_1^*)\|_{\infty,2}\\
&=   \|\frac{\alpha}{nD}\bX_{1}^{\top}(\widetilde \bY- \bX_1\bW_1^*) -\widetilde\bDelta_1+\widetilde\bDelta_1+  \frac{1-\alpha}{nD(D-1)}\sum_{l\neq 1} \bX_{1}^{\top}(\bX_{l}\bW_l^* - \bX_{1}\bW_1^*)\|_{\infty,2}\\
&\leq  \underbrace{\|\frac{\alpha}{nD} \bX_{1}^{\top}\widetilde \bY-\widetilde\bDelta_1\|_{\infty,2}}_{:=I_1}\\
&~~+\underbrace{\|\widetilde\bDelta_1- \frac{\alpha}{nD} \bX_{1}^{\top}\bX_1\bW_1^*+  \frac{1-\alpha}{nD(D-1)} \sum_{l\neq 1} \bX_{1}^{\top}(\bX_{l}\bW_l^* - \bX_{1}\bW_1^*)\|_{\infty,2}}_{:=I_2}.
\end{align*}
Consider $I_1$. From Lemma 4 in \citet{Gaynanova:2019pe}, there exists  $C_1>0$ such that
$$\|\frac{\alpha}{nD}\bX_{1}^{\top}\widetilde \bY-\widetilde\bDelta_1\|_{\infty,2}\leq\frac{C_1}{D}\max_j\sigma_{1,j}\sqrt{\frac{(K-1)\log(p_1 \eta^{-1})}{n}}\leq\frac{C_1}{D}\tau\sqrt{\frac{(K-1)\log(p_1 \eta^{-1})}{n}}$$
 with probability at least $1-\eta$.

Consider $I_2$.
\begin{align*}
    I_2 
    &= \Big\|\widetilde\bDelta_1 - \frac1{n}\bX_1^{\top}\Big\{\alpha \frac1{D}\bX_1\bW_1^{*} + \frac{1-\alpha}{D}\bX_1\bW_1^*- \frac{1-\alpha}{D(D-1)}\sum_{l \neq 1}\bX_l\bW_l^*\Big\}\Big\|_{\infty,2}\\
    &=  \Big\|\widetilde\bDelta_1 - \frac1{Dn}\bX_1^{\top}\Big\{\bX_1\bW_1^{*}- \frac{1-\alpha}{D-1}\sum_{l \neq 1}\bX_l\bW_l^*\Big\}\Big\|_{\infty,2}\\
    &= \|\widetilde\bDelta_1 - \frac1{Dn}\bX_1^{\top}
\bU\|_{\infty,2},
\end{align*}
where $\bU=\bX_1\bW_1^{*}- \frac{1-\alpha}{D-1}\sum_{l \neq 1}\bX_l\bW_l^* \in \R^{n \times (K-1)}$. 
Since
the first $p_1$ rows of $\bG$ are $\bpm\bSigma_1 & -\frac{1-\alpha}{D-1}\bSigma_{12} & \cdots & -\frac{1-\alpha}{D-1}\bSigma_{1D}\epm/D$,
\begin{align*}
\E\left(\frac1{Dn}\bX_1^{\top}\bU\right)&=\frac1D\E\left(\frac1n\bX_1^\top\bpm \bX_1 &- \frac{1-\alpha}{D-1}\bX_2&\cdots&- \frac{1-\alpha}{D-1}\bX_D\epm \bW^*\right) \\&=\frac1D\bpm\bSigma_1 & -\frac{1-\alpha}{D-1}\bSigma_{12} & \cdots & -\frac{1-\alpha}{D-1}\bSigma_{1D}\epm \bG^{-1}\widetilde\bDelta \\
&=\bpm I_{p_1}\ {\bf 0}\epm\widetilde\bDelta = \widetilde\bDelta_1.
\end{align*}
Combining the above gives
\begin{align*} I_2&=\left\|\widetilde\bDelta_1 - \frac1{Dn}\bX_1^{\top}\bU\right\|_{\infty,2}\leq\sqrt{K-1}\left\|\widetilde\bDelta_1 - \frac1{Dn}\bX_1^{\top}\bU\right\|_{\infty}\\
&=\sqrt{K-1}\left\|\E\left(\frac1{Dn}\bX_1^{\top}\bU\right) - \frac1{Dn}\bX_1^{\top}\bU\right\|_{\infty}.
\end{align*}
From Lemma 3 in \citet{Gaynanova:2019pe}, all elements of $\bX_1$ are subgaussian with parameter at most $\tau$.
From Lemma~\ref{l:vwbound}, all elements of $\bU$ are subgaussian with parameter at most ${2\tau\delta g}$.
Therefore, by Lemma~\ref{l:crossProdIndBound}, there exist $C_2>0$ such that with probability at least $1-\eta$
\begin{align*}
    I_2\leq C_2\frac{\tau^2\delta g}{D}\sqrt{\frac{(K-1)\log((K-1)p_1 \eta^{-1})}{n}}.
\end{align*}
Combining the results for $I_1$ and $I_2$ leads to the desired bound.
\end{proof}

\begin{lemma}\label{l:vwbound}
Under Assumptions~\ref{a:p}--\ref{a:norm}, all elements of $\bU_d=\bX_d\bW_d^{*}- \frac{1-\alpha}{D-1}\sum_{l \neq d}\bX_l\bW_l^*$, $d=1,\dots,D$, are subgaussian with parameter ${2\tau\delta g}$.
\end{lemma}
\begin{proof}
Without loss of generality, let $d=1$ and $\bV=\bpm \bX_1 &- \frac{1-\alpha}{D-1}\bX_2&\cdots&- \frac{1-\alpha}{D-1}\bX_D\epm \in \R^{n \times \sum_{i=1}^Dp_i}$ so that $\bU_1 = \bU =\bV\bW^*$.
Let $\bv_i $ 
be the $i^{th}$ row of $\bV$. Under Assumptions~\ref{a:p}--\ref{a:norm}, $\bv_i|\by_i=k$
follows normal distribution with
$$
\E \left[\bv_i \Big | \by_i = k \right]= \bP\bmu_k,\Cov\left[ \bv_i \Big | \by_i = k\right] = \bP\bSigma_{y}\bP=\bar\bSigma_y,$$where  $\bP=\diag(\bI_{p_1},- \frac{1-\alpha}{D-1}\bI_{p_2},\dots,- \frac{1-\alpha}{D-1}\bI_{p_D}).$ Therefore, 
\begin{align*}
   \bW^{*\top}\bv_i&=\widetilde\bDelta^\top \bG^{-1}\bv_i\\
   &=\widetilde\bDelta^\top \bG^{-1}(\bP\sum_{k=1}^K \bmu_{k}\Ind\{y_i=k\} + \bar\bSigma_{y}^{1/2}\be_i)\\
   &=\widetilde\bDelta^\top \bG^{-1}\bP\sum_{k=1}^K \bmu_{k}\Ind\{y_i=k\} + \widetilde\bDelta^\top \bG^{-1}\bar\bSigma_{y}^{1/2}\be_i \\
   &:= \bv_{1i} + \bv_{2i},
\end{align*}
where $\be_i\sim \Ncal(\bI)$ and $\bv_{1i}$, $\bv_{2i}$ are independent random vectors.

Let $\bM = (\bmu_1~ \bmu_2~ \cdots ~\bmu_K) \in\mathbb{R}^{\sum_{i=1}^Dp_i\times K}$. Since $\|\bG^{-1}\|_\infty\leq g$, 
\begin{align*}
    \|\bv_{1i}\|_\infty&=\|\widetilde\bDelta^\top \bG^{-1}\bP\sum_{k=1}^K \bmu_{k}\Ind\{y_i=k\} \|_\infty\leq\|\widetilde\bDelta^\top \bG^{-1}\bP\bM\|_{\infty,2}\\
    &\leq \|\widetilde\bDelta\|_{\infty,2}\| \bG^{-1}\|_\infty\|\bP\bM\|_{\infty}\leq {\delta \tau g},
\end{align*} 
where the second inequality is due to $\|\bA\bB\|_{\infty,2}\leq \|\bA\|_{\infty}\|\bB\|_{\infty,2}$ \citep[Lemma~8]{Obozinski:2011ho}. Hence all elements of $\bv_{1i}$  are subgaussian with parameter at most ${\delta \tau g}$.

On the other hand, $\bv_{2i}$ is a normally distributed vector with mean $\bf0$ and covariance $\Cov(\bv_{2i})=\widetilde\bDelta^\top \bG^{-1}\bar\bSigma_{y}\bG^{-1}\widetilde\bDelta.$ Since
\begin{align*}
    \|\Cov(\bv_{2i})\|_\infty &= \| \widetilde\bDelta^\top \bG^{-1}\bar\bSigma_{y}\bG^{-1}\widetilde\bDelta\|_\infty\\
    &\leq\|\widetilde\bDelta\|_\infty^2\|\bG^{-1}\|_\infty^2\|\bSigma_y\|_\infty\|\bP\|_\infty^2\\
    &\leq\|\widetilde\bDelta\|_{\infty,2}^2\|\bG^{-1}\|_\infty^2\|\bSigma_y\|_\infty\leq{\delta^2\tau^2 g^2},
\end{align*}
all elements of $\bv_{2i}$ are also subgaussian with parameter ${\delta\tau g}$. 

Combining the results for $\bv_{1i}$ and $\bv_{2i}$, 
\begin{align*}
\E(e^{\lambda u_{ij}}) = \E\{e^{\lambda(v_{1ij} + v_{2ij})}\} = \E(e^{\lambda v_{1ij}})\E(e^{\lambda v_{2ij}})\leq e^{\lambda^2\{{2\tau\delta g}\}/2}.
\end{align*}
This implies that all elements of $\bU_1$ are subgaussian with  parameter ${2\tau\delta g}$.
\end{proof}

\begin{lemma}\label{l:crossProdIndBound}
Let $(\bx_i, \by_i)\in\R^{p}\times \R^{q}$ be independent identically distributed pairs of mean zero random vectors with 
$\E(\bx_i\by_i^\top)=\bSigma_{xy}$, and let all elements of $\bx_i$ and $\by_i$ be sub-gaussian with parameters $\tau_1$ and $\tau_2$, respectively. Let $\bX = [\bx_1 \dots \bx_n]^{\top}$, $\bY = [\by_1 \dots \by_n]^{\top}$
If $\log(pq) = o(n)$, then with probability at least $1-\eta$ for some constant $C>0$ 
$$
\left\|\frac1n\bX^{\top}\bY - \bSigma_{xy}\right\|_{\infty}\leq C\tau_1\tau_2\sqrt{\frac{\log (pq/\eta)}{n}}.
$$
\end{lemma}
\begin{proof}
Let $u_{ikj}=x_{ji}y_{jk}$, then $u_{ikj}$ is sub-exponential with parameter $2\tau_1\tau_2$ \citep[Lemma~5.14]{Vershynin:2010vk}. Let $\sigma_{ik}$ be elements of $\bSigma_{xy}$, then $u_{ikj}-\sigma_{ik}$ are sub-exponential with parameter $4\tau_1\tau_2$, and using Bernstein's bound \citep[Proposition~5.16]{Vershynin:2010vk}
 $$ \pr\Big(\Big\|\frac1n\sum_{j=1}^nu_{ikj}- \sigma_{ik}\Big\|_{\infty} \geq \varepsilon\Big)\leq 2\exp\Big\{-C\min(\frac{\varepsilon^2}{16\tau_1^2\tau_2^2}, \frac{\varepsilon}{4\tau_1\tau_2})n\Big\}$$
for some $C>0$. By union bound
$$ \pr(\|\bX^{\top}\bY/n - \bSigma\|_{\infty} \geq \varepsilon) \leq pq\pr\Big(\Big\|\frac1n\sum_{j=1}^nu_{ikj}- \sigma_{ik}\Big\|_{\infty} \geq \varepsilon\Big).$$
Setting $\varepsilon = C_1\tau_1\tau_2\sqrt{\frac{\log (pq/\eta)}{n}}$ and using $\log(pq)=o(n)$ completes the proof.
\end{proof}

\begin{lemma}\label{l:REgroup} Let $\bG^{1/2}$ satisfy $\textrm{RE}(S,  \blambda)$ with $\gamma = \gamma(S,\blambda,\bG^{1/2})$, and let $\lambda_{min}:=\min_{d=1,\dots,D}\lambda_d.$ If $s_d \leq \gamma\lambda^2_{min}(32D\lambda^2_d\|\bG - \bX'^{\top}\bX'\|_{\infty})^{-1}$ , then $\bX'$ satisfies $RE(S, \blambda)$ and 
$$
0 < \gamma(S,\blambda, \bX')\leq 2\gamma(S,\blambda, \bG^{1/2}).
$$
\end{lemma}
\begin{proof}[Proof of Lemma~\ref{l:REgroup}] Since $\bG^{1/2}$ satisfies $\textrm{RE}(S,  \blambda)$, for all $\bA\in \mathcal{C}(S,\blambda)$ 
\begin{align*}
   \Tr(\bA^{\top}\bX'^{\top}\bX'\bA) = \Tr(\bA^{\top}\bG\bA) + \Tr\{\bA^{\top}(\bG-\bX'^{\top}\bX')\bA\} \geq {\gamma}\|\bA\|^2_F - \|\bA\|_{1,2}^2\|\bG-\bX'^{\top}\bX'\|_{\infty}.
\end{align*}
Since $\bA\in \mathcal{C}(S,\blambda)$, we have
\begin{align*}
\|\bA\|_{1,2}&\leq\sum_{d=1}^D\frac{\lambda_d}{\lambda_{min}} (\|\bA_{d,S_d}\|_{1,2} + \|\bA_{d,S_d^c}\|_{1,2})\\
    &\leq4\sum_{d=1}^D\frac{\lambda_d}{\lambda_{min}} \|\bA_{d,S_d}\|_{1,2} \leq 4\sum_{d=1}^D\sqrt{s_d}\frac{\lambda_d}{\lambda_{min}}\|\bA_{d,S_d}\|_{F} \\&
    \leq 4\sqrt{\frac{\sum_{d=1}^D\lambda^2_ds_d}{\lambda^2_{min}}} \|\bA_{S}\|_F \leq 4\sqrt{\frac{\sum_{d=1}^D\lambda^2_ds_d}{\lambda^2_{min}}} \|\bA\|_F, 
\end{align*}

Therefore
\begin{align*}
\Tr(\bA^{\top}\bX'^{\top}\bX'\bA) 
 &\geq {\gamma}\|\bA\|^2_F - 16\frac{{\sum_{d=1}^D\lambda^2_ds_d}}{\lambda^2_{min}} \|\bA\|_F^2 \|\bG-\bX'^{\top}\bX'\|_{\infty} \\
 &\geq {\gamma}\|\bA\|^2_F - \frac\gamma{2}\|\bA\|^2_F = \frac\gamma{2}\|\bA\|_F^2,
\end{align*}
where  the last inequality holds because of the condition on $s_d$. 
\end{proof}

\section{Regression formulation via augmented data approach}\label{sec:FullAugObj}
In this section, we reformulate~\eqref{eq:colP} as a regression problem using augmented data approach. Let $
\bW = ( \bW_1^\top,\dots, \bW_D^\top)^\top,$
\begin{align*}
\bX' &= \bpm
\sqrt{\alpha}\bX_1 & \bf{0}&\bf{0}&\dots&\bf{0}\\
\bf{0} & \sqrt{\alpha}\bX_2&\bf{0}&\dots&\bf{0}\\
&&\vdots\\
\bf{0} & \bf{0}&\bf{0}&\dots&\sqrt{\alpha}\bX_D\\
\sqrt{\frac{1-\alpha}{D-1}}\bX_1 & -\sqrt{\frac{1-\alpha}{D-1}}\bX_2&\bf{0}&\dots&\bf{0}\\
\sqrt{\frac{1-\alpha}{D-1}}\bX_1 & \bf{0}&-\sqrt{\frac{1-\alpha}{D-1}}\bX_3&\dots&\bf{0}\\
\sqrt{\frac{1-\alpha}{D-1}}\bX_1 & \bf{0}&\bf{0}&\dots&-\sqrt{\frac{1-\alpha}{D-1}}\bX_D\\
&&\vdots\\
\bf{0} & \bf{0}&\dots&\sqrt{\frac{1-\alpha}{D-1}}\bX_{D-1}&-\sqrt{\frac{1-\alpha}{D-1}}\bX_D\\
\epm/\sqrt{nD},\\
\bY'&=\sqrt{\frac{\alpha}{nD}}\left (
\begin{array}{rrrrrr}
\undermat{D}{\widetilde \bY^\top& \dots& \widetilde \bY^\top }& \bf{0}&\dots&\bf{0}  \\
\end{array}
\right )^\top.
\end{align*}
Then~\eqref{eq:colP} is equivalent to
\begin{equation*}
\minimize_W\Big\{2^{-1}\|\bY'-\bX'\bW\|^2_F  +\sum_{d=1}^D\lambda_d\Pen(\bW_d)\Big\}.
\end{equation*}

\section{Simulation studies}\label{sec:simu}

We compare the performance of different methods from Section~\ref{sec:COAD} and also consider \textsf{CVR}: Canonical Variate Regression by \citet{Luo:2016tb} as implemented in the corresponding R package \citep{CVR}. For JACA and ssJACA, we set $\alpha = 0.5$ since we aim to improve both classification and association analysis results.

\subsection{Data generation}\label{s:datageneration}
 
We generate the data using factor model~\eqref{eq:factor2}. 
Specifically, given $\widetilde \bSigma_d$, $d=1,\dots,D$, we generate the factor loadings in \eqref{eq:factor2} as follows
\begin{enumerate}
    \item Generate row-sparse matrix $\bB_d\in \R^{p_d \times K-1}$ with $s=10$ non-zero rows. Draw nonzero elements from uniform distribution on $[-2,-1] \cup [1,2]$. Given $c_{d}>0$, 
    rotate and scale $\bB_d$ so that $\bB_d^{\top}\widetilde \bSigma_d \bB_d = \diag(c_{d}^2)$,
    and set $\bDelta_d = \widetilde \bSigma_d \bB_d$. According to Theorem~\ref{t:sigma12}, this sets $K-1$ canonical correlations $\rho_k$ between datasets $d$ and $l$ to be equal to $\rho_k = (c_{d}c_l)/\sqrt{(1+c_d^2)(1+c_l^2)}.$
    \item If $q\neq 0$,  generate $\bM_d\in \R^{p_d\times q}$ with elements from $N(0,1)$, orthogonalize $\bM_d$ with respect to $\bDelta_d$ as
    $
    \bM_d = (\bI -\bP_{\bDelta_d})\bM_d,
    $
    where $\bP_{\bDelta_d}$ is the projection matrix onto column space of $\bDelta_d$.
    For canonical correlation $\rho_k \in (0,1)$, set $c_k = \sqrt{\rho_k/(1-\rho_k)}$, and
   rotate and scale $\bM_d$ so that $\bM_d^{\top}\widetilde \bSigma_d \bM_d = \diag(c_k^2)$. Set $\bA_d = \widetilde \bSigma_d \bM_d $.
\end{enumerate}

We further draw $n$ independent $y$ with $P(y=k) =\pi_k$, $n$ independent $\bu_q$ from $N(0, \bI_q)$, and $n$ independent $\be_1,\dots, \be_d$, each from $N(0, \bI_{p_d})$. We get $n$ replicas $\bX_1, \dots, \bX_d$ according to~\eqref{eq:factor2} with given $\bDelta_d$, $\bA_d$ and $\bmu_d = 0$, $d=1,\dots,D$. By construction, the population discriminant vectors are proportional to $\bB_d$ with corresponding row-sparsity pattern.

\subsection{Evaluation criteria}\label{sec:eval}

We compare the methods in terms of misclassification rate,  strength of association between the views, estimation consistency and variable selection. To compare the classification accuracy, we consider two prediction approaches for each method: prediction based on one view alone out of ($\bX_1,\dots,\bX_d$) using the corresponding subset of canonical vectors, and prediction based on the full concatenated dataset. All predictions are made by linear discriminant analysis model. The misclassification rate of each classifier is calculated as
$$ 
\frac1m\sum_{i=1}^m  \mathds{1}\left\{\text{label}(\bx_i)\neq \text{pred}(\bx_i)  \right\},
$$
where $\bx_is$ are $m$ new samples, $\text{label}(\bx_i)$ denotes the corresponding class membership and $\text{pred}(\bx_i)$ denotes the predicted class membership.

To evaluate the strength of found association between the views, we consider 
\begin{equation*}
\mbox{Sum correlation}(\bW_1,\dots, \bW_d) = \sum_{d=1}^{D-1}\sum_{l=d+1}^D\text{Cor}_{\bSigma}(\bW_d,\bW_l),
\end{equation*}
where 
$$\text{Cor}_{\bSigma}(\bW_d,\bW_l) =  \left(\frac{\Tr( \bW^\top_d\bSigma_{dl}\bW_l\bW_l^\top \bSigma_{dl} \bW_d) }{\sqrt{\Tr(\bW_d^\top  \bSigma_{d}\bW_d)^2}\sqrt{\Tr( \bW_l^\top \bSigma_{l}\bW_l)^2}}\right)^{\frac12},$$
$\bSigma_d$ is the marginal covariance matrix of view $d$, and $\bSigma_{dl}$ is the  marginal cross-covariance matrix of view $d$ and $l$ as in Section~\ref{sec:sCCALDA}.
This criterion is similar to sum correlation in \citet{Gross:2014ux}, however our definition uses population covariance matrices rather than the sample counterparts.

Let $\bTheta_d\propto \widetilde \bSigma_d^{-1}\bDelta_d\in \R^{p_d\times (K-1)}$ be the population matrix of class-specific canonical vectors for view $d$ with $\widetilde\bSigma_d$ as in~\eqref{eq:factor2}, and  $\bW_d$ be the estimated matrix. To evaluate estimation performance, we consider
$$ 
\text{Cor}_{\bSigma}(\bW_d,\bTheta_d)=\left(\frac{\Tr(\bW^\top_d\widetilde\bSigma_{d}\bTheta_d\bTheta_d^\top\widetilde\bSigma_{d}\bW_d)}{\sqrt{\Tr(\bW_d^\top \widetilde\bSigma_{d}\bW_d)^2}\sqrt{\Tr(\bTheta_d^\top\widetilde\bSigma_{d}\bTheta_d)^2}}\right)^{\frac12}
$$
as a measure of similarity between $\bW_d$ and $\bTheta_d$ 
with $\text{Cor}_{\bSigma}(\bW_d,\bTheta_d) = 1$ if and only if $\bW_d$ is equal to $\bTheta_d$ up to scaling and orthogonal transformation, and $\text{Cor}_{\bSigma}(\bW_d,\bTheta_d) = {\bf 0}$ if $\bW_d^{\top}\widetilde\bSigma_{d}\bTheta_d = 0$. We do not use the Frobenius norm considered in Theorem~\ref{t:fast_prob_p} since it is not invariant to column scaling and orthogonal transformation, and hence will make the evaluation positively biased towards our proposed method.

We use precision and recall to compare the methods in terms of variable selection. Let $\bA_d$ be the set of nonzero rows of $\bTheta_d$, and let $\widehat{\bA_d}$ be the set of nonzero rows in $\widehat \bW_d$. Let $\#\{\bA_d\}$ denote the cardinality of $\bA_d$. We define the precision and recall as
$$ \mbox{Precision}(\bW_d) = \frac{\#\{\bA_d\cap \widehat{\bA_d}\}}{\#\{\widehat{\bA_d}\}}\quad\mbox{and}\quad
\mbox{Recall}(\bW_d) = \frac{\#\{\bA_d\cap \widehat{\bA_d}\}}{\#\{\bA_d\}}. $$

\subsection{Two datasets, two groups}\label{sec:2groups}
We set $n_1=160$, $K = 2$, and generate $n_1$ independent  $y\in \{1,2\}$ with $\pi_1 = 0.4$, and pairs $(\bx_1, \bx_2)\in \R^{p_1}\times \R^{p_2}$ with $(p_1, p_2) \in  \{(100,100), (100, 500), (500,500)\}$ following Section~\ref{s:datageneration}. We additionally generate $n_2=100$ samples for ssJACA, and set corresponding class information as missing, so that $n_1 = 160$ samples have complete view and class information, whereas the remaining $n_2 = 100$ samples have information on both views but no class assignment. We train ssJACA on all $n_1+n_2=200$ samples and train other methods on $n_1=160$ complete samples.  We consider autocorrelation structures $\widetilde\bSigma_{1}=(0.8^{|i-j|})_{ij}$, $\widetilde \bSigma_{2}=(0.5^{|i-j|})_{ij}$, 
and set the value of canonical correlation due to shared class as $\rho = 0.8$ by letting $c_1 =c_2 = \sqrt{\rho/(1-\rho)}$ in generating $B_d$ in Section~\ref{s:datageneration}.  We consider the following cases for other shared factors:
\begin{description}
    \item[Case 1:] $q=0$, no shared factors except class $y$;
    \item[Case 2:] $q = 2$ with corresponding values for canonical correlations being $0.6$ and $0.5$;
    \item[Case 3:] $q = 2$ with corresponding values for canonical correlations being $0.9$ and $0.5$.
\end{description}
In Case 2, the leading canonical correlation between the views is due to shared class membership despite the presence of other shared factors, whereas in Case 3 the leading canonical correlation is due to factors independent from class membership. In order to evaluate the misclassification rates, we further generate $10,000$ new samples as test data, and consider 100 replications for each case. 

The results are presented in Tables~\ref{tab:simu_corcase1}--\ref{tab:simu2_corcase3} and Figure~\ref{fig:simu_prerec}. 
Overall, ssJACA gives the best classification and discriminant vectors estimation results. Compared to JACA, it has lower variability across the replications, confirming the advantage of incorporating samples with missing class information in the analysis. Since classification is not the only goal in our project, but rather finding the structures that are coherent across the views and also relevant to the subtypes, we use $\alpha$ to balance the classification and association tasks. Therefore, in some cases, ssJACA compromises the prediction accuracy to significantly improve the sum correlation.  ssJACA also performs the best in terms of sum correlation except for Case 3, where sum correlation for Sparse CCA is stronger. This is not surprising, since in Case 3 the largest canonical correlation is due to the factors independent from class membership. Therefore, the loadings estimated from sparse CCA are almost orthogonal to the true discriminant vectors $\bTheta_d$ as demonstrated by low values of $\Cor_{\bSigma}(\bW_d,\bTheta_d)$. This explanation is also supported by the poor classification results for Sparse CCA in Case 3. In Table~\ref{tab:simu_corcase3}, Sparse CCA achieves around $40\%$ misclassification rate, which is no better than random guessing. ssJACA also achieves the best trade off between precision and recall. ssJACA's precision is second best to SLDA\_joint, but SLDA\_joint has the lowest  recall. ssJACA's recall is comparable to JACA and worse than the recall of sparse CCA methods, but the latter has low values of precision. Finally, CVR is slightly better than Sparse CCA in Case 2 and worse than Sparse CCA in Case 3 in terms of misclassification rates. However, CVR performs worse than JACA and SLDA methods. We conjecture this is likely due to CVR using the logistic model for estimation rather than the factor model~\eqref{eq:factor2}.

\begin{table}[!t]
\center
 \caption{ \label{tab:simu_corcase1}Comparison of misclassification rates of Case $1$ over 100 replications when $D=2$, $K=2$. Standard errors are given in the brackets and the lowest values are highlighted in bold.}

\resizebox{\textwidth}{!}{
 \begin{tabular}{clccccccc}
  \hline
  \hline$(p_1, p_2)$ & Error rate ($\%$) & JACA & ssJACA &  SLDA  sep &SLDA joint&Sparse CCA&Sparse sCCA & CVR \\
  \midrule
(100,100) & $(W_1)$ & 4.496 (0.037) & {\bf 3.255} (0.026) & 4.809 (0.070) & 4.582 (0.044) & 6.376 (0.048) & 6.675 (0.043) & 6.434 (0.189) \\ 
   & $(W_2)$ & 3.168 (0.040) & {\bf 3.111} (0.033) & 3.533 (0.090) & 4.552 (0.127) & 4.069 (0.051) & 4.415 (0.052) & 7.860 (0.415) \\ 
   & $(W_1,W_2)$ & 0.594 (0.011) & {\bf 0.388} (0.007) & 0.729 (0.024) & 0.934 (0.030) & 1.708 (0.016) & 1.862 (0.019) & 2.197 (0.118) \\ 
  (100,500) & $(W_1)$ & 4.299 (0.036) & {\bf 3.363} (0.032) & 4.593 (0.075) & 4.418 (0.042) & 6.286 (0.045) & 6.612 (0.043) & 6.485 (0.220) \\ 
   & $(W_2)$ & {\bf 3.103} (0.050) & 3.729 (0.046) & 3.279 (0.041) & 4.519 (0.107) & 3.955 (0.061) & 6.445 (0.080) & 8.883 (0.418) \\ 
   & $(W_1,W_2)$ & 0.548 (0.013) & {\bf 0.542} (0.009) & 0.695 (0.028) & 0.879 (0.027) & 1.644 (0.019) & 2.283 (0.025) & 2.417 (0.118) \\ 
  (500,500) & $(W_1)$ & 4.513 (0.035) & {\bf 3.127} (0.028) & 4.498 (0.033) & 4.675 (0.041) & 6.044 (0.040) & 7.250 (0.060) & 6.634 (0.167) \\ 
   & $(W_2)$ & {\bf 3.537} (0.042) & {\bf 3.579} (0.045) & 3.764 (0.049) & 4.938 (0.121) & 4.546 (0.047) & 6.713 (0.076) & 8.732 (0.326) \\ 
   & $(W_1,W_2)$ & 0.629 (0.010) & {\bf 0.462} (0.010) & 0.670 (0.012) & 0.953 (0.024) & 1.447 (0.015) & 2.353 (0.029) & 2.408 (0.104) \\ 
   \hline
   \hline
\end{tabular}}
 \end{table} 
 
 \begin{table}[!t]
 \center
 \caption{ \label{tab:simu_corcase2}Comparison of misclassification rates of Case $2$ over 100 replications when $D=2$, $K=2$. Standard errors are given in the brackets and the lowest values are highlighted in bold.}

\resizebox{\textwidth}{!}{
 \begin{tabular}{clccccccc}
  \hline  \hline
  $(p_1, p_2)$ & Error rate ($\%$) & JACA & ssJACA & SLDA  sep &SLDA joint&Sparse CCA&Sparse sCCA & CVR \\
  \midrule
(100,100) & $(W_1)$ & 4.479 (0.038) & {\bf 3.256} (0.026) & 4.785 (0.064) & 4.571 (0.039) & 9.992 (0.798) & 6.920 (0.077) & 8.004 (0.418) \\ 
   & $(W_2)$ & 3.224 (0.041) & {\bf 3.142} (0.034) & 3.687 (0.127) & 4.915 (0.146) & 8.062 (0.873) & 4.675 (0.071) & 8.468 (0.364) \\ 
   & $(W_1,W_2)$ & 0.601 (0.011) & {\bf 0.397} (0.007) & 0.779 (0.041) & 0.972 (0.027) & 5.489 (0.895) & 2.010 (0.037) & 2.558 (0.114) \\ 
  (100,500) & $(W_1)$ & 4.289 (0.034) & {\bf 3.345} (0.034) & 4.616 (0.084) & 4.425 (0.044) & 9.096 (0.831) & 6.990 (0.079) & 8.126 (0.440) \\ 
   & $(W_2)$ & {\bf 3.088} (0.048) & 3.741 (0.046) & 3.264 (0.040) & 4.473 (0.094) & 6.552 (0.890) & 6.542 (0.081) & 9.760 (0.392) \\ 
   & $(W_1,W_2)$ & {\bf 0.543} (0.011) & 0.560 (0.009) & 0.687 (0.025) & 0.876 (0.024) & 4.476 (0.945) & 2.495 (0.045) & 3.029 (0.147) \\ 
  (500,500) & $(W_1)$ & 4.541 (0.040) & {\bf 3.131} (0.027) & 4.493 (0.040) & 4.675 (0.039) & 13.489 (1.364) & 7.307 (0.059) & 7.575 (0.252) \\ 
   & $(W_2)$ & {\bf 3.572} (0.042) & {\bf 3.581} (0.046) & 3.785 (0.043) & 4.975 (0.126) & 12.166 (1.433) & 6.804 (0.079) & 10.151 (0.416) \\ 
   & $(W_1,W_2)$ & 0.632 (0.011) & {\bf 0.465} (0.009) & 0.671 (0.015) & 0.956 (0.026) & 9.658 (1.557) & 2.391 (0.030) & 2.952 (0.135) \\ 
    \hline\hline
  \end{tabular}}
 \end{table}
\begin{table}[!t]
\center
 \caption{ \label{tab:simu_corcase3}Comparison of misclassification rates of Case $3$ over 100 replications when $D=2$, $K=2$. Standard errors are given in the brackets and the lowest values are highlighted in bold.}

\resizebox{\textwidth}{!}{
 \begin{tabular}{clccccccc}
  \hline\hline
 $(p_1, p_2)$ & Error rate ($\%$) & JACA & ssJACA & SLDA  sep &SLDA joint&Sparse CCA&Sparse sCCA & CVR \\
  \hline
(100,100) & $(W_1)$ & 4.428 (0.034) & {\bf 3.293} (0.030) & 4.647 (0.060) & 4.544 (0.040) & 40.189 (0.197) & 12.862 (1.020) & 10.062 (0.511) \\ 
   & $(W_2)$ & 3.295 (0.041) & {\bf 3.237} (0.035) & 3.606 (0.099) & 5.278 (0.154) & 40.446 (0.257) & 11.296 (1.036) & 9.638 (0.424) \\ 
   & $(W_1,W_2)$ & 0.609 (0.011) & {\bf 0.422} (0.010) & 0.746 (0.034) & 1.030 (0.030) & 40.306 (0.217) & 8.862 (1.148) & 2.538 (0.105) \\ 
  (100,500) & $(W_1)$ & 4.298 (0.034) & {\bf 3.406} (0.038) & 4.686 (0.085) & 4.441 (0.044) & 40.268 (0.188) & 10.256 (0.612) & 9.232 (0.448) \\ 
   & $(W_2)$ & {\bf 3.091} (0.049) & 3.821 (0.048) & 3.274 (0.041) & 4.456 (0.102) & 40.453 (0.236) & 8.911 (0.426) & 11.364 (0.454) \\ 
   & $(W_1,W_2)$ & {\bf 0.544} (0.013) & 0.608 (0.012) & 0.723 (0.024) & 0.872 (0.024) & 40.362 (0.201) & 5.908 (0.619) & 3.029 (0.119) \\ 
  (500,500) & $(W_1)$ & 4.537 (0.039) & {\bf 3.131} (0.027) & 4.471 (0.030) & 4.664 (0.039) & 40.583 (0.262) & 8.947 (0.340) & 9.216 (0.483) \\ 
   & $(W_2)$ & {\bf 3.577} (0.042) & {\bf 3.592} (0.046) & 3.799 (0.057) & 5.017 (0.125) & 40.566 (0.255) & 8.404 (0.303) & 10.944 (0.372) \\ 
   & $(W_1,W_2)$ & 0.626 (0.011) & {\bf 0.463} (0.009) & 0.657 (0.011) & 0.960 (0.024) & 40.575 (0.261) & 4.118 (0.346) & 3.067 (0.118) \\ 
    \hline\hline
 \end{tabular}}
 \end{table}
 
 \begin{table}[!t]
 \center
\caption{ \label{tab:simu_sumcor} Comparison of sum correlation over 100 replications when $D=2$, $K=2$. Standard errors are given in the brackets and the highest values are highlighted in bold.}

\resizebox{\textwidth}{!}{
 \begin{tabular}{ccccccccc}
  \hline\hline
 Case & $(p_1, p_2)$  & JACA & ssJACA& SLDA  sep &SLDA joint&Sparse CCA&Sparse sCCA & CVR \\
  \hline
Case 1 & (100,100) & 0.752 (0.001) & {\bf 0.768} (0.001) & 0.744 (0.001) & 0.732 (0.002) & 0.715 (0.001) & 0.708 (0.001) & 0.670 (0.006) \\ 
  &(100,500) & 0.750 (0.001) & {\bf 0.760} (0.001) & 0.743 (0.001) & 0.730 (0.001) & 0.717 (0.001) & 0.685 (0.001) & 0.656 (0.006) \\ 
 & (500,500) & 0.750 (0.001) & {\bf 0.761} (0.001) & 0.747 (0.001) & 0.729 (0.002) & 0.716 (0.001) & 0.677 (0.001) & 0.661 (0.005) \\ 
Case 2&(100,100) & 0.752 (0.001) & {\bf 0.768} (0.001) & 0.742 (0.002) & 0.728 (0.002) & 0.686 (0.006) & 0.704 (0.001) & 0.641 (0.009) \\ 
  &(100,500) & 0.751 (0.001) & {\bf 0.761} (0.001) & 0.743 (0.001) & 0.731 (0.001) & 0.681 (0.011) & 0.682 (0.001) & 0.623 (0.008) \\ 
  &(500,500) & 0.750 (0.001) & {\bf 0.761} (0.001) & 0.748 (0.001) & 0.729 (0.002) & 0.604 (0.021) & 0.676 (0.001) & 0.632 (0.006) \\  
Case 3 &(100,100) & 0.751 (0.001) & 0.768 (0.001) & 0.744 (0.002) & 0.724 (0.002) & {\bf 0.874} (0.000) & 0.715 (0.003) & 0.549 (0.017) \\ 
 & (100,500) & 0.751 (0.001) & 0.761 (0.001) & 0.742 (0.001) & 0.731 (0.001) & {\bf 0.861} (0.000) & 0.684 (0.002) & 0.553 (0.014) \\ 
 & (500,500) & 0.750 (0.001) & 0.761 (0.001) & 0.748 (0.001) & 0.729 (0.002) & {\bf 0.854} (0.000) & 0.675 (0.001) & 0.573 (0.012) \\  
   \hline\hline
\end{tabular}}
\end{table}

\begin{table}[!t]
\center
 \caption{\label{tab:simu2_corcase1}Comparison of estimation correlation of Case $1$ over 100 replications when $D=2$, $K=2$. Standard errors are given in the brackets and the highest values are highlighted in bold.}
\resizebox{\textwidth}{!}{
 \begin{tabular}{ccccccccc}
 \hline\hline
      $(p_1, p_2)$ & $\Cor_{\bSigma}$ & JACA & ssJACA & SLDA  sep &SLDA joint&Sparse CCA&Sparse sCCA & CVR \\
  \hline
(100,100) & $(\bW_1,\bTheta_1)$ & 0.839 (0.021) & {\bf 0.910} (0.015) & 0.823 (0.027) & 0.835 (0.010) & 0.752 (0.013) & 0.740 (0.007) & 0.756 (0.028) \\ 
   & $(\bW_2,\bTheta_2)$ & 0.907 (0.012) & {\bf 0.911} (0.011) & 0.889 (0.018) & 0.825 (0.008) & 0.841 (0.000) & 0.825 (0.000) & 0.704 (0.022) \\ 
  (100,500) & $(\bW_1,\bTheta_1)$ & 0.842 (0.016) & {\bf 0.906} (0.014) & 0.824 (0.032) & 0.833 (0.011) & 0.755 (0.015) & 0.742 (0.008) & 0.751 (0.028) \\ 
   & $(\bW_2,\bTheta_2)$ & {\bf 0.893} (0.013) & 0.876 (0.011) & 0.882 (0.018) & 0.816 (0.008) & 0.844 (0.000) & 0.734 (0.000) & 0.666 (0.026) \\ 
  (500,500) & $(\bW_1,\bTheta_1)$ & 0.839 (0.022) & {\bf 0.900} (0.019) & 0.839 (0.026) & 0.830 (0.014) & 0.758 (0.016) & 0.711 (0.003) & 0.745 (0.032) \\ 
   & $(\bW_2,\bTheta_2)$ & {\bf 0.897} (0.013) & 0.886 (0.011) & 0.883 (0.008) & 0.817 (0.008) & 0.836 (0.000) & 0.738 (0.000) & 0.674 (0.019) \\  
   \hline\hline
 \end{tabular}}
 \end{table}
 \begin{table}[!t]
 \center
  \caption{\label{tab:simu2_corcase2}Comparison of estimation correlation of Case $2$ over 100 replications when $D=2$, $K=2$. Standard errors are given in the brackets and the highest values are highlighted in bold.}
\resizebox{\textwidth}{!}{
 \begin{tabular}{ccccccccc}
  \hline\hline
      $(p_1, p_2)$ & $\Cor_{\bSigma}$ & JACA & ssJACA & SLDA  sep &SLDA joint&Sparse CCA&Sparse sCCA & CVR \\
  \hline
(100,100) & $(\bW_1,\bTheta_1)$ & 0.840 (0.021) & {\bf 0.912} (0.015) & 0.825 (0.031) & 0.836 (0.010) & 0.687 (0.017) & 0.744 (0.008) & 0.726 (0.028) \\ 
   & $(\bW_2,\bTheta_2)$ & 0.907 (0.011) & {\bf 0.912} (0.010) & 0.883 (0.019) & 0.816 (0.008) & 0.755 (0.009) & 0.823 (0.000) & 0.697 (0.021) \\ 
  (100,500) & $(\bW_1,\bTheta_1)$ & 0.844 (0.015) & {\bf 0.908} (0.014) & 0.825 (0.030) & 0.834 (0.010) & 0.704 (0.017) & 0.745 (0.009) & 0.718 (0.027) \\ 
   & $(\bW_2,\bTheta_2)$ & {\bf 0.895} (0.013) & 0.877 (0.010) & 0.883 (0.017) & 0.818 (0.008) & 0.780 (0.010) & 0.732 (0.000) & 0.640 (0.024) \\ 
  (500,500) & $(\bW_1,\bTheta_1)$ & 0.838 (0.022) & {\bf 0.900} (0.020) & 0.840 (0.025) & 0.831 (0.014) & 0.592 (0.018) & 0.711 (0.003) & 0.718 (0.028) \\ 
   & $(\bW_2,\bTheta_2)$ & {\bf 0.898} (0.013) & 0.886 (0.011) & 0.884 (0.009) & 0.817 (0.008) & 0.657 (0.033) & 0.738 (0.000) & 0.637 (0.018) \\ 
    \hline\hline
  \end{tabular}}
 \end{table}
\begin{table}[!t]
\center
 \caption{\label{tab:simu2_corcase3}Comparison of estimation correlation of Case $3$ over 100 replications when $D=2$, $K=2$. Standard errors are given in the brackets and the highest values are highlighted in bold.}
\resizebox{\textwidth}{!}{
 \begin{tabular}{llccccccc}
  \hline\hline
      $(p_1, p_2)$ & $\Cor_{\bSigma}$ & JACA & ssJACA & SLDA  sep &SLDA joint&Sparse CCA&Sparse sCCA & CVR \\
  \hline
(100,100) & $(\bW_1,\bTheta_1)$ & 0.843 (0.021) & {\bf 0.914} (0.017) & 0.830 (0.028) & 0.837 (0.009) & 0.098 (0.001) & 0.682 (0.012) & 0.727 (0.010) \\ 
   & $(\bW_2,\bTheta_2)$ & 0.904 (0.010) & {\bf 0.908} (0.011) & 0.888 (0.019) & 0.805 (0.008) & 0.040 (0.016) & 0.720 (0.000) & 0.714 (0.020) \\ 
  (100,500) & $(\bW_1,\bTheta_1)$ & 0.844 (0.016) & {\bf 0.909} (0.014) & 0.822 (0.030) & 0.833 (0.006) & 0.117 (0.001) & 0.728 (0.013) & 0.731 (0.013) \\ 
   & $(\bW_2,\bTheta_2)$ & {\bf 0.896} (0.013) & 0.876 (0.011) & 0.884 (0.020) & 0.820 (0.007) & 0.043 (0.015) & 0.700 (0.000) & 0.625 (0.027) \\ 
  (500,500) & $(\bW_1,\bTheta_1)$ & 0.839 (0.022) & {\bf 0.900} (0.020) & 0.842 (0.025) & 0.831 (0.015) & 0.033 (0.000) & 0.694 (0.003) & 0.714 (0.019) \\ 
   & $(\bW_2,\bTheta_2)$ & {\bf 0.898} (0.013) & 0.887 (0.012) & 0.885 (0.008) & 0.817 (0.008) & 0.033 (0.008) & 0.716 (0.000) & 0.636 (0.019) \\ 
   \hline\hline
 \end{tabular}}
 \end{table}
   \begin{figure}[!t]
  \centering
  \includegraphics[scale=.35]{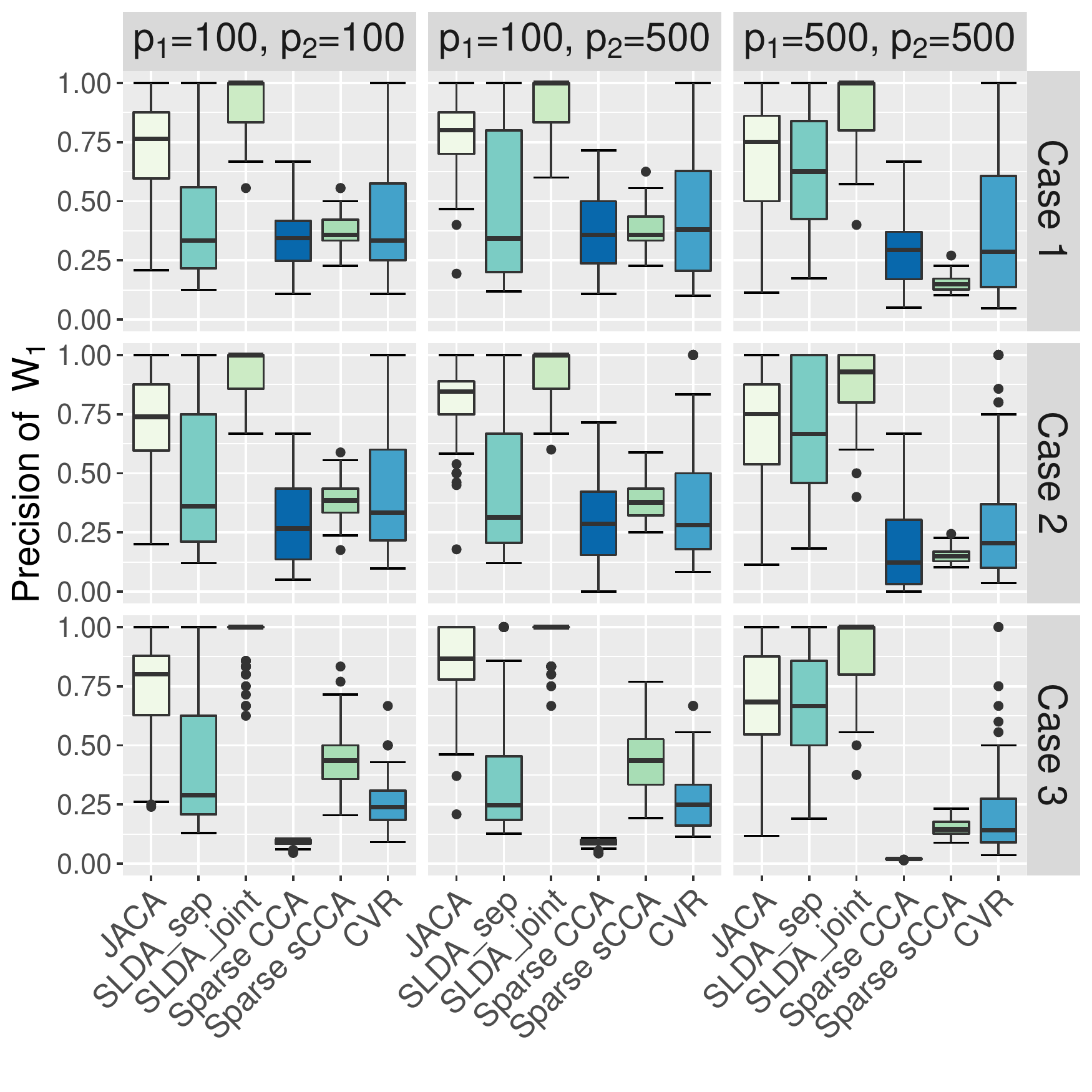} 
  \includegraphics[scale=.35]{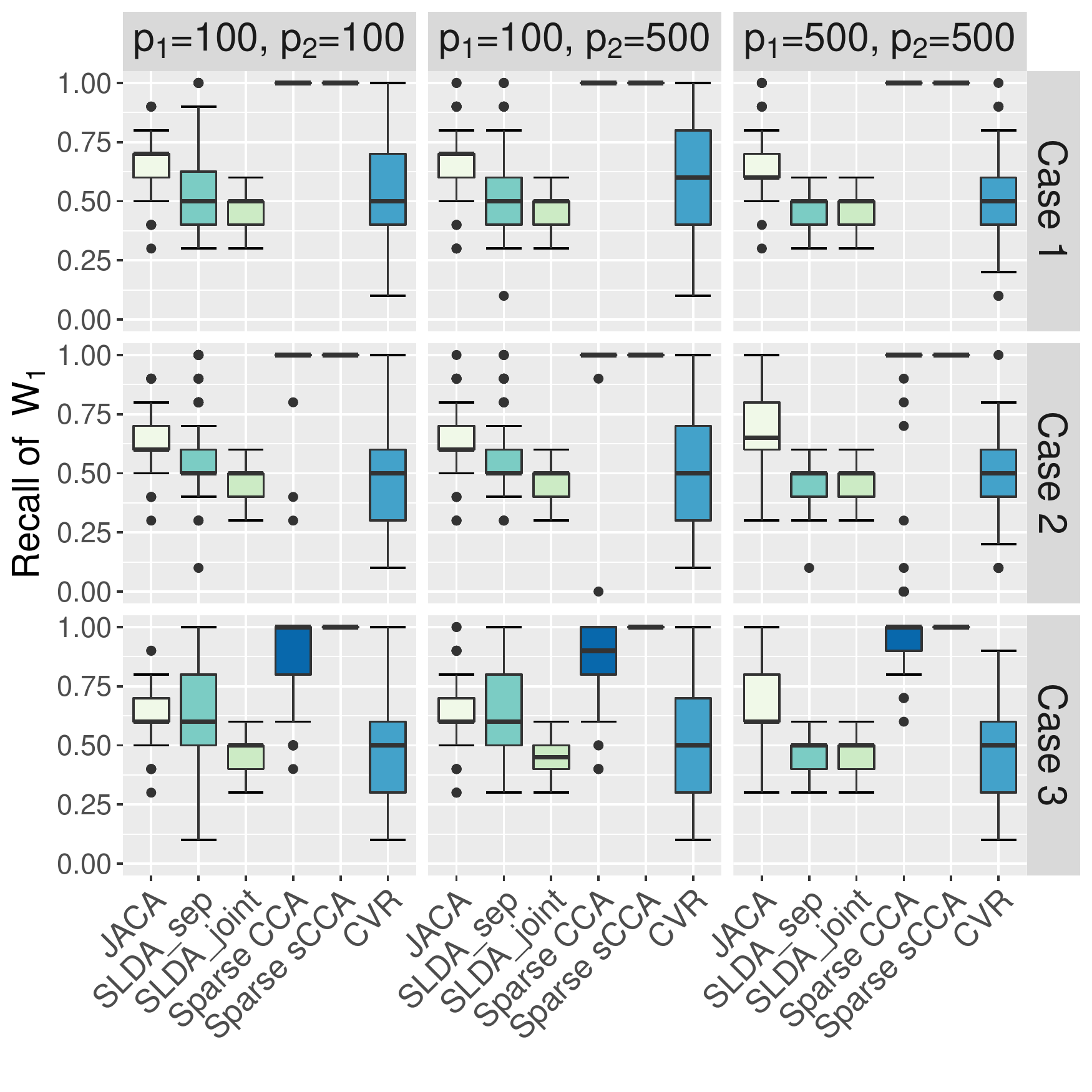}
  \includegraphics[scale=.35]{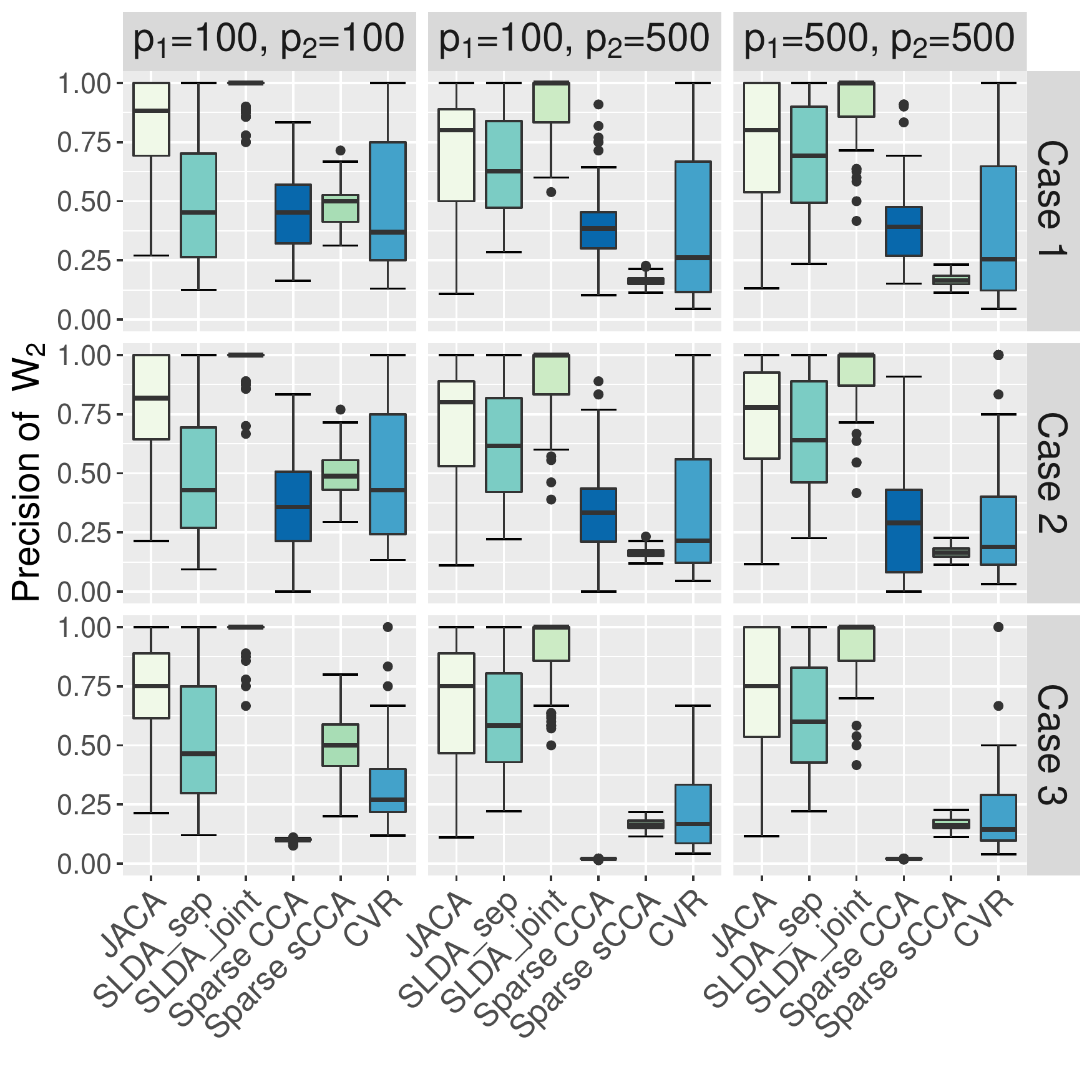} 
  \includegraphics[scale=.35]{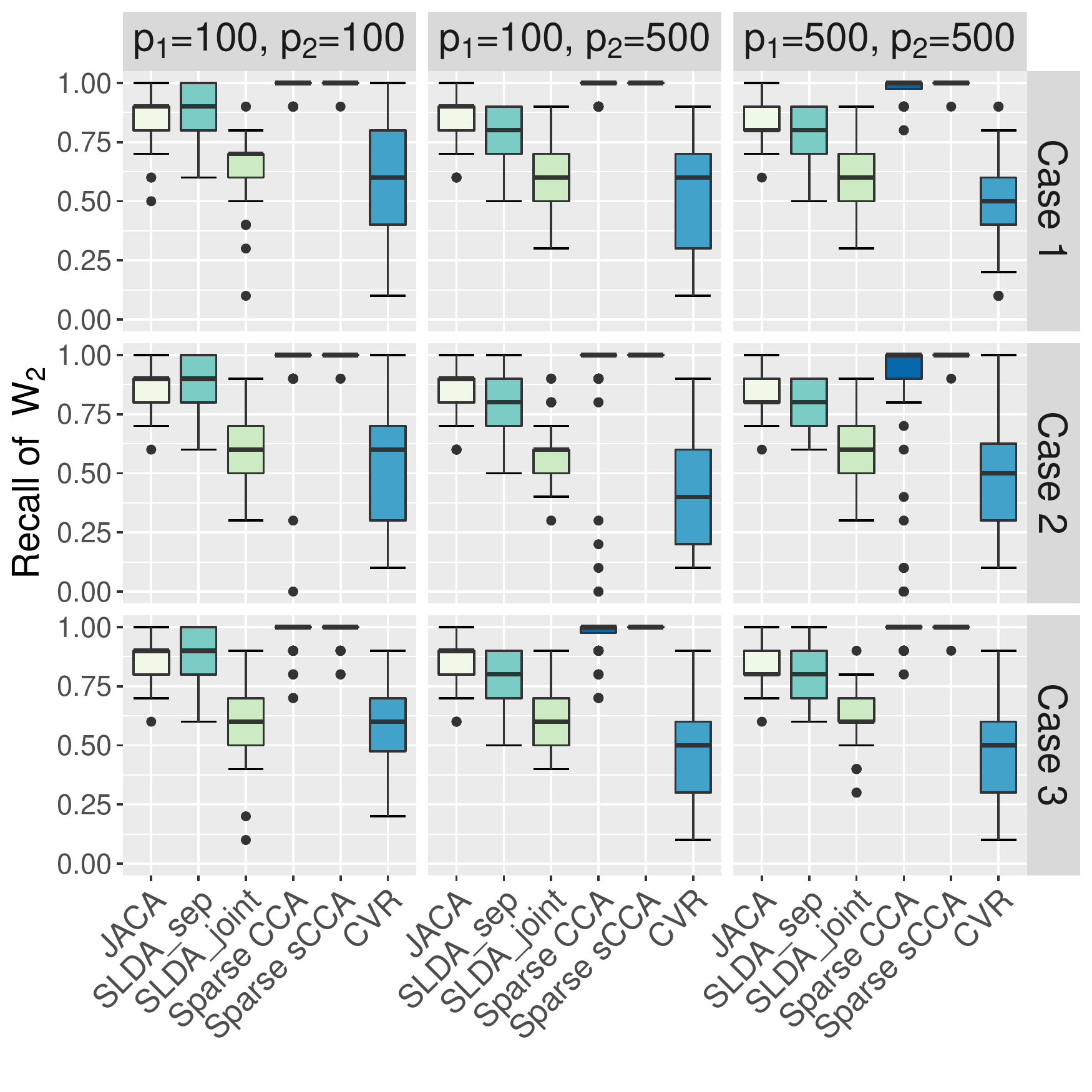} 
  \caption{Precision and Recall over 100 replications when $D=2$ and $K=2$.}
  \label{fig:simu_prerec}
  \end{figure}

\newpage
  \subsection{Multiple datasets, multiple groups}\label{sec:multiple}

We set $n_1=240$, $K=3$, and generate $n_1$ independent $y\in\{1,2,3\}$ with $\pi_1=0.4$, $\pi_2=\pi_3=0.3$. We also generate $n_1    $ tuples $(\bx_1, \bx_2, \bx_3)\in \R^{p_1}\times \R^{p_2}\times \R^{p_3}$ with $p_1 = p_2 = p_3\in \{100, 500\}$ following Section~\ref{s:datageneration}. Next, we generate $n_2=100$ samples and set   corresponding  class  information  as  missing. We train ssJACA on all $n_1+n_2=340$ samples and train other methods on $n_1=240$ complete samples.
We set $\widetilde \bSigma_{1}=(0.8^{|i-j|})_{ij}$, $\widetilde \bSigma_{2}=(0.5^{|i-j|})_{ij}$ and $\widetilde \bSigma_{3} = \bI$. 
We let canonical correlations due to class membership be $\rho_1 = \rho_2 = 0.8$, and consider the following cases for other shared factors:
\begin{description}
\item[Case 1:] $q=0$, no shared factors except class $y$;
\item[Case 2:] $q = 3$ with $\rho_3 = \rho_4 = \rho_5 = 0.6$;
\item[Case 3:] $q = 3$ with $\rho_3=0.9$, $\rho_4 = 0.9$, $\rho_5 = 0.5$. 
\end{description}
Similar to Section~\ref{sec:2groups}, the  misclassification rates are evaluated on $10,000$ independently generated test samples.  We do not consider Sparse CCA methods because they are not directly applicable to the case of more than two views and more than two classes.  While the issue of more than two views can be addressed by Multi~CCA generalization \citep{Witten:2009wa}, both Sparse~CCA and Multi~CCA find $K-1$ pairs of canonical vectors sequentially. As a result, one also needs to tune sparsity parameters sequentially leading to computationally expensive procedure with different sparsity patterns across canonical vector pairs. We also do not consider CVR as it is only implemented for binary classification problem. 

The results for JACA, ssJACA and SLDA methods are reported in Tables~\ref{tab:simu_multicorcase}--\ref{tab:simu_multi_corr_case} and Figure~\ref{fig:multi_prerec}.   ssJACA performs the best in terms of misclassification rates and estimation consistency in most scenarios, and always performs the best in terms of sum correlation. It also achieves the best trade off between precision and recall.  When predicted based on $\bX_1$ alone, JACA and ssJACA have similar performance with SLDA\_sep, but SLDA\_sep's performance decreases significantly as $p$ increases. On the other hand, SLDA\_joint performs poorly in most cases.

   \begin{table}[!t]
   \center
    \caption{ \label{tab:simu_multicorcase}Comparison of misclassification  rates  over 100 replication when $D=3$, $K=3$.
 Standard errors are given in the brackets and the lowest values are highlighted in bold.}
\resizebox{\textwidth}{!}{ \begin{tabular}{llcccccccc}
 \hline\hline
  & &\multicolumn{4}{c}{$p_1=p_2=p_3=100$}&\multicolumn{4}{c}{$p_1=p_2=p_3=500$}\\
\cmidrule(lr){3-6} \cmidrule(lr){7-10}
      &Error rate ($\%$) & JACA & ssJACA & SLDA sep &SLDA joint& JACA & ssJACA & SLDA sep &SLDA joint \\
  \hline
Case 1 & $(\bX_1)$ & 2.632 (0.051) & {\bf 2.268} (0.019) & 2.511 (0.056) & 7.182 (0.155) & 4.555 (0.076) & {\bf 2.583} (0.022) & 5.398 (0.105) & 7.014 (0.150) \\ 
   & $(\bX_2)$ & 2.112 (0.017) & {\bf 1.610} (0.014) & 2.350 (0.046) & 4.746 (0.241) & 1.988 (0.013) & {\bf 1.917} (0.020) & 2.343 (0.062) & 4.328 (0.160) \\ 
   & $(\bX_3)$ & 1.750 (0.016) & {\bf 1.556} (0.013) & 1.802 (0.035) & 22.344 (0.957) & {\bf 1.450} (0.016) & 1.577 (0.014) & 1.581 (0.041) & 22.041 (1.044) \\ 
   & $(\bX_1,\bX_2, \bX_3)$ & 0.010 (0.001) & {\bf 0.005} (0.001) & 0.015 (0.001) & 0.389 (0.034) & 0.025 (0.001) & {\bf 0.011} (0.001) & 0.051 (0.003) & 0.430 (0.036) \\ 
  Case 2 & $(\bX_1)$ & 2.545 (0.049) & {\bf 2.279} (0.019) & 2.370 (0.040) & 7.389 (0.166) & 4.524 (0.077) & {\bf 2.580} (0.023) & 5.361 (0.100) & 7.245 (0.188) \\ 
   & $(\bX_2)$ & 2.127 (0.017) & {\bf 1.615} (0.014) & 2.363 (0.043) & 4.596 (0.151) & 1.994 (0.013) & {\bf 1.940} (0.020) & 2.225 (0.042) & 4.441 (0.165) \\ 
   & $(\bX_3)$ & 1.770 (0.016) & {\bf 1.555} (0.012) & 1.790 (0.032) & 22.771 (0.935) & {\bf 1.458} (0.017) & 1.581 (0.013) & 1.550 (0.037) & 22.573 (0.992) \\ 
   & $(\bX_1,\bX_2, \bX_3)$ & 0.009 (0.001) & {\bf 0.005} (0.001) & 0.015 (0.001) & 0.363 (0.027) & 0.025 (0.001) & {\bf 0.012} (0.001) & 0.048 (0.003) & 0.449 (0.036) \\ 
  Case 3 & $(\bX_1)$ & 2.384 (0.039) & {\bf 2.266} (0.019) & 2.289 (0.039) & 7.355 (0.140) & 4.391 (0.080) & {\bf 2.550} (0.021) & 5.351 (0.105) & 7.259 (0.182) \\ 
   & $(\bX_2)$ & 2.139 (0.018) & {\bf 1.633} (0.015) & 2.393 (0.049) & 4.534 (0.147) & 2.006 (0.015) & {\bf 1.925} (0.019) & 2.337 (0.059) & 4.536 (0.190) \\ 
   & $(\bX_3)$ & 1.818 (0.017) & {\bf 1.573} (0.012) & 1.809 (0.033) & 23.773 (0.980) & {\bf 1.485} (0.018) & 1.577 (0.013) & 1.589 (0.036) & 22.821 (1.004) \\ 
   & $(\bX_1,\bX_2, \bX_3)$ & 0.010 (0.001) & {\bf 0.005} (0.001) & 0.016 (0.001) & 0.364 (0.029) & 0.025 (0.001) & {\bf 0.012} (0.001) & 0.052 (0.004) & 0.472 (0.038) \\ 
  \hline\hline
 \end{tabular}}
 \end{table}

\begin{table}[!t]
\center
   \caption{ \label{tab:simu_multisumcor}Comparison of sum correlation  over 100 replication when $D=3$, $K=3$.
 Standard errors are given in the brackets and the highest values are highlighted in bold.}
\resizebox{\textwidth}{!}{ \begin{tabular}{ccccccccc}
 \hline\hline
   &\multicolumn{4}{c}{$p_1=p_2=p_3=100$}&\multicolumn{4}{c}{$p_1=p_2=p_3=500$}\\
\cmidrule(lr){2-5} \cmidrule(lr){6-9}
     & JACA & ssJACA & SLDA sep &SLDA joint& JACA &ssJACA & SLDA sep &SLDA joint \\
  \hline
Case 1 & 2.321 (0.001) & {\bf 2.344} (0.001) & 2.309 (0.004) & 1.196 (0.021) & 2.282 (0.001) & {\bf 2.336} (0.001) & 2.186 (0.010) & 1.213 (0.023) \\ 
  Case 2 & 2.322 (0.001) & {\bf 2.344} (0.001) & 2.314 (0.002) & 1.190 (0.020) & 2.282 (0.001) & {\bf 2.336} (0.001) & 2.186 (0.010) & 1.213 (0.023) \\ 
  Case 3 & 2.326 (0.001) & {\bf 2.346} (0.001) & 2.316 (0.003) & 1.196 (0.020) & 2.284 (0.002) & {\bf 2.337} (0.001) & 2.183 (0.011) & 1.205 (0.022) \\ 
\hline
\end{tabular}}
\end{table}

\begin{table}[!t]
 \caption{\label{tab:simu_multi_corr_case}Comparison of estimation correlation  over 100 replication when $D=3$, $K=3$. 
 Standard errors are given in the brackets and the highest values are highlighted in bold.}
\resizebox{\textwidth}{!}{ \begin{tabular}{cccccccccc}
 \hline\hline
  & &\multicolumn{4}{c}{$p_1=p_2=p_3=100$}&\multicolumn{4}{c}{$p_1=p_2=p_3=500$}\\
\cmidrule(lr){3-6} \cmidrule(lr){7-10}
      &$\Cor_{\bSigma}$ & JACA & ssJACA & SLDA sep &SLDA joint& JACA & ssJACA & SLDA sep &SLDA joint \\
  \hline
Case 1 & $(\bW_1,\bTheta_1)$ & 0.903 (0.002) & {\bf 0.913} (0.001) & 0.906 (0.002) & 0.795 (0.002) & 0.848 (0.002) & {\bf 0.901} (0.001) & 0.825 (0.003) & 0.800 (0.002) \\ 
   & $(\bW_2,\bTheta_2)$ & 0.945 (0.001) & {\bf 0.957} (0.001) & 0.929 (0.003) & 0.794 (0.008) & 0.937 (0.001) & {\bf 0.948} (0.001) & 0.913 (0.004) & 0.801 (0.008) \\ 
   & $(\bW_3,\bTheta_3)$ & 0.959 (0.001) & {\bf 0.975} (0.001) & 0.960 (0.002) & 0.710 (0.010) & 0.969 (0.001) & {\bf 0.977} (0.001) & 0.961 (0.003) & 0.726 (0.011) \\ 
  Case 2 & $(\bW_1,\bTheta_1)$ & 0.908 (0.002) & {\bf 0.914} (0.001) & 0.914 (0.002) & 0.795 (0.002) & 0.850 (0.002) & {\bf 0.902} (0.001) & 0.827 (0.003) & 0.798 (0.002) \\ 
   & $(\bW_2,\bTheta_2)$ & 0.946 (0.001) & {\bf 0.958} (0.001) & 0.931 (0.003) & 0.799 (0.007) & 0.937 (0.001) & {\bf 0.948} (0.001) & 0.921 (0.003) & 0.797 (0.007) \\ 
   & $(\bW_3,\bTheta_3)$ & 0.959 (0.001) & {\bf 0.976} (0.001) & 0.962 (0.002) & 0.709 (0.010) & 0.969 (0.001) & {\bf 0.976} (0.001) & 0.963 (0.002) & 0.726 (0.010) \\ 
  Case 2 & $(\bW_1,\bTheta_1)$ & 0.917 (0.002) & 0.917 (0.001) & {\bf 0.921} (0.002) & 0.802 (0.002) & 0.855 (0.002) & {\bf 0.903} (0.001) & 0.828 (0.003) & 0.799 (0.002) \\ 
   & $(\bW_2,\bTheta_2)$ & 0.948 (0.001) & {\bf 0.961} (0.001) & 0.930 (0.003) & 0.805 (0.007) & 0.937 (0.001) & {\bf 0.949} (0.001) & 0.914 (0.004) & 0.794 (0.008) \\ 
   & $(\bW_3,\bTheta_3)$ & 0.957 (0.001) & {\bf 0.976} (0.001) & 0.963 (0.002) & 0.702 (0.010) & 0.967 (0.001) & {\bf 0.977} (0.001) & 0.960 (0.003) & 0.719 (0.010) \\ 
   \hline\hline
 \end{tabular}}
 \end{table}
 
   \begin{figure}[!t]
  \centering
  \includegraphics[scale=.25]{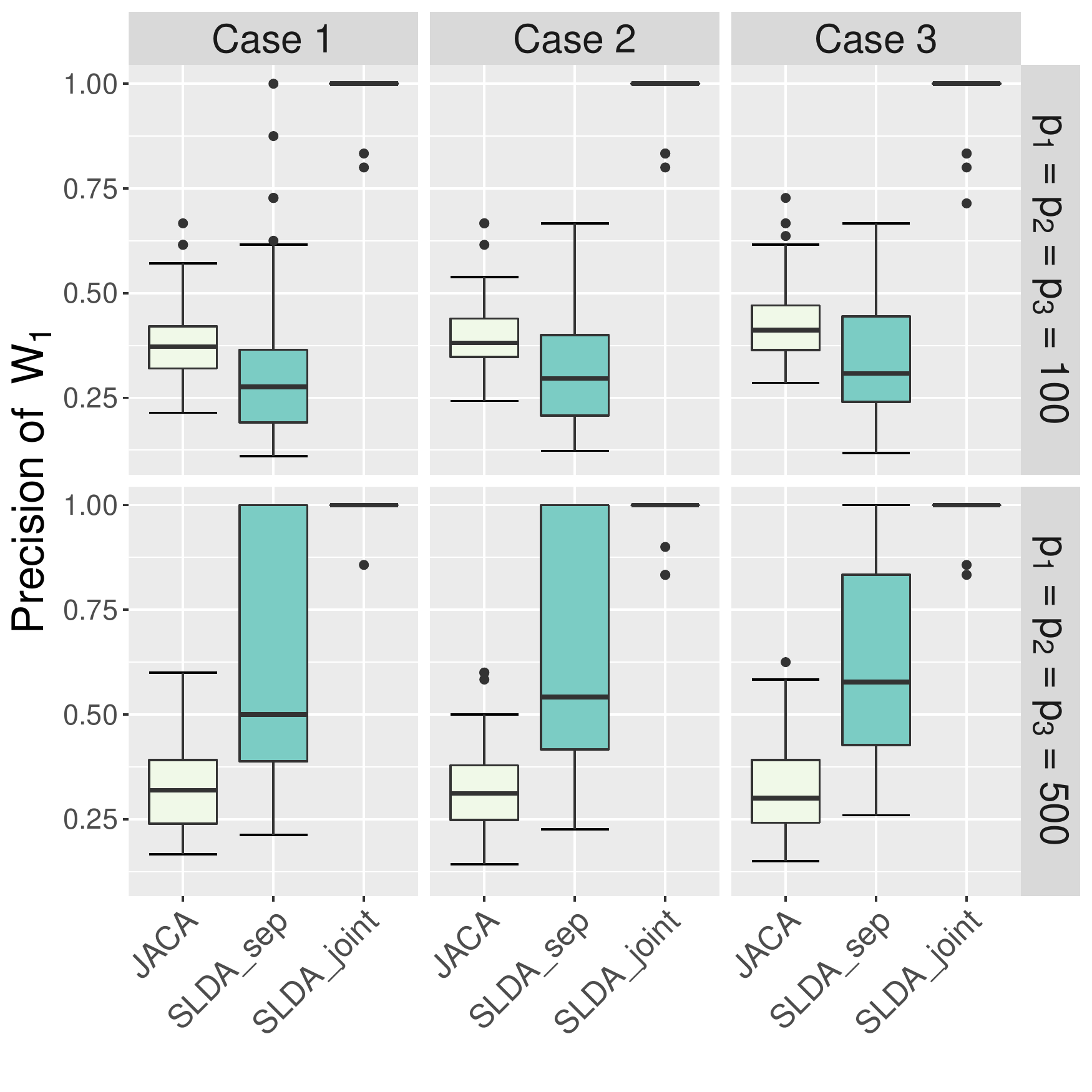} 
  \includegraphics[scale=.25]{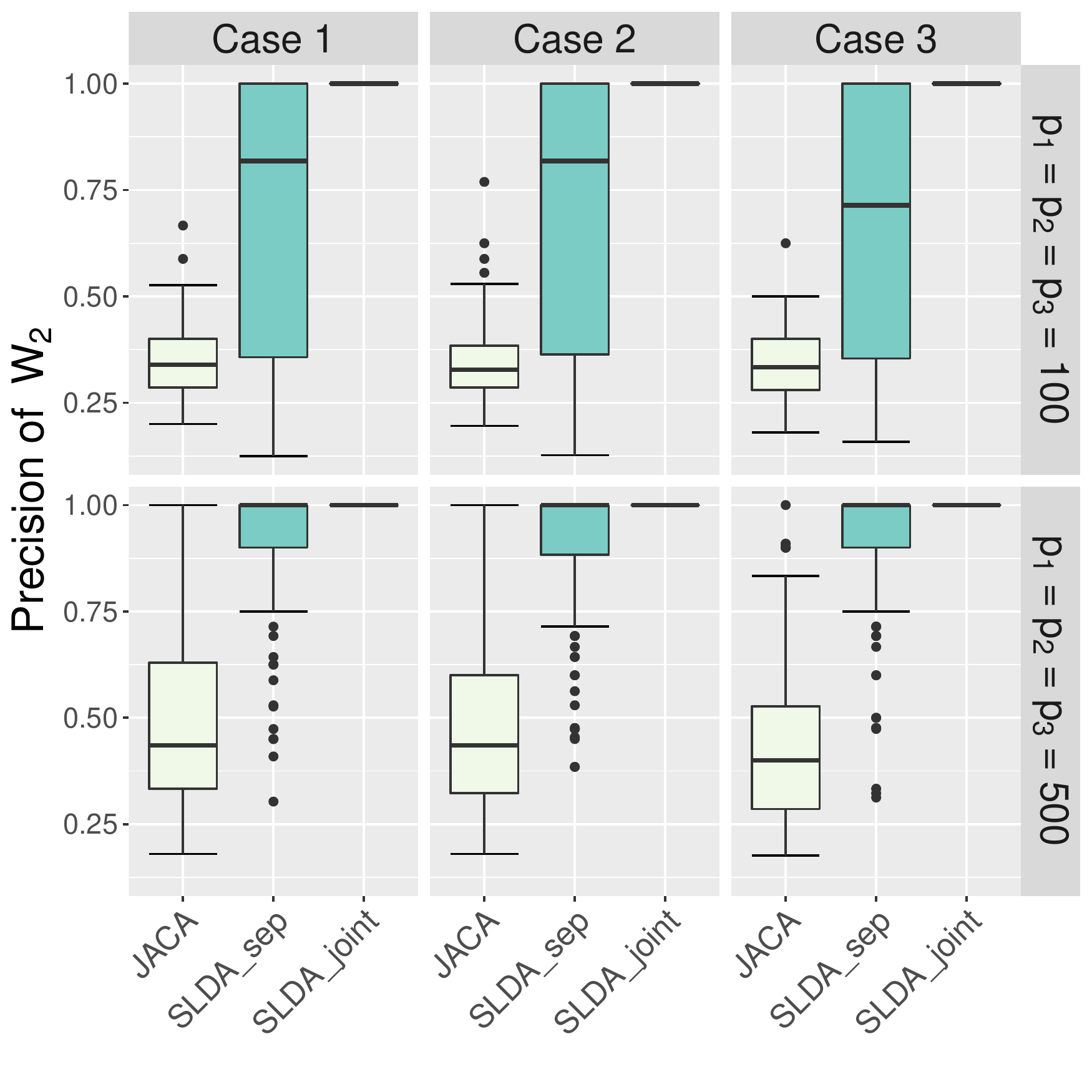} 
  \includegraphics[scale=.25]{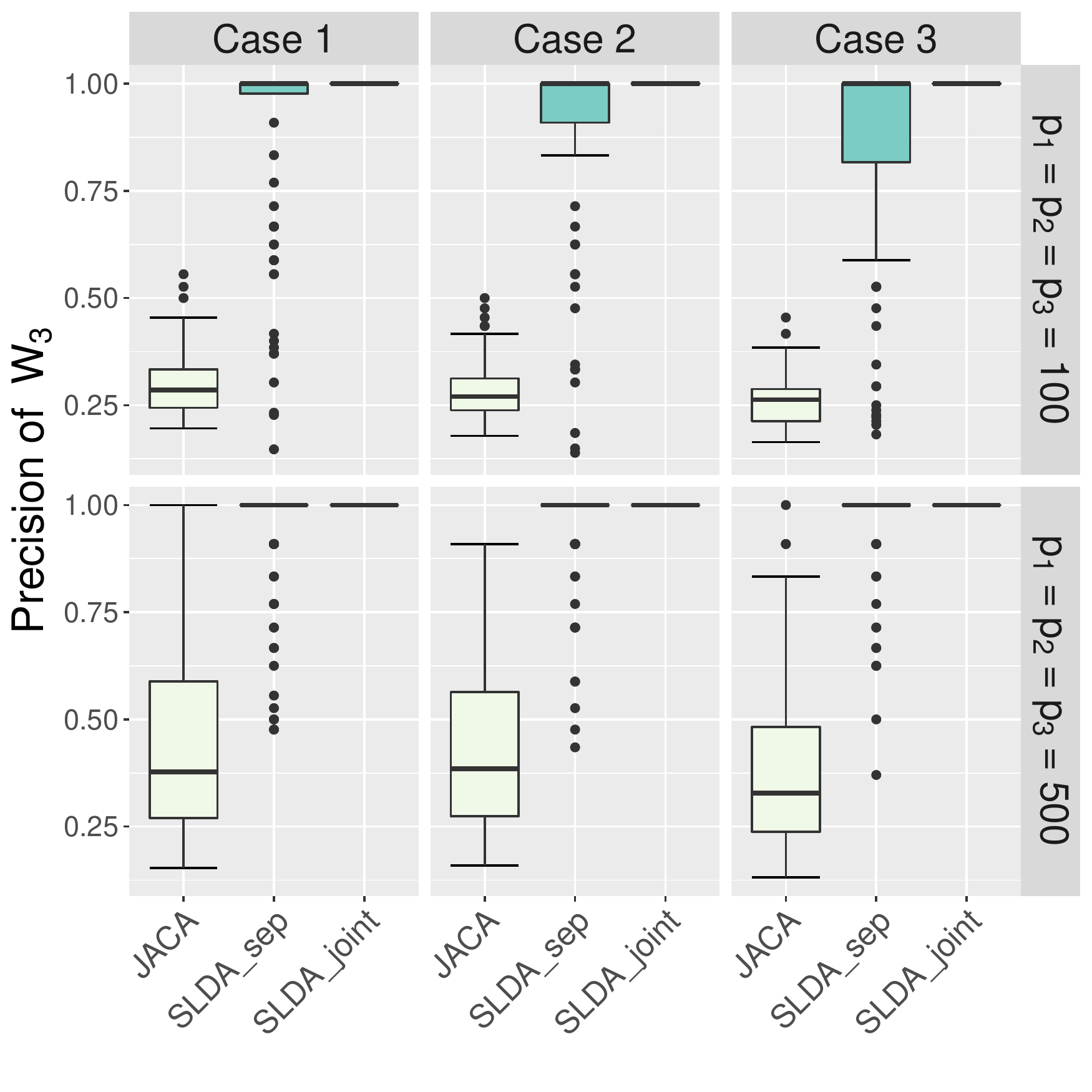} 
  \includegraphics[scale=.25]{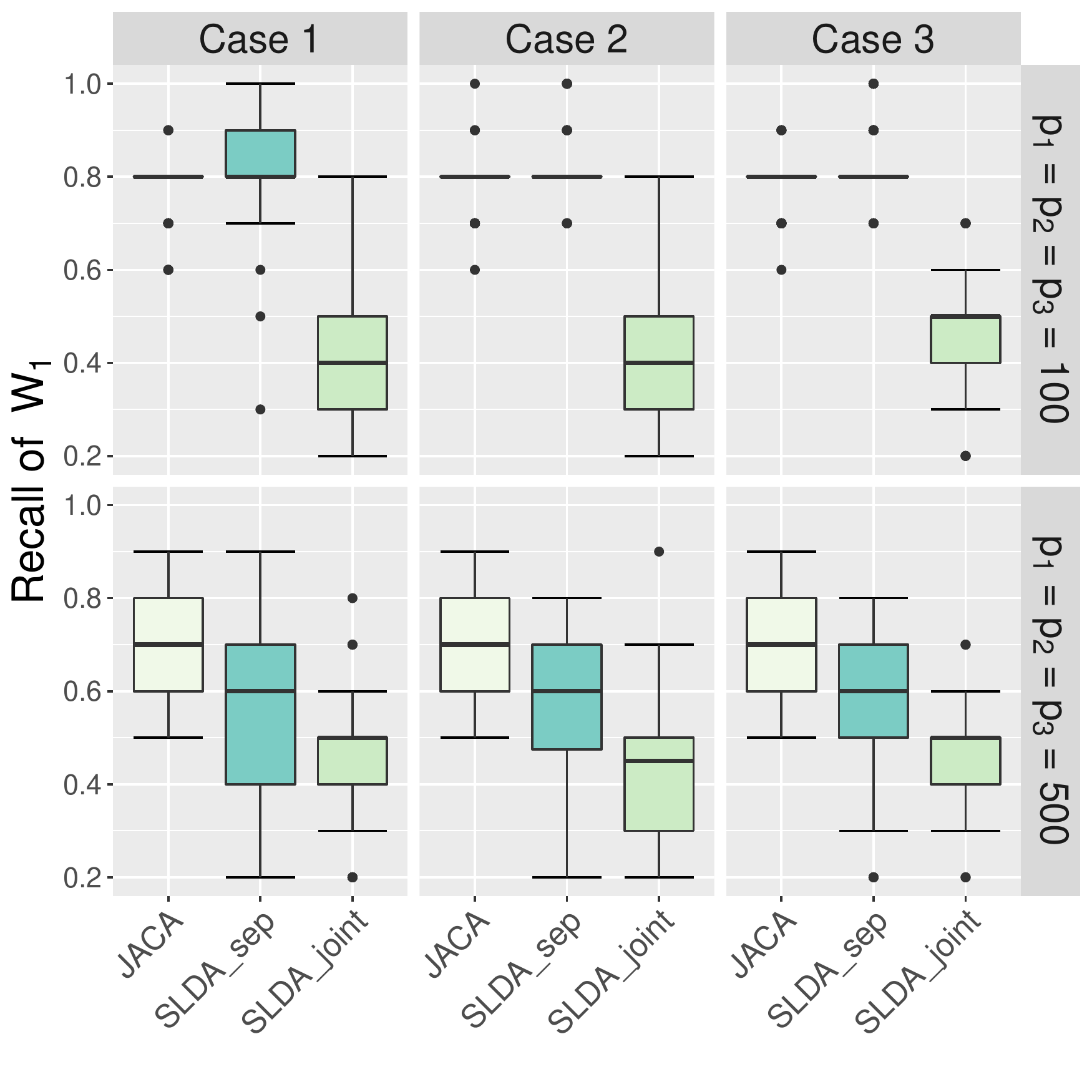}
  \includegraphics[scale=.25]{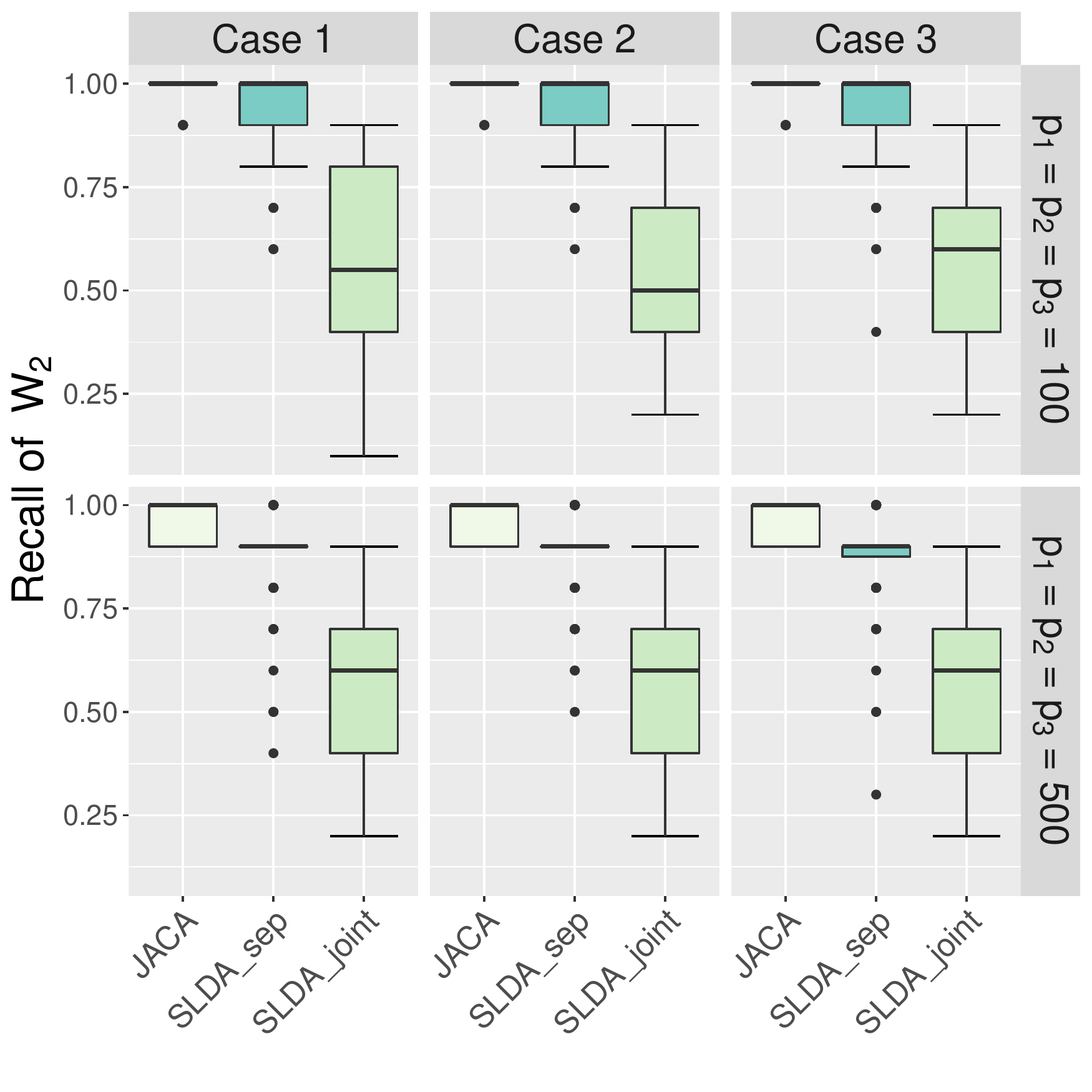} 
  \includegraphics[scale=.25]{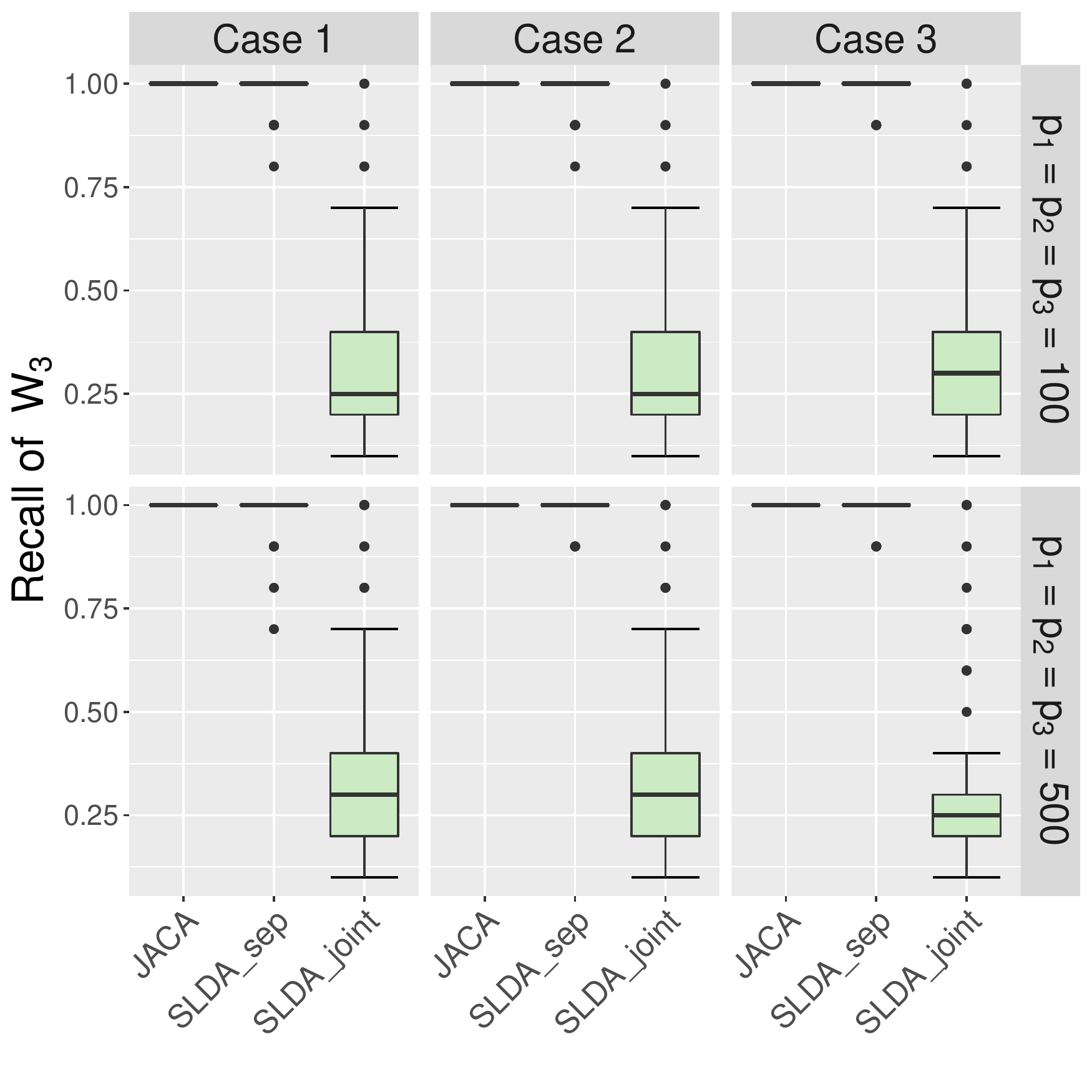} 
  \caption{Precision and Recall over 100 replications when $D = 3$ and $K=3$.}
  \label{fig:multi_prerec}
  \end{figure}



\section{Additional data analysis}\label{sec:addiData}

\subsection{TCGA-COAD data}\label{sec:addiCOAD}
In this section we present additional results from the analysis of COAD data from Section~\ref{sec:COAD}. The enlarged heatmaps from Figure~\ref{fig:heatmap} are displayed in Figure~\ref{fig:ssJACAheatmap}.


 \begin{figure}[!t]
  \centering
  \includegraphics[scale=.45]{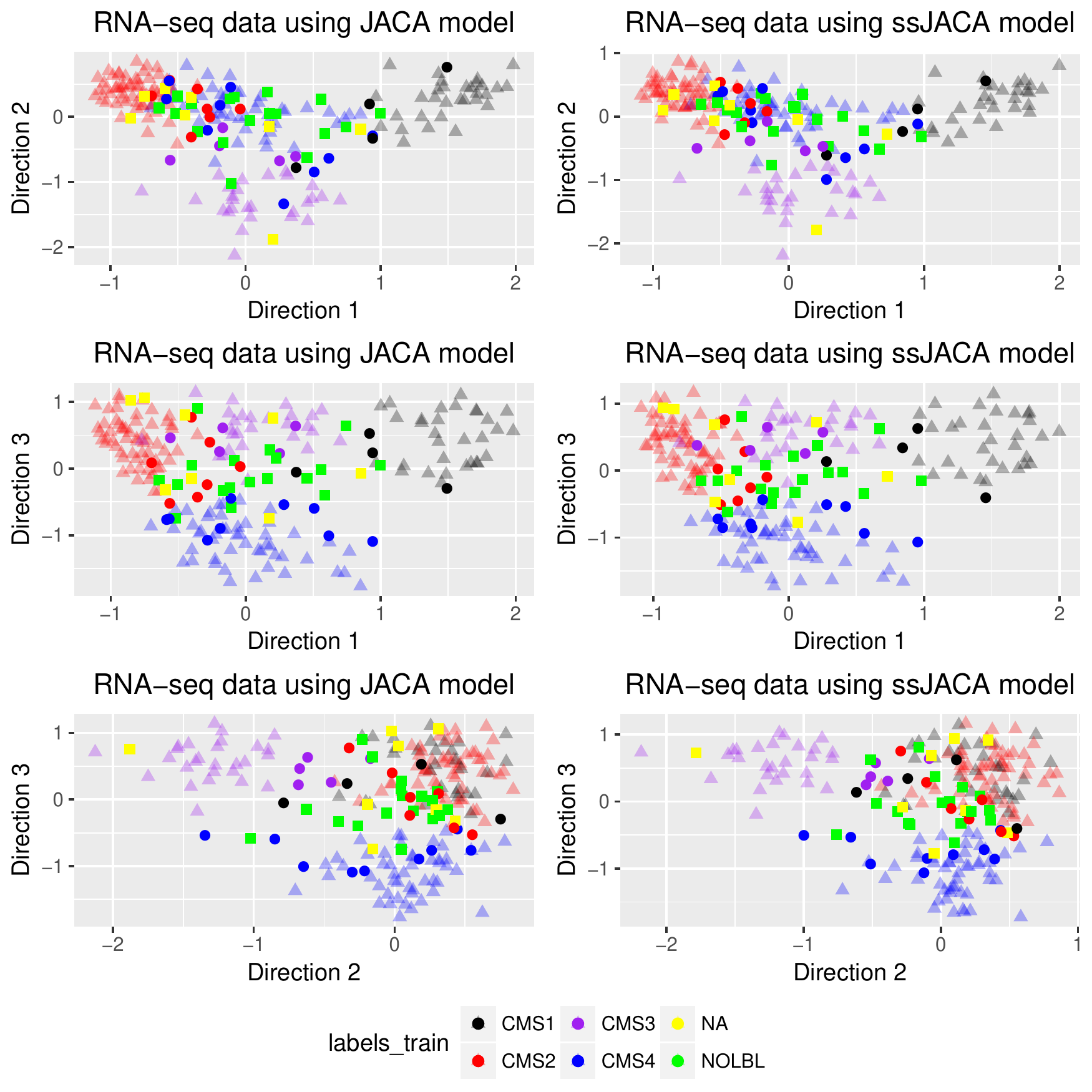} 
  \includegraphics[scale=.45]{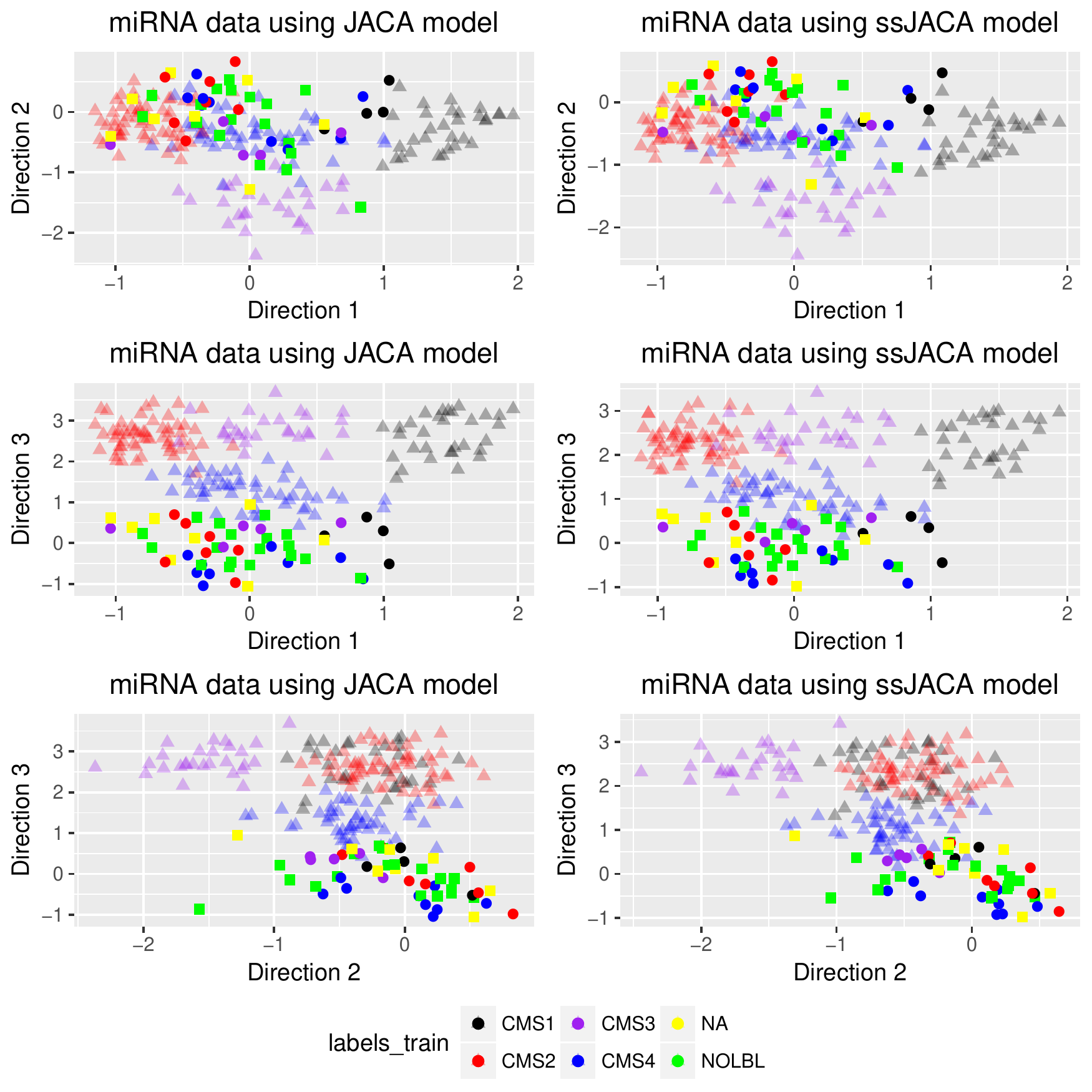} 

  \caption{Projection of RNAseq and miRNA views from COAD data onto discriminant directions found by JACA and ssJACA.}
  \label{fig:projection}
  \end{figure}  

     \begin{figure}[!t]
  \centering
  \includegraphics[scale=.56]{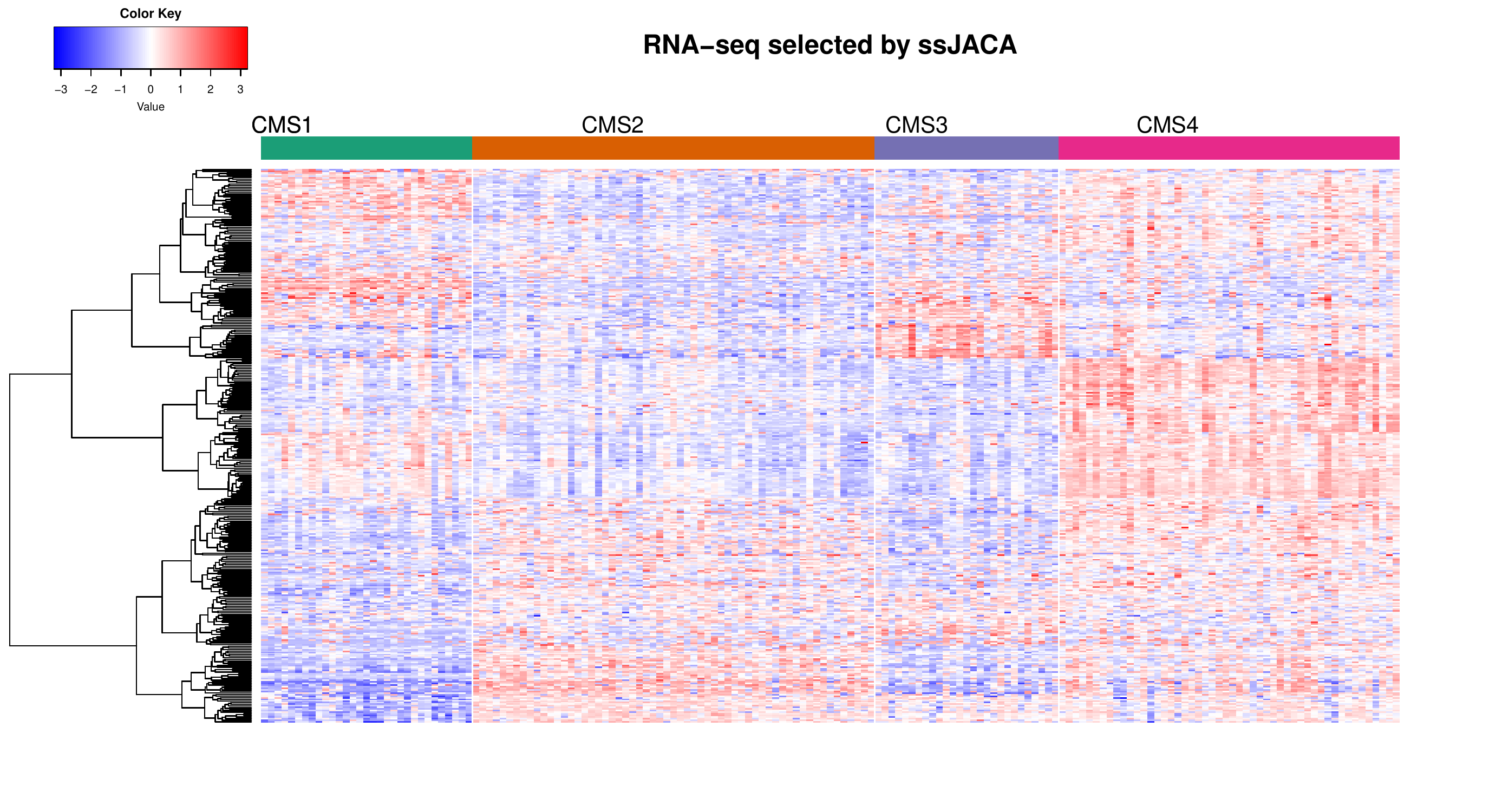} 
  \includegraphics[scale=.56]{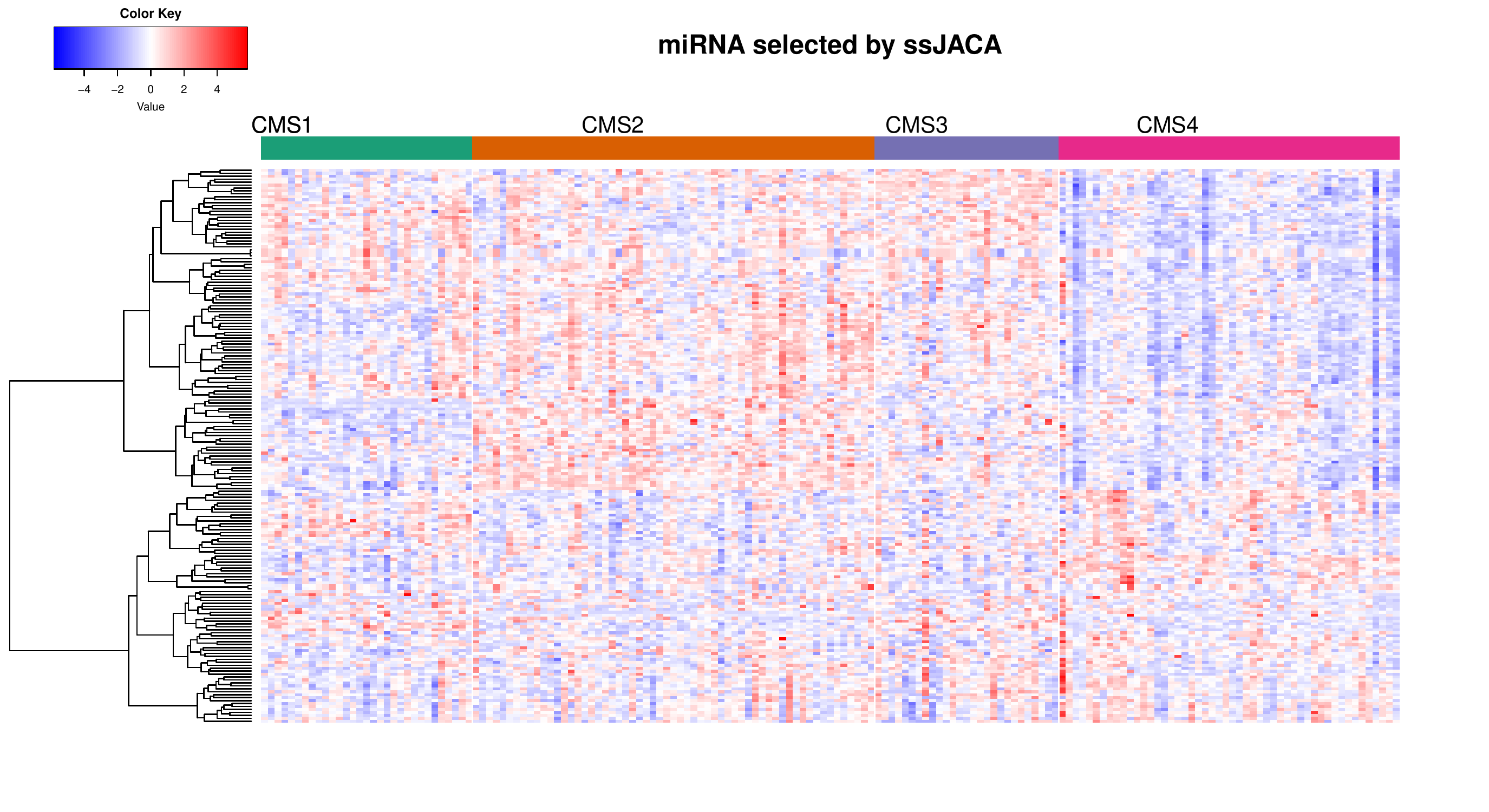} 
  \caption{Heatmaps of RNAseq and miRNA views from COAD data based on features selected by ssJACA. We use Ward's linkage with euclidean distances for feature ordering.}
  \label{fig:ssJACAheatmap}
  \end{figure} 
  
  We also consider the visual separation of subtypes based on the projection of RNAseq and miRNA data using discriminant directions found by JACA and ssJACA (Figure \ref{fig:projection}).  The triangular points in transparent colors indicate $167$ subjects with complete view and subtype information. The round points in solid colors are subjects who have missing subtypes, but for whom the subtypes have been previously predicted using random forest classifier \citep{Guinney:2015dm}. We treat these predictions as the gold standard. The square points in solid colors are subjects with no assigned subtype, which are deemed to have mixed subtype membership \citep{Guinney:2015dm}. The subtype separation is clear based on the projected values, with square points being often in the middle of other subtypes, thus confirming the possibility of mixed subtype membership for those subjects.
  

\subsection{TCGA-BRCA dataset}\label{sec:BRCA}

We consider breast cancer data from The Cancer Genome Atlas project with 4 views: gene expression (GE), DNA methylation (ME), miRNA expression (miRNA), and reverse phase protein array (RPPA). The samples are separated into 4 breast cancer subtypes: Basal, LumA, LumB and Her2 \citep{Network:2012fy}. Five samples are labelled as Normal-like, and we exclude them from the analyses. \citet{Li:2016cf} incorporate subtypes into supervised singular value decomposition, however only GE view is considered. \citet{Lock:2013vy} and \citet{Gaynanova:2017te} jointly analyze all views, however do not take advantage of the subtypes. In this section, we apply JACA to understand the subtype-driven relationships between the views. We use data from \url{https://gdc.cancer.gov/about-data/publications/brca_2012} and the same data-processing as in \citet{Lock:2013vy}. While the combined number of subjects is 792, only 377 have complete view/subtype information (see Table~\ref{tab:BRCAdata}). 

\begin{table}[!t]
\center
\caption{\label{tab:BRCAdata}Number of samples in BRCA data with different missing patterns of views and cancer subtype. There are only 377 samples with complete information, whereas semi-supervised JACA approach allows to use 708 (all except the last row).}
\begin{tabular}{rrrrrr}
  \hline\hline
GE & ME & miRNA & RPPA & Cancer type & Count \\ 
  \hline  
yes & yes & yes & yes & yes & 377 \\ 
  yes & yes & yes & no & yes & 114 \\ 
  yes & yes & no & yes & yes & 19 \\ 
  yes & yes & no & no & yes & 3 \\ 
  yes & no & yes & yes & yes & 1 \\ 
  no & yes & yes & yes & no & 1 \\ 
  no & yes & yes & no & no & 193 \\ 
  no & yes & no & no & no & 84 \\ 
  \hline
  &  &  &  &  & Total = 792 \\ 
  \hline  \hline
\end{tabular}
\end{table}

We compare JACA and ssJACA with SLDA\_sep and SLDA\_joint on the 377 subjects with complete view/subtype information following the same strategy as in Section~\ref{sec:COAD}. We randomly select $299$ samples for training and the rest for testing. For ssJACA, we additionally add $331$ subjects (at least two views available) into the training set.  We set $\alpha$ in JACA and ssJACA to be $0.7$ as in Section~\ref{sec:COAD}. We do not consider CVR due to $K>2$ and $D>2$, and we do not consider Sparse CCA or Sparse sCCA 
due to their poor performance on COAD data. 

Tables~\ref{tab:error_BRCA} and \ref{tab:card_BRCA} display the mean misclassification error rates and the number of selected variables for each view, where the predictions are made either separately on each view, or jointly using all views. The results are similar to Section~\ref{sec:COAD}. The misclassification rates are higher when using ME, miRNA or RPPA compared to GE, which is not surprising since BRCA subtypes are originally determined based on gene expression. When predicting based on ME, JACA outperforms ssJACA in terms of misclassification rates, but it has a lower sum correlation in the meantime. Again, this is the result of the trade-off between the classification and association tasks. If the classification is the sole goal, then we recommend to accordingly modify the parameter selection criterion in Section~\ref{sec:tuning}. 
SLDA\_joint has the worst error rates, especially when using other views than GE. This is because SLDA\_joint selects very few variables from other views since the subtype-specific signal is the strongest in GE view. JACA and ssJACA have slightly better performance than SLDA\_sep using GE, and significantly better performance on other views, which suggests the advantage of taking into account the associations between the views. JACA and ssJACA also have higher cardinality, which is consistent with simulation results in Section~\ref{sec:simu}. Table~\ref{tab:correlation_BRCA} displays the sum correlation, with ssJACA performing the best compared to other methods.

\begin{table}[!t]
\center
\caption{\label{tab:error_BRCA}Mean misclassification error rates over 100 splits of 377 samples from BRCA data, standard errors are given in brackets and the lowest values are highlighted in bold.} 
\begin{tabular}{lccccc}
  \hline\hline
  & \multicolumn{5}{c}{Misclassification Rate  (\%)}\\
  \cmidrule(lr){2-6} 
 Method & GE & ME & miRNA & RPPA & All \\
 \hline
JACA & {\bf 4.41} (0.17) & {\bf10.5} (0.31) & 10.58 (0.27) & 14.22 (0.14) & {\bf7.23} (0.2) \\ 
  ssJACA & {\bf4.35} (0.16) & 14.79 (0.28) & {\bf9.72} (0.24) & {\bf13.76} (0.2) & 8.12 (0.18) \\ 
  SLDA\_sep & 6.76 (0.18) & 20 (0.46) & 17.9 (0.39) & 18.99 (0.3) & 12.01 (0.16) \\ 
  SLDA\_joint & 10.94 (0.21) & 55.08 (0.96) & 57.44 (1.29) & 45.01 (1.1) & 11.92 (0.24) \\ 
  \hline\hline
\end{tabular}
\end{table}

\begin{table}[!t]
\center
\caption{\label{tab:card_BRCA}Mean numbers of selected features over 100 splits of 377 samples from BRCA data, standard errors are given in brackets and the lowest values are highlighted in bold. }
\begin{tabular}{lccccc}
  \hline\hline
  & \multicolumn{5}{c}{Cardinality} \\
  \cmidrule(lr){2-6}
 Method & GE & ME & miRNA & RPPA & All   \\
 \hline 
JACA & 388 (15.7) & 321.1 (7.9) & 233.8 (6.4) & 114 (2.2) & 709.1 (23.4) \\ 
  ssJACA & 482.4 (11.6) & 397.6 (5.6) & 284.4 (4.4) & 136.9 (1.5) & 880 (17.1) \\ 
  SLDA\_sep & 65.6 (3.2) & 80.9 (4.1) & 56.1 (2.5) & 28.1 (2.1) & 146.5 (5.4) \\ 
  SLDA\_joint & {\bf48.9} (2.5) & {\bf2.6} (0.3) & {\bf2.2} (0.3) & 3{\bf.1 (0.2)} & {\bf51.5} (2.8) \\ 
  \hline \hline
\end{tabular}
\end{table}

\begin{table}[!t]
\center
\caption{\label{tab:correlation_BRCA}Analysis based on 377 samples from BRCA data with complete view and subtype information based on 100 random splits. Mean of sum correlation $\sum_{l\neq d}\Cor(\bX_l\widehat\bW_l, \bX_d\widehat\bW_d)$, where $\bX_d$ are samples from test data and $\widehat \bW_d$ are estimated from training data. Standard errors are given in brackets and the highest value is highlighted in bold.}
\begin{tabular}{lcccc}
  \hline\hline
 & JACA & ssJACA & SLDA\_sep & SLDA\_joint \\ 
  \hline
Correlation & 5.23 (0.004) & {\bf 5.46} (0.004) & 4.99 (0.006) & 1.1 (0.052) \\ 
  \hline\hline
\end{tabular}
\end{table}

\end{appendix}

\bibliographystyle{apalike} 
\bibliography{Main}

\end{document}